\documentclass{article}
\PassOptionsToPackage{numbers,sort&compress}{natbib}

\usepackage[utf8]{inputenc} % allow utf-8 input
\usepackage[T1]{fontenc}    % use 8-bit T1 fonts
\usepackage{hyperref}       % hyperlinks
\usepackage{url}            % simple URL typesetting
\usepackage{booktabs}       % professional-quality tables
\usepackage{amsmath, amssymb, amsthm, amsfonts}       % blackboard math symbols
\usepackage{nicefrac, xfrac}       % compact symbols for 1/2, etc.
\usepackage{microtype}      % microtypography
\usepackage{xcolor}         % colors
\usepackage{bm}
\usepackage{enumitem}
\usepackage{graphicx}
\usepackage{algorithm}
\usepackage{algorithmic}
\usepackage{natbib}[numbers,sort&compress]
\usepackage{geometry}
\usepackage{authblk,textcomp}

\usepackage[cal=euler]{mathalfa}
\usepackage{libertine}

\usepackage{mathtools}
\usepackage{amssymb}
\usepackage{graphicx}
\usepackage{etoolbox}
\usepackage[utf8]{inputenc}
\usepackage[OT1]{fontenc}
\usepackage{amsthm}

\usepackage{microtype}
\usepackage{subfigure}
\usepackage{hyperref}
\usepackage{physics}
\usepackage{enumitem}

\usepackage{amsmath}
\usepackage{xfrac}
\usepackage{bbm}
\usepackage{bigints}

\providecommand\given{}
\DeclarePairedDelimiterX{\set}[1]\{\}{\renewcommand\given{\SetSymbol[\delimsize]}#1}

\DeclareMathOperator{\Cov}{Cov}
\DeclareMathOperator{\diag}{diag}
\DeclareMathOperator{\E}{\mathbf E}
\DeclareMathOperator{\x}{x} % reverting bold notation for inputs.
\DeclareMathOperator{\z}{\mathbf z}
\DeclareMathOperator{\Spec}{Spec}
\DeclareMathOperator{\dist}{dist}
\DeclareMathOperator{\He}{He}
\newcommand{\R}{\mathbb R}

\newcommand{\N}{\mathbf{N}}

\newcommand{\vx}{{x}}
\newcommand{\vy}{{y}}
\newcommand{\dif}{\operatorname{d}\!{}}

\usepackage[capitalize,noabbrev]{cleveref}

\newcommand{\wh}{\widehat}
\DeclareFontFamily{U}{mathx}{\hyphenchar\font45}
\DeclareFontShape{U}{mathx}{m}{n}{
      <5> <6> <7> <8> <9> <10>
      <10.95> <12> <14.4> <17.28> <20.74> <24.88>
      mathx10
      }{}
\DeclareSymbolFont{mathx}{U}{mathx}{m}{n}
\DeclareFontSubstitution{U}{mathx}{m}{n}
\DeclareMathAccent{\wc}{0}{mathx}{"71}
\DeclareMathAccent{\wideparen}{0}{mathx}{"75}

\hypersetup{pdfauthor={ENS},pdftitle={dRF},%
            colorlinks, linktocpage=true, pdfstartpage=1, pdfstartview=FitV,%
    breaklinks=true, pdfpagemode=UseNone, pageanchor=true, pdfpagemode=UseOutlines,%
    plainpages=false, bookmarksnumbered, bookmarksopen=true, bookmarksopenlevel=1,%
    hypertexnames=true, pdfhighlight=/O,%
    urlcolor=orange, linkcolor=blue, citecolor=blue
        }

\title{Deterministic equivalent and error universality \\ of deep random features learning}

\author[1*]{Dominik Schr\"oder}
\author[2*]{Hugo Cui}
\author[3]{Daniil Dmitriev}
\author[4]{Bruno Loureiro}
\affil[1]{\small Department of Mathematics, ETH Zurich, 8006 Z\"urich, Switzerland}
\affil[2]{\small Statistical Physics Of Computation lab.,
Institute of Physics, \'Ecole Polytechnique F\'ed\'erale de Lausanne (EPFL), \newline 1015 Lausanne, Switzerland}
\affil[3]{\small Department of Mathematics, ETH Zurich and ETH AI Center, 8092 Z\"urich, Switzerland}
\affil[4]{\small D\'epartement d'Informatique, \'Ecole Normale Sup\'erieure (ENS) - PSL \& CNRS, 
F-75230 Paris cedex 05, France}
\affil[ ]{\textit {dschroeder@ethz.ch, hugo.cui@epfl.ch, daniil.dmitriev@ai.ethz.ch, bruno.loureiro@di.ens.fr}}

\affil[ *]{\textit {Main contributions}}
\makeatletter
\newtheorem*{rep@theorem}{\rep@title}
\newcommand{\newreptheorem}[2]{%
\newenvironment{rep#1}[1]{%
 \def\rep@title{#2 \ref{##1}}%
 \begin{rep@theorem}}%
 {\end{rep@theorem}}}
\makeatother

\theoremstyle{plain}
\newtheorem{theorem}{Theorem}
\numberwithin{theorem}{section}
\newreptheorem{theorem}{Theorem}

\newtheorem{lemma}[theorem]{Lemma}
\newtheorem{proposition}[theorem]{Proposition}

\newtheorem{corollary}[theorem]{Corollary}
\newtheorem{conjecture}[theorem]{Conjecture}
\newtheorem{assumption}[theorem]{Assumption}

\theoremstyle{remark}
\newtheorem{remark}[theorem]{Remark}
\date{\today}

\begin{document}
\maketitle

%%%%%%%%%%%%%%%%%%%%%%%%%%%%%%%%%%%%%%%%%%%%%%%%%%%%%%%%%%%%%%%%%%%%%%%%%%%%%%%
\begin{abstract}
This manuscript considers the problem of learning a random Gaussian network function using a fully connected network with frozen intermediate layers and trainable readout layer. This problem can be seen as a natural generalization of the widely studied random features model to deeper architectures. First, we prove Gaussian universality of the test error in a ridge regression setting where the learner and target networks share the same intermediate layers, and provide a sharp asymptotic formula for it.
Establishing this result requires proving a deterministic equivalent for traces of the deep random features sample covariance matrices which can be of independent interest. Second, we conjecture the asymptotic Gaussian universality of the test error in the more general setting of arbitrary convex losses and generic learner/target architectures. We provide extensive numerical evidence for this conjecture. In light of our results, we investigate the interplay between architecture design and implicit regularization.
\end{abstract}

%%%%%%%%%%%%%%%%%%%%%%%%%%%%%%%%%%%%%%%%%%%%%%%%%%%%%%%%%%%%%%%%%%%%%%%%%%%%%%%
\section{Introduction}
\label{Introduction}
Despite the incredible practical progress in the applications of deep neural networks to almost all fields of knowledge, our current theoretical understanding thereof is still to a large extent incomplete. Recent progress on the theoretical front stemmed from the investigation of simplified settings, which despite their limitations are often able to capture some of the key properties of "real life" neural networks. A notable example is the recent stream of works on random features (RFs), originally introduced by \cite{Rahimi2007RandomFF} as a computationally efficient approximation technique for kernel methods, but more recently studied as a surrogate model for two-layers neural networks in the lazy regime \cite{Chizat2018OnLT, Pennington2019NonlinearRM, Mei2019TheGE, Gerace2020GeneralisationEI}. RFs are a particular instance of random neural networks, whose statistical properties have been investigated in a sizeable body of works \cite{Lee2018DeepNN, g.2018gaussian, Fan2020SpectraOT, NEURIPS2021_1baff70e, NEURIPS2021_ae4503ec}. The problem of training the readout layer of such networks has been addressed in the shallow (one hidden layer) case by \cite{Mei2019TheGE, Gerace2020GeneralisationEI}, who provide sharp asymptotic characterizations for the test error. A similar study in the generic deep case is, however, still missing. In this manuscript, we bridge this gap by considering the problem of learning the last layer of a deep, fully-connected random neural network, hereafter referred to as the \emph{deep random features} (dRF) model. More precisely, our \textbf{main contributions} in this manuscript are:
\begin{itemize}
    \item In~\Cref{sec:spectrum}, we state Theorem \ref{prop:mult_layers}, which proves an asymptotic deterministic equivalent for the traces of the product of deterministic matrices with both conjugate kernel and sample covariance matrix of the layer-wise post-activations. 
    
    \item As a consequence of Thm.~\ref{prop:mult_layers}, in~\Cref{sec:Error_uni} we derive a sharp asymptotic formula for the test error of the dRF model in the particular case where the target and learner networks share the same intermediate layers, and when the readout layer is trained with the squared loss. This result establishes the Gaussian equivalence of the test error for ridge regression in this setting. 
    
    \item Finally, we conjecture (and provide strong numerical evidence for) the Gaussian universality of the dRF model for general convex losses, and generic target/learner network architectures. \textcolor{black}{More specifically, we provide exact asymptotic formulas for the test error that leverage recent progress in high-dimensional statistics \cite{Loureiro2021CapturingTL} and a closed-form formula for the population covariance of network activations appearing in \cite{Cui2023}}. These formulas show that in terms of second-order statistics, the dRF is equivalent to a linear network with noisy layers. We discuss how this effective noise translates into a depth-induced implicit regularization in~\Cref{sec:architecture}. 
\end{itemize}

A GitHub repository with the code employed in the present work can be found \href{https://github.com/HugoCui/dRF}{here}.

\subsection*{Related work} 
\emph{Random features} were first introduced by \cite{Rahimi2007RandomFF}. The asymptotic spectral density of the single-layer conjugate kernel was characterized in \cite{Liao2018, Pennington2019NonlinearRM, Benigni2021}. Sharp asymptotics for the test error of the RF model appeared in \cite{Mei2019TheGE, Mei2021GeneralizationEO} for ridge regression, \cite{Gerace2020GeneralisationEI, Dhifallah2020} for general convex losses and \cite{Sur2020, Bosch2022} for other penalties. The implicit regularization of RFs was discussed in \cite{Jacot2020}. The RFs model has been studied in many different contexts as a proxy for understanding overparametrisation, e.g. in uncertainty quantification \cite{Clarte2022ASO}, ensembling \cite{Ascoli20a, Loureiro2022}, the training dynamics \cite{NEURIPS2021_b4f8e5c5, bordelon2022learning}, but also to highlight the limitations of lazy training \cite{Ghorbani2019, Ghorbani2020WhenDN, NEURIPS2019_5481b2f3, pmlr-v139-refinetti21b}; \\

\noindent \emph{Deep random networks} were shown to converge to Gaussian processes in \cite{Lee2018DeepNN, g.2018gaussian}. They were also studied in the context of inference in \cite{8006899, NEURIPS2018_6d0f8463}, and as generative priors to inverse problems in \cite{NEURIPS2019_2f3c6a4c, NEURIPS2018_1bc2029a, pmlr-v107-aubin20a}. The distribution of outputs of deep random nets was characterized in \cite{NEURIPS2021_1baff70e, NEURIPS2021_ae4503ec}. Close to our work is \cite{Fan2020SpectraOT}, which provide exact formulas for the asymptotic spectral density and Stieltjes transform of the NTK and conjugate kernel in the proportional limit. Our formulas for the sample and population covariance are complementary to theirs. The test error of linear-width deep networks has been recently studied in \cite{Li2021, Ariosto2022} through the lens of Bayesian learning;\\

\noindent \emph{Gaussian universality} of the test error for the RFs model was shown in \cite{Mei2019TheGE}, conjectured to hold for general losses in \cite{Gerace2020GeneralisationEI} and was proven in \cite{Goldt2021TheGE, Hu2020UniversalityLF}. Gaussian universality has also been shown to hold for other classes of features, such as two-layer NTK \cite{Montanari2022UniversalityOE}, kernel features \cite{bordelon2022learning, Jacot2020, Cui2019LargeDF, Cui2022ErrorRF}.  \cite{Loureiro2021CapturingTL} provided numerical evidence for Gaussian universality of more general feature maps, including pre-trained deep features. \\

\noindent \emph{Deterministic equivalents} of sample covariance matrices have first been established in~\cite{10.1016/j.physa.2004.05.048,10.1007/s00440-016-0730-4} for separable covariances, generalizing the seminal work~\cite{10.1070/SM1967v001n04ABEH001994} on the free convolution of spectra in an anisotropic sense. More recently these results have been extended to non-separable covariances, first in tracial~\cite{10.2307/24308489}, and then also in anisotropic sense~\cite{louart2018concentration, chouard2022quantitative}. 
%%%%%%%%%%%%%%%%%%%%%%%%%%%%%%%%%%%%%%%%%%%%%%%%%%%%%%%%%%%%%%%%%%%%%%%%%%%%%%%
\section{Setting \& preliminaries}
\label{sec:setting}
Let $(\x^\mu,y^\mu)\in\mathbb{R}^{d}\times \mathcal{Y}$, $\mu\in[n]\coloneqq\{1,\cdots, n\}$, denote some training data, with $\x^{\mu}\sim\mathcal{N}(0_{d},\Omega_{0})$ independently and $y^{\mu}=f_{\star}(\x^{\mu})$ a (potentially random) target function. This work is concerned with characterising the learning performance of generalised linear estimation:
\begin{equation}
\label{eq:GLM}
\hat{y} = \sigma\left(\frac{\theta^{\top}\varphi(\x)}{\sqrt{k}}\right),
\end{equation}
with \emph{deep random features} (dRF):
\begin{align}
\label{eq:definition_multilayer_RF}
    \varphi(\x)\coloneqq \underbrace{\left(
    \varphi_{L}\circ \varphi_{L-1}\circ \dots\circ \varphi_2\circ\varphi_1\right)}_{L}
    (\x),
\end{align}
\noindent where the post-activations are given by:
\begin{equation}
    \varphi_\ell(h)=\sigma_\ell\left(
    \frac{1}{\sqrt{k_{\ell-1}}}W_\ell \cdot h
    \right), \quad \ell \in [L].
\end{equation}
The weights $\{W_\ell\in\mathbb{R}^{k_\ell\times k_{\ell-1}}\}_{\ell\in[L]}$ are assumed to be independently drawn Gaussian matrices with i.i.d. entries $(W_\ell)_{ij}\sim\mathcal{N}(0,\Delta_\ell) ~~\forall 1\le i\le k_\ell,~1\le j\le k_{\ell-1}$. To alleviate notation, sometimes it will be convenient to denote $k_{L} = k$. Only the readout weights $\theta\in\mathbb{R}^{k}$ in \eqref{eq:GLM} are trained according to the usual regularized \emph{empirical risk minimization} procedure: 
\begin{align}  
\label{eq:ERM}    
\hat{\theta}=\underset{\theta\in\mathbb{R}^{k}}{\mathrm{argmin}}\left[\sum\limits_{\mu=1}^n \ell(y^\mu,\theta^{\top}\varphi(\x^{\mu}))+\frac{\lambda}{2}||\theta||^2\right],  
\end{align}
\noindent where $\ell:\mathcal{Y}\times\mathbb{R}\to\mathbb{R}_{+}$ is a loss function, which we assume convex, and $\lambda>0$ sets the regularization strength. 

To assess the training and test performances of the empirical risk minimizer \eqref{eq:ERM}, we let $g:\mathcal{Y}\times\mathbb{R}\to\mathbb{R}_{+}$ be any performance metric (e.g. the loss function itself or, in the case of classification, the probability of misclassifying), and define the test error:
\begin{align}
\label{eq:def:errors}
\epsilon_g(\hat{\theta}) &\coloneqq \E\left[g(y, \hat{\theta}^{\top}\varphi(\x))\right]
\end{align}

Our main goal in this work is to provide a sharp characterization of \eqref{eq:def:errors} in the proportional asymptotic regime $n,d, k_\ell\xrightarrow{}\infty$ at fixed $\mathcal{O}(1)$ ratios $\alpha\coloneqq\sfrac{n}{d}$ and $\gamma_\ell\coloneqq\sfrac{k_\ell}{d}$, for all layer index $\ell\in[L]$. This requires a precise characterization of the {\it sample and population covariances} and the {\it Gram} matrices of the post-activations.

%%%%%%%%%%%%%%%%%%%%%%%%%%%%%%%%%%%%%%%
\subsection{Background on sample covariance matrices}
%%%%%%%%%%%%%%%%%%%%%%%%%%%%%%%%%%%%%%%
\paragraph{Marchenko-Pastur and free probability:}
We briefly introduce basic nomenclature on sample covariance matrices. For a random vector $x\in\R^d$ with mean zero $\E x=0$ and covariance $\Sigma:=\E xx^\top\in\R^{d\times d}$, we call the matrix $\wh\Sigma:=\mathcal X\mathcal X^\top/n\in\R^{d\times d}$ obtained from \(n\) independent copies \(x_1,\ldots,x_n\) of \(x\) written in matrix form as \(\mathcal X:=(x_1,\ldots,x_n)\) the \emph{sample covariance matrix} corresponding to the \emph{population covariance matrix} \(\Sigma\). The \emph{Gram matrix} \(\wc\Sigma:=\mathcal X^\top\mathcal X/n\in\R^{n\times n}\) has the same non-zero eigenvalues as the sample covariance matrix but unrelated eigenvectors. The systematic mathematical study of sample covariance and Gram matrices has a long history dating back to~\cite{10.1093/biomet/20A.1-2.32}. While in the ``classical'' statistical limit $n\to\infty$ with $d$ being fixed the sample covariance matrix converges to the population covariance matrix $\wh\Sigma\to\Sigma$, in the proportional regime \(d\sim n\gg 1\) the non-trivial asymptotic relationship between the spectra of \(\wh\Sigma\) and \(\Sigma\) has first been obtained in the seminal paper~\cite{10.1070/SM1967v001n04ABEH001994}: the empirical spectral density \(\mu(\wh\Sigma) := d^{-1}\sum_{\lambda\in\Spec(\wh\Sigma)}\delta_\lambda\) of \(\wh\Sigma\) is approximately equal to the \emph{free multiplicative convolution} of \(\mu(\Sigma)\) and a Marchenko-Pastur distribution \(\mu_\mathrm{MP}^{c}\) of aspect ratio \(c=d/n\),
\begin{equation}
\label{SC MP}
    \mu(\wh\Sigma)\approx\mu(\Sigma)\boxtimes\mu_\mathrm{MP}^{d/n}.
\end{equation}
Here the free multiplicative convolution \(\mu\boxtimes \mu_\mathrm{MP}^c\) may be defined as the unique distribution \(\nu\) whose Stieltjes transform \(m=m_\nu(z):=\int (x-z)^{-1}\dif\nu(x)\) satisfies the scalar \emph{self-consistent equation} 
\begin{equation}
    z m = \frac{z}{1-c-czm} m_{\mu}\left(\frac{z}{1-c-czm}\right).
\end{equation}
The spectral asymptotics~\eqref{SC MP} originally were obtained in the case of Gaussian \(\mathcal X\) or, more generally, for separable correlations \(\mathcal X=\sqrt{\Sigma}Y\) for some i.i.d.\ matrix \(Y\in\R^{d\times n}\). These results were later extended~\cite{10.2307/24308489} to the general case under essentially optimal assumptions on concentrations of quadratic forms \(x^\top A x\) around their expectation \(\Tr A\Sigma\).  

\paragraph{Deterministic equivalents:}
It has only been recognised much later~\cite{10.1016/j.physa.2004.05.048,10.1007/s00440-016-0730-4} that the relationship~\eqref{SC MP} between the asymptotic spectra of \(\Sigma\) and \(\wh\Sigma,\wc\Sigma\) actually extends to eigenvectors as well, and that the resolvents \(\wh G(z):=(\wh\Sigma-z)^{-1}\), \(\wc G(z):=(\wc\Sigma-z)^{-1}\) are asymptotically equal to \emph{deterministic equivalents}
\begin{equation}\label{scov det eq}
    \begin{split}
        \wh M(z) := -\frac{(\Sigma\wc m(z)+I_d)^{-1}}{z} , \quad 
        \wc M(z) &:= \wc m(z)I_n,
    \end{split}
\end{equation}
also in an \emph{anisotropic} rather than just a tracial sense, highlighting that despite the simple relationship between their averaged traces 
\begin{equation*}
    \wh m(z):=m_{\mu(\Sigma)\boxtimes\mu_\mathrm{MP}^c}(z), \quad \wc m(z) = \frac{c-1}{z} + c\wh m(z),
\end{equation*}
the sample covariance and Gram matrices carry rather different non-spectral information. The anisoptric concentration of resolvents (or in physics terminology, the self-averaging) has again first been obtained in the Gaussian or separable cases~\cite{10.1016/j.physa.2004.05.048,10.1007/s00440-016-0730-4}. The extension to general sample covariance matrices was only achieved much more recently~\cite{louart2018concentration, chouard2022quantitative} under Lipschitz concentration assumptions. In this work we specifically use the deterministic equivalent for sample covariance matrices with general covariance from~\cite{chouard2022quantitative} and extend it to cover Gram matrices. 

\paragraph{Application to the deep random features model:}
In this work we apply the general theory of anisotropic deterministic equivalents to the deep random features model. As discussed in Section \ref{sec:Error_uni}, to prove error universality even for the simple ridge regression case, it is not enough to only consider the spectral convergence of the matrices, and a stronger result is warranted. The application of non-linear activation functions makes the model neither Gaussian nor separable, hence our analysis relies on the deterministic equivalents from~\cite{chouard2022quantitative} and our extension to Gram matrices, which appear naturally in the explicit error derivations.

\subsection{Notation}
We will adopt the following notation:
\begin{itemize}
    \item For $A\in\R^{n\times n}$ we denote $\langle A \rangle \coloneqq \sfrac{1}{n}\tr A$.
    \item For matrices \(A\in\R^{n\times m}\) we denote the operator norm (with respect to the \(\ell^2\)-vector norm) by \(\norm{A}\), the max-norm by \(\norm{A}_\mathrm{max}:=\max_{ij}\abs{A_{ij}}\), and the Frobenius norm by \(\norm{A}_\mathrm{F}^2:=\sum_{ij}\abs{A_{ij}}^2\). 
    \item For any distribution $\mu$ we denote the push-forward under the map \(\lambda\mapsto a\lambda+b\) by $a\otimes \mu \oplus b$ in order to avoid confusion with e.g.\ the convex combination \(a\mu_1 +(1-a)\mu_2\) of measures \(\mu_1,\mu_2\).% refers to the 
    \item We say that a sequence of random variables \((X_n)_n\) is \emph{stochastically dominated} by another sequence \((Y_n)_n\) if for all small \(\epsilon>0\) and large \(D<\infty\) it holds that \(P(X_n>n^{\epsilon}Y_n)\le n^{-D}\) for large enough \(n\), and in this case write \(X_n\prec Y_n\).
\end{itemize}
%%%%%%%%%%%%%%%%%%%%%%%%%%%%%%%%%%%%%%%%%%%%%%%%%%%%%%%%%%%%%%%%%%%%%%%%%%%%%%%
\section{Deterministic equivalents}
\label{sec:spectrum}
Consider the sequence of variances defined by the recursion
\begin{align}
\label{eq:r_multilayer}
r_{\ell+1}=\Delta_{\ell+1}\E_{\xi\sim\mathcal{N}(0,r_\ell)}\left[\sigma_\ell(\xi)^2\right]
\end{align}
with initial condition $r_1\coloneqq\Delta_1 \sfrac{\langle \Omega_0\rangle}{d}$
and coefficients
\begin{align}
\label{eq:kappa_multilayer}
    &\kappa_1^{\ell}=\frac{1}{r_\ell}\E_{\xi\sim\mathcal{N}(0,r_\ell)}\left[\xi\sigma_\ell(\xi)\right],
    \notag\\ 
    &\kappa_*^{\ell}=\sqrt{\E_{\xi\sim\mathcal{N}(0,r_\ell)}\left[\sigma_\ell(\xi)^2\right]-r_\ell\left(\kappa_1^{\ell}\right)^2}.
\end{align}
\subsection{Rigorous results on the multi-layer sample covariance and Gram matrices}
Our main result on the anisotropic deterministic equivalent of dRFs follows from iterating the following proposition. We consider a data matrix \(X_0\in\mathbb R^{d\times n}\) whose Gram matrix concentrates as 
\begin{equation}\label{X0 assump}
    \norm{\frac{X_0^\top X_0}{d}-r_1 I}_\mathrm{max}\prec \frac{1}{\sqrt{n}}, \quad \norm{\frac{X_0}{\sqrt{d}}}\prec 1
\end{equation}
for some positive constant $r_1$. The Assumption~\eqref{X0 assump} for instance is satisfied if the columns $\x$ of $X_0$ are independent with mean $\E \x=0$ and covariance $\E \x\x^\top=\Omega_0\in\R^{d\times d}$ (together with some mild assumptions on the fourth moments), in which case \(r_1=\langle \Omega_0\rangle\) is the normalised trace of the covariance. We then consider $X_1\coloneqq\sigma_1(W_1 X_0/\sqrt{d})$ assuming the entries of \(W_1\in\R^{k_1\times d}\) are iid.\ $\mathcal N(0,1)$ elements, and $\sigma_1$ satisfies $\E_{\xi\sim\mathcal{N}(0,1)} \sigma_1(\sqrt{r_1}\xi)=0$ 
%for a standard Gaussian $N$ 
in the proportional \(n\sim d\sim k_1\) regime. Upon changing \(\sigma_1\) there is no loss in generality in assuming \(\Delta_1=1\) which we do for notational convenience.\\

\begin{proposition}[Deterministic equivalent for RF]
\label{prop 1layer}
\label{prop:one_layer}
For any deterministic \(A\) and Lipschitz-continuous activation function \(\sigma_1\), under the assumptions above, we have that, for any \(z \in \mathbf{C} \setminus \R_+\)
\begin{equation*}
\left | \left\langle A\Bigl[ \Bigl(\frac{X_1^\top X_1}{k_1} - z\Bigr)^{-1} - \wh M(z)\Bigr]\right\rangle \right | \prec \frac{\langle A A^\ast\rangle^{1/2}}{\delta^9\sqrt{n}},
\end{equation*}
and 
\begin{equation*}
    \left | \left\langle A \Bigl(\frac{X_1 X_1^\top}{k_1} - z\Bigr)^{-1}\right\rangle - \langle A\rangle\wc m(z) \right | \prec \frac{\langle A A^\ast\rangle^{1/2}}{\delta^9\sqrt{n}},
\end{equation*}     
where $\delta \coloneqq \dist(z, \R_+)$, 
\begin{equation}\label{Sigma lin def}
    \begin{split}
        -z\wh M(z)&\coloneqq \Bigl(\wc m(z)\Sigma_{\mathrm{lin}} + I\Bigr)^{-1},\\
        \Sigma_\mathrm{lin} &\coloneqq (\kappa_1^1)^2\frac{X_0^\top X_0}{d} + (\kappa_\ast^1)^2 I,
    \end{split}
\end{equation}
and 
\begin{equation*}
    \wh m(z):=m_{\mu(\Sigma_\mathrm{lin})\boxtimes\mu_\mathrm{MP}^{n/k_1}}(z), \quad \wc m(z) = \frac{n-k_1}{nz} + \frac{n}{k_1}\wh m(z).
\end{equation*}
Furthermore, Assumption~\eqref{X0 assump} holds true with $X_0,r_1$ replaced by $X_1,r_2$, respectively, and we have that \(\dist(-1/\wc m(z),\R_+)\ge \dist(z,\R_+)\). 
\end{proposition}
\begin{remark}
    The tracial version of~\Cref{prop 1layer} has appeared multiple times in the literature, e.g.~\cite{10.2307/24308489}. It implies that the spectrum \(\wh\mu_1\) of \(X_1^\top X_1/k_1\) is approximately given by the free multiplicative convolution
    \begin{equation}\label{eq hatmu1 rec}
        \begin{split}
            \wh\mu_1 &\approx \mu\Bigl((\kappa_1^1)^2\frac{X_0^\top X_0}{d}+(\kappa_\ast^1)^2I\Bigr) \boxtimes \mu_\mathrm{MP}^{n/k_1} \\
            & = \Bigl(\mu\Bigl((\kappa_1^1)^2\frac{X_0^\top X_0}{d}\Bigr) \boxplus \delta_{(\kappa_\ast^1)^2} \Bigr)\boxtimes \mu_{\mathrm{MP}}^{n/k_1}.
        \end{split}
    \end{equation}
    In case \(c\le 1\), i.e.\ when \(\mu_\mathrm{MP}^{c}\) has no atom at \(0\), it was shown in~\cite{MR2682261} that 
    \begin{equation}\label{distributive conv}
        \sqrt{\mu\boxtimes\mu_\mathrm{MP}^c}\boxplus_c\sqrt{\mu'\boxtimes \mu_\mathrm{MP}^c} = \sqrt{(\mu\boxplus\mu')\boxtimes\mu_\mathrm{MP}^c}
    \end{equation}
    which allows to simplify~\eqref{eq hatmu1 rec}. Here \(\boxplus_c\) is the \emph{rectangular free convolution} which models the distribution of singular values of the addition of two free rectangular random matrices, and the square-root is to be understood as the push-forward of the square-root map. Applying~\eqref{distributive conv} to~\eqref{eq hatmu1 rec} yields 
    \begin{equation}
        \sqrt{\wh\mu_1} \approx \Bigl(\kappa_1^1 \otimes\sqrt{\wh\mu_0\boxtimes \mu_\mathrm{MP}^{n/k_1}}\Bigr) \boxplus_{n/k_1} \kappa_\ast^1 \otimes \sqrt{\mu_\mathrm{MP}^{n/k_1}},
    \end{equation}
    suggesting that the non-zero singular values of \(X_1/\sqrt{k}\) can be modeled by the non-zero singular values of the \emph{Gaussian equivalent model}:
    \begin{equation}
        c' W' X_0 + c'' W''
    \end{equation}
    for some suitably chosen constants \(c',c''\) and independent Gaussian matrices \(W,W'\). 
\end{remark}

The last assertion of~\Cref{prop 1layer} allows to iterate over an arbitrary (but finite) number of layers. Indeed, after one layer we have 
\begin{equation}
    \begin{split}
        \Bigl(\frac{X_1^\top X_1}{k_1}-z_1\Bigr)^{-1} &\approx \Bigl(-\wc m(z_1)z_1\Sigma_\mathrm{lin}-z_1\Bigr)^{-1}\\
        &= c_1 \Bigl(\frac{X_0^\top X_0}{k_0}-z_0\Bigr)^{-1},
    \end{split}
\end{equation}
using the definitions from~\Cref{prop:mult_layers} for \(c_1,z_0\) below.
\begin{theorem}[Deterministic equivalent for dRF]
\label{prop:mult_layers}
For any deterministic \(A\) and Lipschitz-continious activation functions \(\sigma_1, \ldots, \sigma_\ell\) satisfying \(\E_{\xi\sim\mathcal{N}(0,1)} \sigma_m(\sqrt{r_m}\xi)=0\), under the Assumption~\eqref{X0 assump} above, we have that for any \(z_\ell\in \mathbf C\setminus\R_+\)
\begin{equation*}
\begin{aligned}
&\left| \left\langle A \Bigl(\frac{X_\ell^\top X_\ell}{k_\ell}-z_\ell\Bigr)^{-1}\right\rangle 
- c_1\cdots c_\ell \wc m_0\langle A\rangle \right|\prec \frac{\langle A A^\ast\rangle^{1/2}}{\delta^9\sqrt{n}}
\end{aligned}
\end{equation*} 
and that 
\begin{equation*}
    \left | \left\langle A \Bigl(\frac{X_\ell X_\ell^\top}{k_\ell} - z_\ell\Bigr)^{-1}\right\rangle - \wc m_\ell\langle A\rangle \right | \prec \frac{\langle A A^\ast\rangle^{1/2}}{\delta^9\sqrt{n}},
\end{equation*}
where \(\delta:=\dist(z_\ell,\R_+)\), and we recursively define 
\begin{equation}
\begin{split}
\label{eq:companion_stieltjes}
\Sigma_\mathrm{lin}^{\ell-1} &\coloneqq (\kappa_1^{\ell})^2\frac{X_{\ell-1}^\top X_{\ell-1}}{k_{\ell-1}} + (\kappa_\ast^{\ell})^2 I,\\
\wc m_\ell &\coloneqq \frac{n-k_\ell}{nz_\ell}+ \frac{n}{k_\ell} m_{\mu(\Sigma_\mathrm{lin}^{\ell-1})\boxtimes \mu_\mathrm{MP}^{n/k_\ell}}(z_\ell)\\
-\frac{1}{c_\ell} &\coloneqq \wc m_\ell z_\ell (\kappa_1^{\ell})^2, \quad z_{\ell-1} \coloneqq c_\ell z_\ell-\Bigl(\frac{\kappa_\ast^\ell}{\kappa_1^\ell}\Bigr)^2
\end{split}
\end{equation}
for \(\ell\ge 1\) and finally 
\begin{equation}
    \wc m_0 \coloneqq \frac{d-n}{nz_0} + \frac{d^2}{n^2}m_{\mu(\Omega_0)\boxtimes \mu_\mathrm{MP}^{d/n}}\Bigl(\frac{d}{n}z_0\Bigr). 
\end{equation}
\end{theorem}
Proofs of~\Cref{prop:one_layer}~and~\Cref{prop:mult_layers} are given in App.~\ref{App:anisotropic}.
\begin{remark}
    The same iteration argument has appeared before in~\cite{Fan2020SpectraOT}. The main difference to our present work is the anisotropic nature of our estimate which allows to test both sample covariance, as well as Gram resolvent against arbitrary deterministic matrices. As we will discuss in the next section, this is crucial in order to provide closed-form asymptotics for the test error of the deep random features model. 
\end{remark}

\subsection{Closed-formed formula for the population covariance}
In~\Cref{prop 1layer,prop:mult_layers} we iteratively considered \(X_\ell^\top X_\ell/k_\ell\) as a sample-covariance matrix with population covariance 
\begin{equation*}
    \E_{W_\ell} \frac{X_\ell^\top X_\ell}{k_\ell} = \E_w \sigma_\ell\Bigl(\frac{X_{\ell-1}^\top w}{\sqrt{k_{\ell-1}}}\Bigr)\sigma_\ell\Bigl(\frac{w^\top X_{\ell-1}}{\sqrt{k_{\ell-1}}}\Bigr) \approx \Sigma_\mathrm{lin}^\ell
\end{equation*}
and from this obtained formulas for the deterministic equivalents for both \(X_\ell^\top X_\ell\) and \(X_\ell X_\ell^\top\). A more natural approach would be to consider \(X_\ell X_\ell^\top/n\) as a sample covariance matrix with population covariance 
\begin{align}
\label{eq:def_pop_cov}
    \Omega_\ell&:=\E_{X_0} \frac{X_\ell X_\ell^\top}{n},
%\notag\\&=\E_{x\sim \mathcal N(0,\Omega_0)} \left[(\varphi_\ell\circ...\circ\varphi_1)(x)(\varphi_\ell\circ...\circ\varphi_1)(x)^\top\right],
\end{align}
noting that the matrix \(X_\ell\) conditioned on \(W_1,\ldots,W_\ell\) has independent columns. \Cref{thm:res_concentration,co resolvent} apply also in this setting, but lacking a rigorous expression for \(\Omega_\ell\) the resulting deterministic equivalent is less descriptive than the one from~\Cref{prop:mult_layers}.
\textcolor{black}{A heuristic closed-form formula for the population covariance which is conjectured to be exact was recently derived in \cite{Cui2023}. We now discuss this result, and for the sake of completeness provide a derivation in Appendix App.~\ref{App:population_heuristic}}.
% Nevertheless, in App.~\ref{App:population_heuristic} we derive an heuristic closed-form formula for the population covariance, which we conjecture exact. 
Consider the sequence of matrices $\{\Omega^{\mathrm{lin}}_\ell\}_\ell$ defined by the recursion
\begin{equation}
\label{eq:Omega_lin_recursion}
\Omega^{\mathrm{lin}}_{\ell+1}=\kappa_1^{(\ell+1) 2}\frac{W_{\ell+1} \Omega^{\mathrm{lin}}_\ell W_{\ell+1}^\top}{k_\ell}+\kappa_*^{(\ell+1)2}I_{k_{\ell+1}}.
\end{equation}
with $\Omega^{\mathrm{lin}}_0\coloneqq \Omega_0$. Informally, $\Omega^{\mathrm{lin}}_\ell$ provides an asymptotic approximation of $\Omega_\ell$ in the sense that the normalized distance $\sfrac{||\Omega^{\mathrm{lin}}_\ell-\Omega_\ell||_F}{\sqrt{d}}$ is of order $\mathcal{O}(\sfrac{1}{\sqrt{d}})$. Besides, the recursion \eqref{eq:Omega_lin_recursion} implies that $\Omega^{\mathrm{lin}}_\ell$ can be expressed as a sum of products of Gaussian matrices (and transposes thereof), and affords a straightforward way to derive an analytical expression its asymptotic spectral distribution. This derivation is presented in App.\, \ref{App:population_heuristic}.\looseness=-1

It is an interesting question whether an approximate formula for the population covariance matrix like the one in~\cref{eq:Omega_lin_recursion} can be obtained indirectly via~\Cref{prop:mult_layers}. There is extensive literature on this \emph{inverse problem}, i.e.\ how to infer spectral properties of the population covariance spectrum from the sample covariance spectrum, e.g.~\cite{10.1214/07-AOS581} but we leave this avenue to future work.

\subsection{Consistency of~\Cref{prop:mult_layers} and the approximate population covariance}
What we can note, however, is that~\Cref{eq:Omega_lin_recursion} is \emph{consistent} with~\Cref{prop:mult_layers}. We demonstrate this in case of equal dimensions \(n=d=k_1=\cdots=k_\ell\) to avoid unnecessary technicalities due to the zero eigenvalues. We define 
\begin{equation}
    \wh\mu_\ell \coloneqq \mu\Bigl(\frac{X_\ell^\top X_\ell}{k_\ell}\Bigr)=\wc\mu_\ell \coloneqq \mu\Bigl(\frac{X_\ell X_\ell^\top}{n}\Bigr)
\end{equation}
and recall that~\Cref{prop 1layer} implies that 
\begin{equation}
    \wh\mu_\ell \approx ((\kappa_1^l)^2\otimes\wh\mu_{l-1}\oplus (\kappa_\ast^l)^2)\boxtimes \mu_\mathrm{MP}.
\end{equation}
On the other hand~\eqref{SC MP} applied to the sample covariance matrix \(X_\ell X_\ell^\top/n\) with population covariance \(\Omega_\ell\approx\Omega_\ell^\mathrm{lin}\) implies that 
\begin{equation}
    \begin{split}
        \wc\mu_\ell &\approx \mu(\Omega_\ell^\mathrm{lin}) \boxtimes\mu_\mathrm{MP}\\
        &=\mu\Bigl((\kappa_1^{\ell})^2\frac{W_{\ell} \Omega^{\mathrm{lin}}_{\ell-1} W_{\ell}^\top}{k_{\ell-1}}+(\kappa_*^{\ell})^2 I_{k_{\ell}}\Bigr)\boxtimes \mu_\mathrm{MP}\\
        &\approx\Bigl((\kappa_1^\ell)\otimes \mu(\Omega_{\ell-1}^\mathrm{lin})\boxtimes\mu_\mathrm{MP}\oplus(\kappa_\ast^\ell)^2\Bigr)\boxtimes\mu_\mathrm{MP}\\
        &\approx \Bigl((\kappa_1^\ell) \otimes \wc\mu_{\ell-1}\oplus(\kappa_\ast^\ell)^2\Bigr)\boxtimes\mu_\mathrm{MP},
    \end{split}
\end{equation}
demonstrating that both approaches lead to the same recursion. Here in the third step we applied~\eqref{SC MP} to the sample covariance matrix \(\sqrt{\Omega_{\ell-1}^\mathrm{lin}}W_\ell^\top\), and in the fourth step used the first approximation for \(\ell\) replaced by \(\ell-1\). 
%%%%%%%%%%%%%%%%%%%%%%%%%%%%%%%%%%%%%%%%%%%%%%%%%%%%%%%%%%%%%%%%%%%%%%%%%%%%%%%
\section{Gaussian universality of the test error}
\label{sec:Error_uni}
In the second part of this work, we discuss how the results on the asymptotic spectrum of the empirical and population covariances of the features can be used to provide sharp expressions for the test and training errors \eqref{eq:def:errors} when the labels are generated by a deep random neural network:
 \begin{align}
 \label{eq:target}
     f_{\star}(\x^{\mu})=\sigma^\star\left(\frac{\theta_\star^\top\varphi^\star
    (\x^{\mu})}{\sqrt{k^\star}}\right).
 \end{align}
The feature map $\varphi^\star$ denotes the composition $\varphi_{L^\star}^\star\circ...\circ \varphi_1^\star$ of the $L^{\star}+1$ layers:
\begin{align*}
    \varphi^{\star}_\ell(\x)=\sigma_\ell^\ast\left(
    \frac{1}{\sqrt{k_{\ell-1}^\star}}W_\ell^\star \cdot \x
    \right),
\end{align*}
\noindent and $\theta_\star\in\mathbb{R}^{k^\star}$ is the last layer weights. To alleviate notations, we denote $k^\star:=k^\star_L$. The weight matrices $\{W_{\ell}^{\star}\}_{\ell\in [L^{\star}]}$ have i.i.d Gaussian entries sampled from $\mathcal{N}(0,\Delta_\ell^\star)$. Note that we do not require the sequence of activations $\{\sigma_\ell^\star\}_\ell$ and widths $\{\gamma_\ell\coloneqq \sfrac{k_\ell^\star}{d}\}_\ell$ to match with those of the learner dRF \eqref{eq:definition_multilayer_RF}. We address in succession
\begin{itemize}
    \item The well-specified case where the target and learner networks share the same intermediate layers (i.e. same architecture, activations and weights) $\varphi^{\star}_{\ell}=\varphi_{\ell}$, $\ell\in[L]$ with $L^{\star}=L$, and the readout of the dRF is trained using ridge regression. This is equivalent to the interesting setting of ridge regression on a linear target, with features drawn from a non-Gaussian distribution, resulting from the propagation of Gaussian data through several non-linear layers.
    \item The general case where the target and learner possess generically distinct architectures, activations and weights, and a generic convex loss.
\end{itemize}
In both cases, we provide a sharp asymptotic characterization of the test error. Furthermore, we establish the equality of the latter with the test error of an equivalent learning problem on \textit{Gaussian samples} with matching population covariance, thereby showing the Gaussian universality of the test error. In the well-specified case, our results are rigorous, and make use of the deterministic equivalent provided by Theorem \ref{prop:mult_layers}. In the fully generic case, we formulate a conjecture, which we strongly support with finite-size numerical experiments.

%%%%%%%%%%%%%%%%%%%%%%%%%%%%%%%%%%%
\subsection{Well-specified case}
%%%%%%%%%%%%%%%%%%%%%%%%%%%%%%%%%%%
We first establish the Gaussian universality of the test error of dRFs in the matched setting $\varphi=\varphi^\star$, for a readout layer trained using a square loss. This corresponds to $\mathcal{Y}=\mathbb{R}$, $\ell(y,\hat{y}) = \sfrac{1}{2}(y-\hat{y})^2$. This case is particularly simple since the empirical risk minimization problem \eqref{eq:ERM} admits the following closed form solution:
\begin{align}
\label{eq:ridge:sol}
\hat{\theta} = \sfrac{1}{\sqrt{k}}(\lambda I_{k}+\sfrac{1}{k}X_{L}X_{L}^{\top})^{-1}X_{L}y
\end{align}
\noindent where we recall the reader $X_{L}\in\mathbb{R}^{k\times n}$ is the matrix obtained by stacking the last layer features column-wise and $y\in\mathbb{R}^{n}$ is the vector of labels. For a given target function, computing the test error boils down to a random matrix theory problem depending on variations of the trace of deterministic matrices times the resolvent of the features sample covariance matrices (c.f. App.\, \ref{App:error:ridge} for a derivation):
\begin{align}
\label{eq:ridge:rmt}
\epsilon_g(\hat{\theta}) &= \Delta\left(\left\langle\Omega_{L}\left(\lambda I_{k}+\sfrac{1}{k}X_{L}X_{L}\right)^{-1}\right\rangle+1\right)\notag\\
&\qquad-\lambda(\lambda-\Delta)\partial_{\lambda}\left\langle\Omega_{L}\left(\lambda I_{k}+\sfrac{1}{k}X_{L}X_{L}\right)^{-1}\right\rangle
\end{align}
Applying Theorem \ref{prop:mult_layers} yields the following corollary:
\begin{corollary}[Ridge universality of matched target] 
\label{prop:universality:ridge}
Let $\lambda > 0$. 
In the asymptotic limit $n, d, k_\ell\xrightarrow{}\infty$ with fixed $\mathcal{O}(1)$ ratios $\alpha = \sfrac{n}{d}$, $\gamma_\ell\coloneqq\sfrac{k_\ell}{d}$ and under the assumptions of Theorem \ref{prop:mult_layers}, the asymptotic test error of the ridge estimator \eqref{eq:ridge:sol} on the target \eqref{eq:target} with $L=L^{*}$ and $\varphi_{\ell}^{*} = \varphi_{\ell}$ and additive Gaussian noise with variance $\Delta>0$ is given by:
\begin{align}
\label{eq:error:ridge}
\epsilon_g(\hat{\theta}) \xrightarrow{k\to\infty} \epsilon_g^{\star} &= \Delta\left(\langle\Omega_{L}\rangle\wc m_{L}(-\lambda)+1\right)\notag\\
&\qquad-\lambda(\lambda-\Delta)\langle\Omega_{L}\rangle\partial_{\lambda}\wc m_{L}(-\lambda)
\end{align}
where $\wc m_{L}$ can be recursively computed from \eqref{eq:companion_stieltjes} respectively. In particular, this implies Gaussian universality of the asymptotic mean-squared error in this model, since \eqref{eq:error:ridge} exactly agrees with the asymptotic test error of ridge regression on Gaussian data $\x\sim\mathcal{N}(0_{d},\Omega_{L})$ derived in \cite{Dobriban2015HighDimensionalAO}.
\end{corollary}

A detailed derivation of \eqref{eq:ridge:rmt} and Corollary \ref{prop:universality:ridge} is given in App.~\ref{App:error:ridge}, together with a discussion of possible extensions to deterministic last-layer weights and general targets. Note that, while it is not needed to establish the Gaussian equivalence of ridge dRF regression in the well-specified case, the trace of the population covariance $\langle\Omega_{L}\rangle$ can be explicitly computed from the closed-form formula \eqref{eq:Omega_lin_recursion}.

\subsection{General case}

Despite the major progress stemming from the application of the random matrix theory toolbox to learning problems, the application of the latter has been mostly limited to quadratic problems where a closed-form expression of the estimators, such as \eqref{eq:ridge:sol}, are available. Proving universality results akin to Corollary \ref{prop:universality:ridge} beyond quadratic problems is a challenging task, which has recently been the subject of intense investigation. In the context of generalized linear estimation \eqref{eq:ERM}, universality of the test error for the $L=1$ random features model under a generic convex loss function was heuristically studied in \cite{Gerace2020GeneralisationEI}, where the authors have shown that the asymptotic formula for the test error obtained under the Gaussian design assumption perfectly agreed with finite-size simulations with the true features. This Gaussian universality of the test error was later proven by \cite{Hu2020UniversalityLF} by combining a Lindeberg interpolation scheme with a generalized central limit theorem. Our goal in the following is to provide an analogous contribution as \cite{Gerace2020GeneralisationEI} to the case of multi-layer random features. This result builds on a rigorous, closed-form formula for the asymptotic test error of misspecified generalized linear estimation in the high-dimensional limit considered here, which was derived in \cite{Loureiro2021CapturingTL}. 

We show that in the high-dimensional limit the asymptotic test error for the model introduced in~\Cref{sec:setting} is in the \emph{Gaussian universality class}. More precisely, the test error of this model is asymptotically equivalent to the test error of an equivalent Gaussian covariate model (GCM) consisting of doing generalized linear estimation on a dataset $\check{\mathcal{D}}=\{v^\mu,\check{y}^\mu\}_{\mu\in[n]}$ with labels $\check{y}^\mu=f_\star(\sfrac{1}{\sqrt{k^\star}}\theta_\star^\top u^\mu)$ and jointly Gaussian covariates: 
\begin{equation}
\label{eq:g3m}
   (u,v)\sim \mathcal{N}\left(
\begin{array}{cc}
     \Psi_{L^\star}&\Phi_{L^\star L}  \\
     \Phi_{L^\star L} ^\top &\Omega_L 
\end{array}
\right) 
\end{equation}
where we recall $\Omega_L$ is the variance of the model features \eqref{eq:def_pop_cov} and $\Phi\in\mathbb{R}^{k^\star \times k}$ and $\Psi\in\mathbb{R}^{k^\star\times k^\star}$ are the covariances between the model and target features and the target variance respectively:
\begin{align}
\label{eq:Phi}
    \Phi_{L^\star L}\coloneqq\E\left[\varphi^{\star}(\x)\varphi(\x)^\top\right],~ \Psi_{L^\star }\coloneqq\E \left[\varphi^\star(\x) \varphi^\star(\x)^\top\right]
\end{align}
This result adds to a stream of recent universality results in high-dimensional linear estimation \cite{Loureiro2021CapturingTL, Montanari2022UniversalityOE, Gerace2022}, and generalizes the random features universality of \cite{Mei2021GeneralizationEO, Goldt2021TheGE, Hu2020UniversalityLF} to $L>1$. It can be summarized in the following conjecture:
\begin{figure}[t!]

    \centering
     \includegraphics[scale=.46]{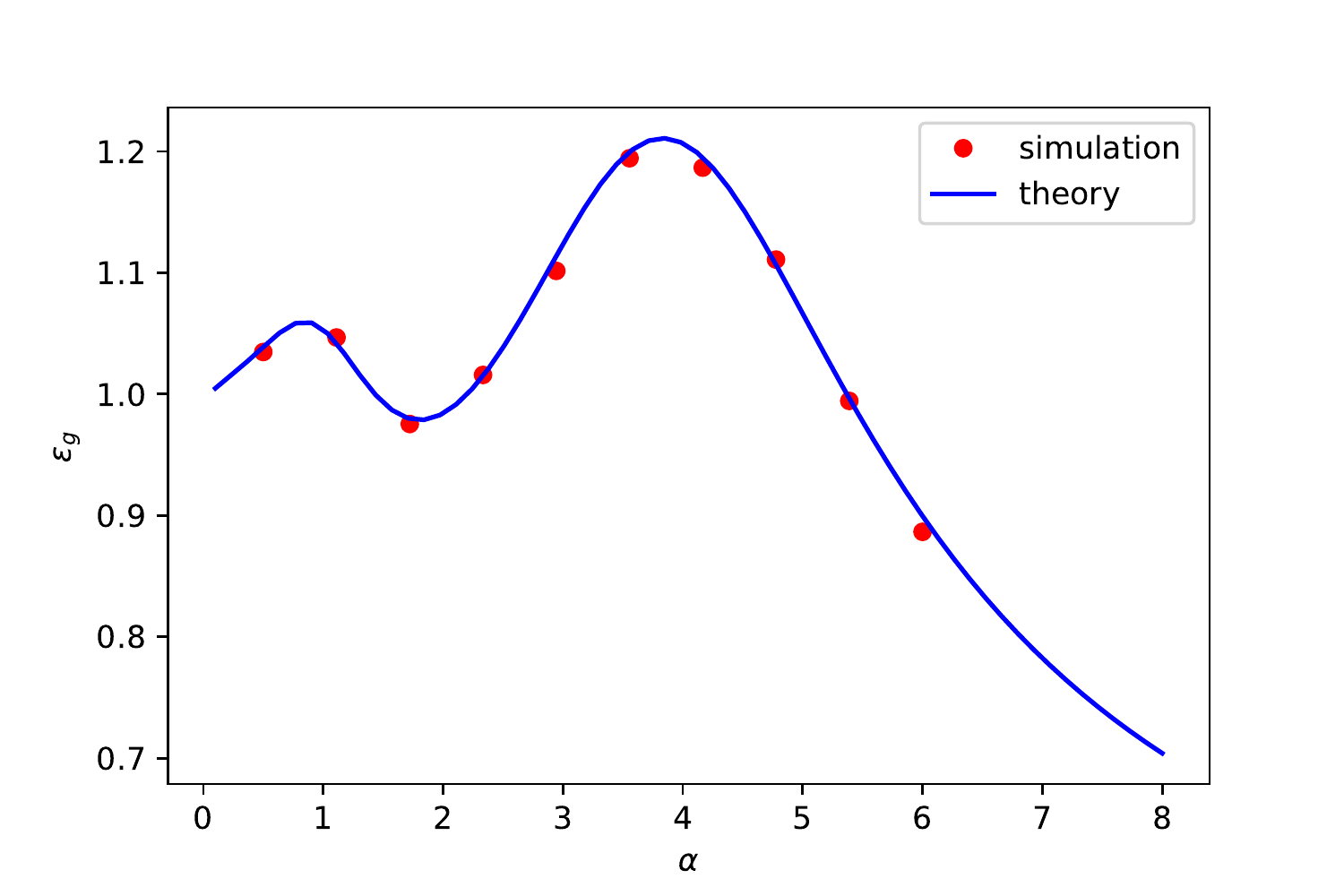}
    \includegraphics[scale=.46]{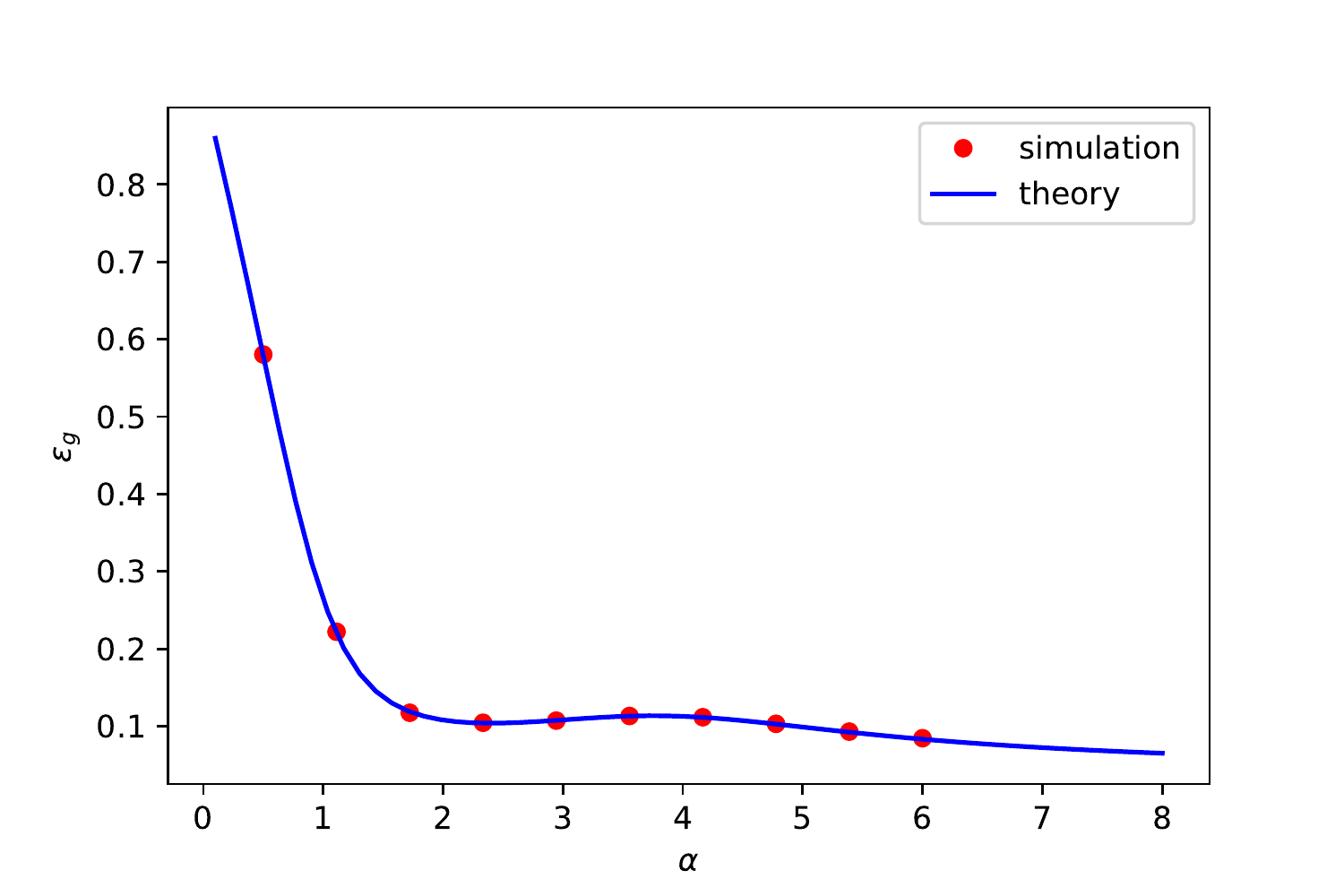}

    \caption{Learning curves $\epsilon_g(\alpha)$ for ridge regression ($\sigma_\star=id$,  $\ell(y,z)=\sfrac{1}{2}(y-z)^2$, and $g(y,\hat{y})=(y-\hat{y})^2$)
    . Red dots correspond to numerical simulations on the learning model \eqref{eq:definition_multilayer_RF} \eqref{eq:target}, averaged over $20$ runs. The solid line correspond to sharp asymptotic characterization provided by conjecture \ref{conj:error_uni_lin}, and detailed in App.\,\ref{App:error:general}. (left) 2-layers target ($L^\star=1$,$\sigma^\star_1=\mathrm{sign}$), (right) single-layer target ($L^\star=0$). Both are learnt with a $2-$hidden layers RF \eqref{eq:definition_multilayer_RF} with $\sigma_{1,2}(x)=\tanh(2x)$ activation and regularization $\lambda=0.001$.}
    \label{fig:regression}
\end{figure}

\begin{figure}

    \centering
    \includegraphics[scale=.46]{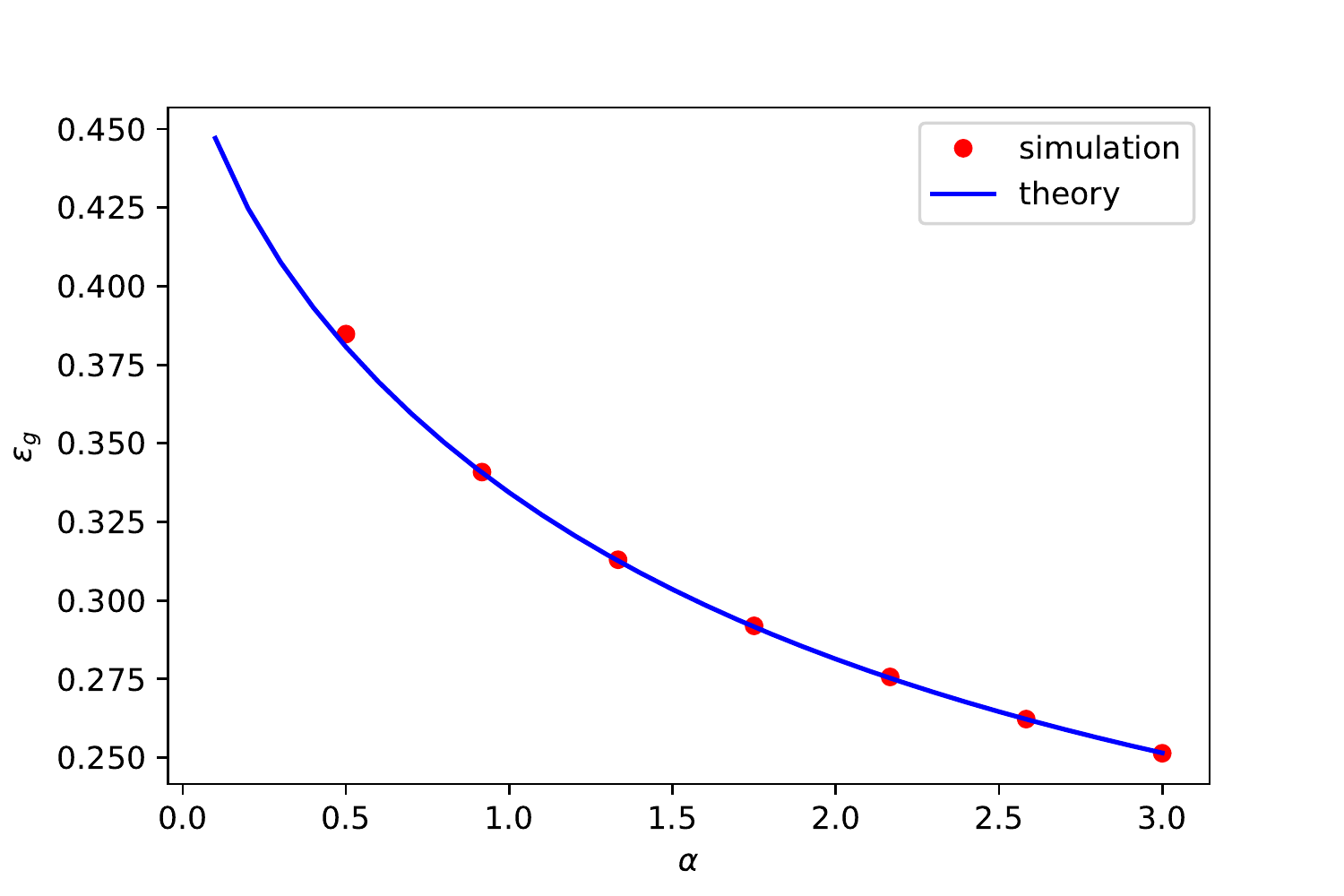}
    \includegraphics[scale=.46]{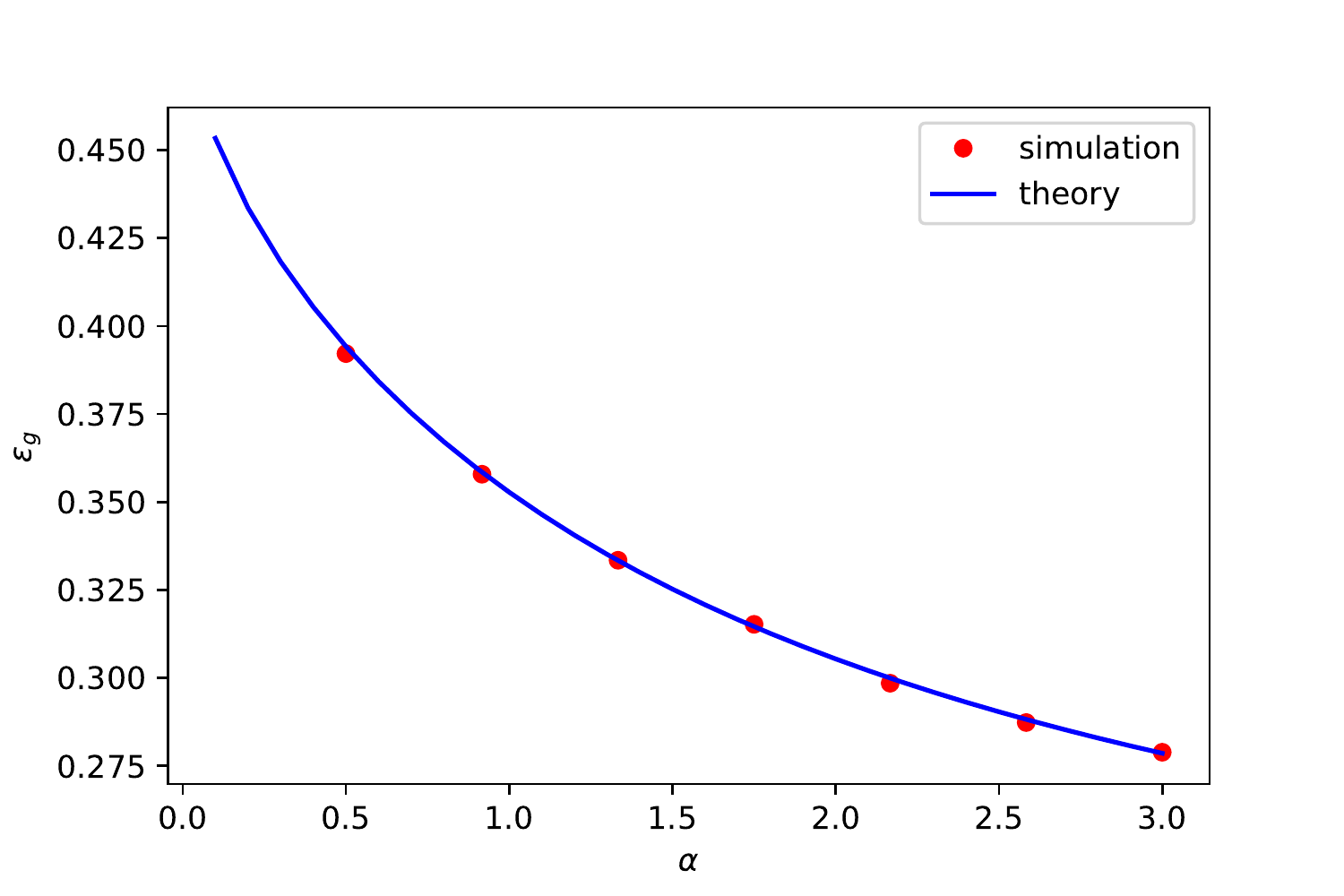}

    \caption{Learning curves $\epsilon_g(\alpha)$ for logistic regression ($\sigma_\star=\mathrm{sign}$, $\ell(y,z)=\ln(1+e^{-yz})$ and metric $g(y,\hat{y})=1-\Theta(y\hat{y})$). Red dots correspond to numerical simulations on the learning model \eqref{eq:definition_multilayer_RF} \eqref{eq:target}, averaged over $20$ runs. The solid line correspond to sharp asymptotic characterization provided by conjecture \ref{conj:error_uni_lin}, and detailed in App.\,\ref{App:error:general}. (left) single-layer target ($L^\star=0$), (right) two-layer target ($L^\star=1$, $\sigma^\star_1=\mathrm{erf}$) \eqref{eq:target} hidden sign layer. Both are learnt with a depth $L=2$ dRF \eqref{eq:definition_multilayer_RF} with activation $\sigma_{1,2}(x)=\tanh(2x)$ 
 and regularization $\lambda=0.05$ (top) and $\sigma_{1,2}(x)=\erf(x)$ and $\lambda=0.1$ (bottom).}
    \label{fig:logistic}
\end{figure}

\begin{conjecture}
\label{thm:gep}
In the high-dimensional limit $n,d, k_\ell\xrightarrow{}\infty$ at fixed $\mathcal{O}(1)$ ratios $\alpha\coloneqq\sfrac{n}{d}$ and $\gamma_\ell\coloneqq\sfrac{k_\ell}{d}$, the test error of the empirical risk minimizer \eqref{eq:ERM} trained on $\mathcal{D}=\{(\x^{\mu}, y^{\mu})\}_{\mu\in[n]}$ with covariates $\x^{\mu}\sim\mathcal{N}(0_{d},\Omega_{0})$ and labels from \eqref{eq:target} is equal to the one of a Gaussian covariate model \eqref{eq:g3m} with matching second moments $\Psi, \Phi, \Omega$ as defined in \eqref{eq:def_pop_cov} and \eqref{eq:Phi}.
\end{conjecture}
We go a step further and provide a sharp asymptotic expression for the test error. Construct recursively the sequence of matrices
\begin{align}
\label{eq:Psi_lin_recursion}
   \Psi^{\mathrm{lin}}_{\ell+1}=\left(\kappa_1^{\star(\ell+1) }\right)^{2}\frac{W_{\ell+1}^\star \Psi^{\mathrm{lin}}_\ell W_{\ell+1}^{\star\top}}{k_\ell^\star}+\left(\kappa_*^{\star(\ell+1)}\right)^{2}I_{k^\star_{\ell+1}}
\end{align}
with the initial condition $\Omega^{\mathrm{lin}}_0=\Psi^{\mathrm{lin}}_0\coloneqq\Omega_0$. Further define
   \begin{align}
\label{eq:Phi_lin_DRM}
    \Phi^{\mathrm{lin}}_{L^\star L}=
    \left(\prod\limits_{\ell=L^\star}^1 \frac{\kappa_1^{\star\ell}W_{\ell}^\star}{\sqrt{k_\ell^\star}}\right)\cdot \Omega_0 \cdot \left( \prod\limits_{\ell=1}^L\frac{\kappa_1^\ell W_\ell^\top}{\sqrt{k_\ell}}\right).
\end{align}
The sequence $\{\kappa_{1}^{\star\ell} \kappa_*^{\star\ell}\}_{\ell=1}^{L^\star}$ is define by \eqref{eq:kappa_multilayer} with $\sigma_\ell^\star,\Delta_\ell^\star$. In the special case $L^\star=0$, which correspond to a single-index target function, the first product in $\Phi^{\mathrm{lin}}_{L^\star L}$ should be replaced by $I_d$. This particular target architecture is also known, in the case $L=1$, as the \textit{hidden manifold model} \cite{Goldt2020ModellingTI,Gerace2020GeneralisationEI} and affords a stylized model for structured data. The present paper generalizes these studies to arbitrary depths $L$. One is then equipped to formulate the following, stronger, conjecture:\\
\begin{conjecture}
\label{conj:error_uni_lin}
In the same limit as in Conjecture \ref{thm:gep}, the test error of the empirical risk minimizer \eqref{eq:ERM} trained on $\mathcal{D}=\{(\x^{\mu}, y^{\mu})\}_{\mu\in[n]}$ with covariates $\x^{\mu}\sim\mathcal{N}(0_{d},\Omega_{0})$ and labels from \eqref{eq:target} is equal to the one of a Gaussian covariate model \eqref{eq:g3m} with the matrices $\Psi^{\mathrm{lin}}_{L^\star},\Omega^{\mathrm{lin}}_L,\Phi^{\mathrm{lin}}_{L^\star L}$ \eqref{eq:Omega_lin_recursion},\eqref{eq:Phi_lin_DRM}. 
\end{conjecture}
Conjecture \ref{conj:error_uni_lin} allows to give a fully analytical sharp asymptotic characterization of the test error, which we detail in App.\, \ref{App:error:general}. Importantly, observe that it also affords compact closed-form formulae for the population covariances $\Omega_L, \Phi_{L^\star L},\Psi_{L^\star}$. In particular the spectrum of $\Psi^{\mathrm{lin}}_{L^\star},\Omega^{\mathrm{lin}}_L$ can be analytically computed and compares excellently with empirical numerical simulations. We report those results in detail in App.\, \ref{App:population_heuristic}. Figs.~\ref{fig:regression}  and \ref{fig:logistic} present the resulting theoretical curve and contrasts them to numerical simulations in dimensions $d=1000$, revealing an excellent agreement.

\begin{figure}[t!]

    \centering
    \includegraphics[scale=.53]{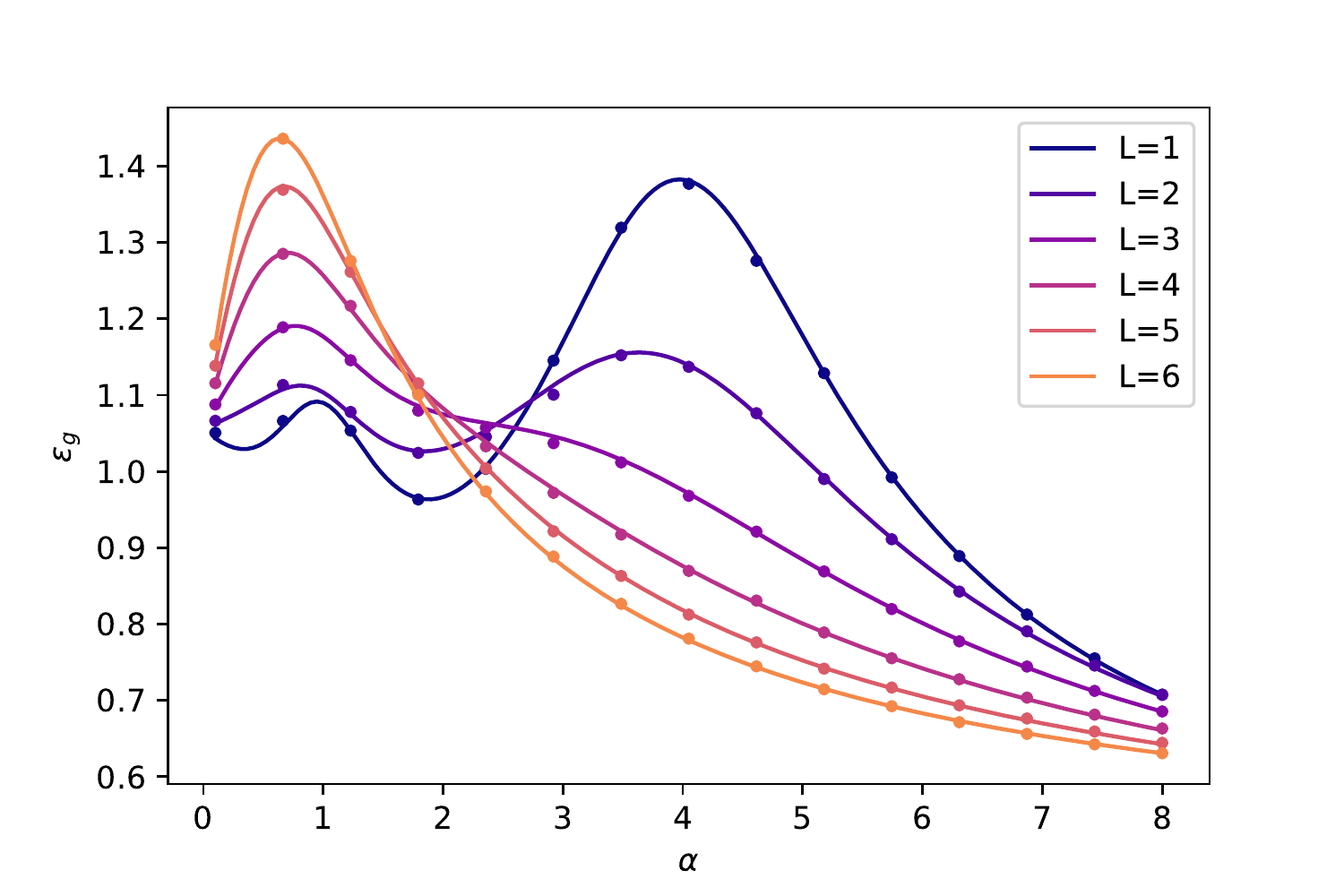}
    \includegraphics[scale=.53]{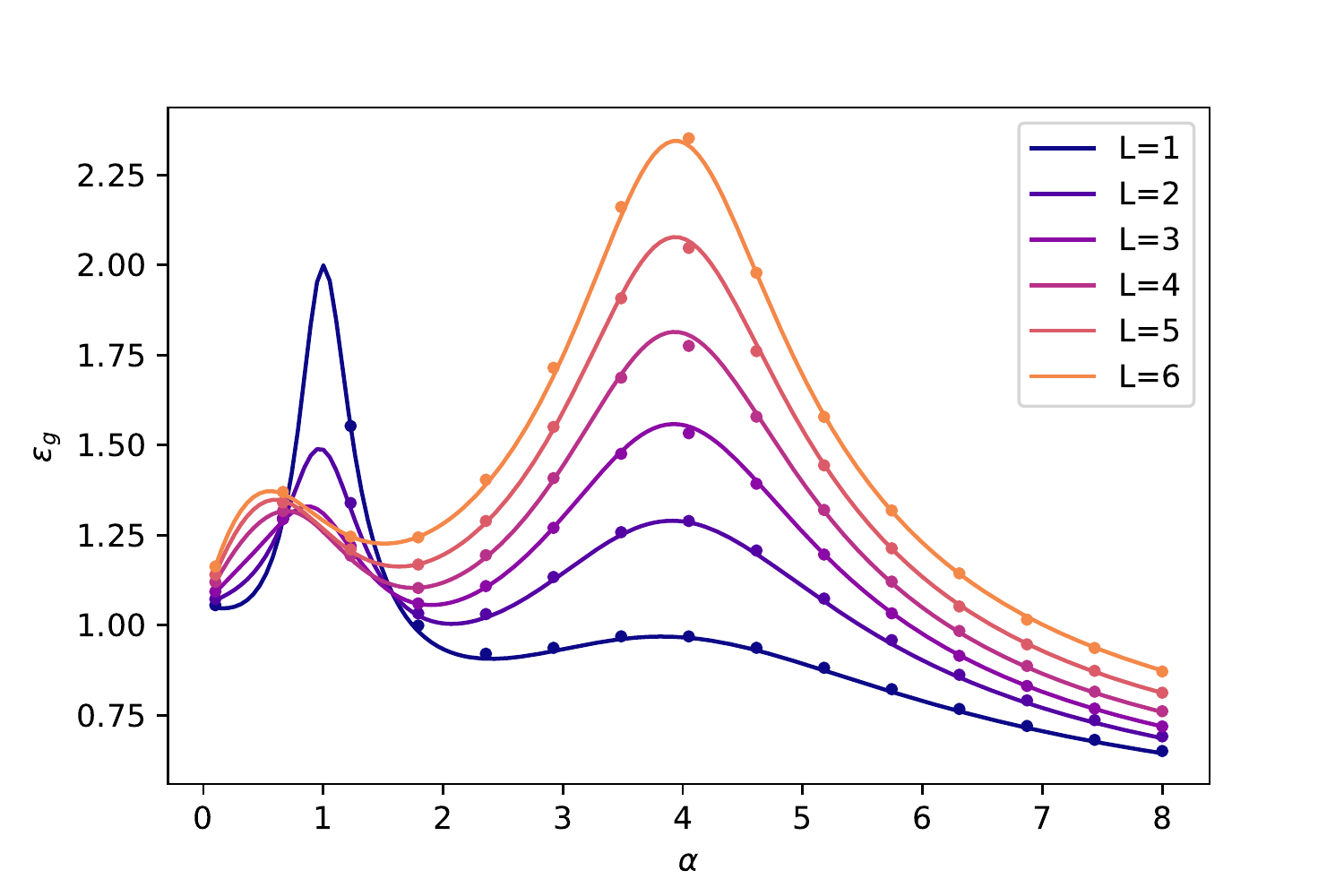}
    
    \caption{Learning curves for ridge regression on a $1$-hidden layer target function ($\gamma_1^\star=2$, $\sigma_1^\star=\mathrm{sign}$) using a $L-$hidden layers learner with widths  $\gamma_1=...=\gamma_L=4$ and $\sigma_{1,...,L}=\tanh$ activation (left) or $\sigma_{1,...,L}(x)=1.1\times \mathrm{sign}(x)\times\min(2,|x|)$ clipped linear activation (right), for depths $1\le L\le 6$. The regularization is $\lambda=0.001$. Solid lines represent theoretical curves evaluated from the sharp characterization of conjecture \ref{conj:error_uni_lin}, while numerical simulations, averaged over $50$ runs, are indicated by dots. The linear peak can be observed at $\alpha=1$, while the non-linear peak occurs for $\alpha=\gamma=4$ \cite{dAscoli2020TripleDA}. Despite sharing the same architecture, the use of different activations induces different implicit regularizations, leading to the linear (resp. non-linear) peak being further suppressed as the depth increases for the clipped linear activation (resp. tanh activation).}
    \label{fig:triple_descent}
\end{figure}
%%%%%%%%%%%%%%%%%%%%%%%%%%%%%%%%%%%%%%%%%%%%%%%%%%%%%%%%%%%%%%%%%%%%%%%%%%%%%%%
\section{Depth-induced implicit regularization}
\label{sec:architecture}
An informal yet extremely insightful takeaway from Conjecture~\ref{conj:error_uni_lin}, and in particular the closed-form expressions \eqref{eq:Omega_lin_recursion}, is that the activations in a deep non-linear dRF \eqref{eq:definition_multilayer_RF} share the same population statistics as the activations in a deep \textit{noisy} linear network, with layers 
\begin{equation}
    \label{eq:linearized_layer}
    \varphi^{\mathrm{lin}}_\ell(\x)=\kappa^\ell_1 \frac{W_\ell^\top \x}{\sqrt{k_{\ell-1}}}+\kappa_*^\ell \xi_\ell,
\end{equation}
where $\xi_\ell\sim \mathcal{N}(0_{k_\ell},I_{k_\ell})$ is a Gaussian noise term. It is immediate to see that \eqref{eq:linearized_layer} lead to the same recursion as \eqref{eq:Omega_lin_recursion}. This observation, which was made in the concomitant work \cite{Cui2023}, essentially allows to equivalently think of the problem of learning using a dRF \eqref{eq:definition_multilayer_RF} as one of learning with linear noisy network. Indeed, Conjecture \ref{conj:error_uni_lin} essentially suggests that the asymptotic test error depends on the second-order statistics of the last layer acrivations, shared between the dRF and the equivalent linear network. Finally, it is worthy to stress that, while the learner dRF is deterministic conditional on the weights $\{W_\ell\}$, the equivalent linear network \eqref{eq:linearized_layer} is intrinsically stochastic in nature due to the effective noise injection $\xi_\ell$ at each layer. Statistical common sense dictates that this effective noise injection has a regularizing effect, by introducing some randomness in the learning, and helps mitigating overfitting. Since the effective noise is a product of the propagation through a non-linear layer, this suggest that \textit{adding random non linear layers induces an implicit regularization}. We explore this intuition in this last section. 

Observe first that the equivalent noisy linear network \eqref{eq:linearized_layer} reduces to a simple shallow noisy linear model
\begin{align}
    \label{eq:linear_model}\hat{y}^{\mathrm{lin}}_\theta(\x)=\sigma\left(\frac{1}{\sqrt{k}}\theta^\top \left( A_L\cdot \x+\xi_L\right)\right)
\end{align}
where the effective weight matrix $A$ is 
$$
A_L\coloneqq\prod\limits_{\ell=1}^L\left(\kappa_1^\ell
\frac{W_\ell}{\sqrt{k_{\ell-1}}}
\right)
$$
and the effective noise $\xi_L$ is Gaussian with covariance $C_\xi^L$
$$
C_\xi^L=\sum\limits_{\ell_0=1}^{L-1}(\kappa_*^{\ell_0} )^2{\scriptstyle\left(\prod\limits_{\ell=\ell_0+1}^L \frac{\kappa_1^\ell W_\ell^\top}{\sqrt{k_{\ell-1}}}\right)^\top \left(\prod\limits_{\ell=\ell_0+1}^L \frac{\kappa_1^\ell W_\ell^\top}{\sqrt{k_{\ell-1}}}\right)}+(\kappa_*^L)^2I_{k}.
$$
The signal-plus-noise structure of the equivalent linear features \eqref{eq:linear_model} has profound consequences on the level of the learning curves of the model \eqref{eq:definition_multilayer_RF}:
\begin{itemize}
    \item When $\alpha=1$, there are as many training samples as the dimension of the data $d-$ dimensional submanifold $A_L \x$, resulting in a standard interpolation peak. The noise part $\xi_L$ induces an implicit regularization which helps mitigate the overfitting.
    \item As $\alpha=\gamma_L$, the number of training samples matches the dimension $k_L$ of the noise, and the \textit{noise} part is used to interpolate the training samples, resulting in another peak. This second peak is referred to as the non-linear peak by \cite{dAscoli2020TripleDA}.
\end{itemize}
Therefore, there exists an interplay between the two peaks, with higher noise $\xi_L$ both helping to mitigate the linear peak, and aggravating the non-linear peak. The depth of the network plays a role in that it modulates the amplitudes of the signal part and the noise part, depending on the activation through the recursions \eqref{eq:kappa_multilayer}.

We give two illustrations of the regularization effect of depth in Fig.~\ref{fig:triple_descent}. Two activations are considered : $\sigma_a=\tanh$ (for which the noise level, as measure by $\tr C_\xi^L$ decreases with depth) , and a very weakly non-linear activation $\sigma_b(x)=1.1\times \mathrm{sign}(x)\times\min(2,|x|)$, corresponding to a linear function clipped between $-2.2 $ and $2.2$ (for which $\tr C_\xi^L$ increases with depth). Note that, because for $\sigma_a$ the effective noise decreases with depth, the linear peak is aggravated for deeper networks, while the non-linear peak is simultaneously suppressed. Conversely, for $\sigma_b$, additional layers introduce more noise and cause a higher non-linear peak, while the induced implicit regularization mitigates the linear peak. Further discussion about the effect of architecture design on the generalization ability of dRFs \eqref{eq:definition_multilayer_RF} is provided in App.\, \ref{App:architecture}.
%%%%%%%%%%%%%%%%%%%%%%%%%%%%%%%%%%%%%%%%%%%%%%%%%%%%%%%%%%%%%%%%%%%%%%%%%%%%%%%
\section{Conclusion}
\label{sec:conclusion}
We study the problem of learning a deep random network target function by training the readout layer of a deep network, with frozen random hidden layers (deep Random Features).
We first prove an asymptotic deterministic equivalent for the conjugate kernel and sample covariance of the activations in a deep Gaussian random networks. This result is leveraged to establish a sharp asymptotic characterization of the test error in the specific case where the learner and teacher networks share the same intermediate layers, and the readout is learnt using a ridge loss. This proves the Gaussian universality of the test error of ridge regression on non-linear features corresponding to the last layer activations. In the fully generic case, we conjecture a sharp asymptotic formula for the test error, for fully general target/learner architectures and convex loss. The formulas suggest that the dRF behaves like a linear noisy network, characterized by an implicit regularization. We explore the consequences of this equivalence on the interplay between the architecture of the dRF and its generalization ability.
%%%%%%%%%%%%%%%%%%%%%%%%%%%%%%%%%%%%%%%%%%%%%%%%%%%%%%%%%%%%%%%%%%%%%%%%%%%%%%%
\section*{Acknowledgements}
We thank Gabriele Sicuro for discussion during the course of this project. BL acknowledges support from the \textit{Choose France - CNRS AI Rising Talents} program. DD is supported by ETH AI Center doctoral fellowship. DS is supported by SNSF Ambizione Grant \texttt{PZ00P2\_209089}. HC acknowledges support from the ERC under the European Union’s Horizon 2020 Research and Innovation Program Grant Agreement 714608-SMiLe.
%%%%%%%%%%%%%%%%%%%%%%%%%%%%%%%%%%%%%%%%%%%%%%%%%%%%%%%%%%%%%%%%%%%%%%%%%%%%%%%

\bibliographystyle{unsrt}
\bibliography{biblio}
%%%%%%%%%%%%%%%%%%%%%%%%%%%%%%%%%%%%%%%%%%%%%%%%%%%%%%%%%%%%%%%%%%%%%%%%%%%%%%%
% APPENDIX
%%%%%%%%%%%%%%%%%%%%%%%%%%%%%%%%%%%%%%%%%%%%%%%%%%%%%%%%%%%%%%%%%%%%%%%%%%%%%%%
\newpage
\appendix
%%%%%%%%%%%%%%%%%%%%%%%%%%%%%%%%%%%%%%%%%%%%%%%%%%%%%%%%%%%%%%%%%%%%%%%%%%%%%%%
\section{Anisotropic deterministic equivalent}
\label{App:anisotropic}
% \documentclass[9pt,reqno]{amsart}  
% %
% %
% \usepackage[english]{babel}
% \usepackage[p,osf]{cochineal}
% \usepackage[scale=.95,type1]{cabin}
% \usepackage[zerostyle=c,scaled=.94]{newtxtt}
% \usepackage[T1]{fontenc} 
% \usepackage[utf8]{inputenc} 
% \usepackage{amsthm}
% \usepackage{amsmath}
% \usepackage{amssymb}
% \usepackage{amsfonts}
% \usepackage{amsaddr}
% \usepackage{enumitem} 
% \usepackage{csquotes}
% \usepackage[colorlinks]{hyperref}
% \usepackage[nosort,capitalise]{cleveref}
% \crefname{enumi}{}{}
% \crefname{equation}{}{}
% \usepackage{mathtools}
% \mathtoolsset{centercolon}
% \usepackage{bm}
% \DeclareMathOperator{\Tr}{Tr}
% \DeclareMathOperator{\erfc}{erfc}
% \DeclareMathOperator{\He}{He}
% \DeclareMathOperator{\sgn}{sgn}
% \DeclareMathOperator{\supp}{supp}
% \DeclareMathOperator{\rank}{rank}
% \DeclareMathOperator{\Cov}{Cov}
% \DeclareMathOperator{\Var}{Var}
% \DeclareMathOperator{\E}{\mathbf{E}}
% \DeclareMathOperator{\spec}{spec}
% \DeclareMathOperator{\diag}{diag}
% \DeclareMathOperator{\dist}{dist}
% \DeclareMathOperator{\Spec}{Spec}
% \ifdefined\C
% \renewcommand{\C}{\mathbf{C}}
% \else
% \newcommand{\C}{\mathbf{C}}
% \fi
% \newcommand{\HC}{\mathbf{H}}
\newcommand{\cX}{\mathcal X}
% \newcommand{\cO}{\mathcal O}
% \newcommand{\ii}{\mathrm i}
% \newcommand{\wh}{\widehat}
% %% code from mathabx.sty and mathabx.dcl
% \DeclareFontFamily{U}{mathx}{\hyphenchar\font45}
% \DeclareFontShape{U}{mathx}{m}{n}{
%       <5> <6> <7> <8> <9> <10>
%       <10.95> <12> <14.4> <17.28> <20.74> <24.88>
%       mathx10
%       }{}
% \DeclareSymbolFont{mathx}{U}{mathx}{m}{n}
% \DeclareFontSubstitution{U}{mathx}{m}{n}
% \DeclareMathAccent{\wc}{0}{mathx}{"71}
% \DeclareMathAccent{\wideparen}{0}{mathx}{"75}

% \def\cs#1{\texttt{\char`\\#1}}
% \newcommand{\wt}{\widetilde}
\newcommand{\C}{\mathbb{C}}
\providecommand\given{}

\subsection{Sample covariance matrices}
Consider a random vector \(\vx\in\R^d\) with \(\E \vx = 0\) and \(\E \vx \vx^\top=\Sigma\) and for \(n\in\N\) construct \(\cX=(\vx_1,\ldots,\vx_n)\in\R^{d\times n}\) using \(n\) independent copies \(\vx_1,\ldots,\vx_n\) of \(\vx\). We are interested in the sample covariance and Gram matrices 
\begin{equation}
    \wh\Sigma := \frac{\cX\cX^\top}{n}=\frac{1}{n}\sum_{i=1}^n\vx_i\vx_i^\top\in\R^{d\times d} \qquad \text{and} \qquad \wc\Sigma := \frac{\cX^\top\cX}{n}=\Bigl(\frac{\vx_i^\top \vx_j}{n}\Bigr)_{i,j=1}^n \in \R^{n\times n}
\end{equation}
and their resolvents 
\begin{equation}
    \wh G(z):=(\wh\Sigma-z)^{-1} \in \C^{d\times d} \qquad \text{and} \qquad \wc G(z):=(\wc\Sigma-z)^{-1} \in \C^{n\times n}.
\end{equation}
The expectations of the sample covariance and Gram matrices are
\begin{equation}
    \E \wh\Sigma = \Sigma, \quad \E \wc\Sigma = \frac{d}{n}\langle{\Sigma}\rangle I_n,
\end{equation}
where we introduced the averaged trace \(\langle{A}\rangle:=m^{-1}\Tr A\) for \(A\in\R^{m\times m}\).

 Note that while the two resolvents behave differently as matrices, their traces are related due to the fact that the non-zero eigenvalues of \(\wh\Sigma\) and \(\wc\Sigma\) agree, whence
\begin{equation}
    \langle{\wh G(z)}\rangle = \frac{n}{d}\langle{\wc G(z)}\rangle + \frac{n-d}{pz}.
\end{equation}
The classical result on normalised traces of sample covariance and Gram resolvents is the following variance estimate under essentially optimal conditions.
%\begin{equation}
%    \braket{\wc G(z)} = \frac{d}{n}\braket{\wh G(z)} + \frac{d-n}{nz} \approx \frac{d}{n} \frac{\kappa}{z}\braket{(\Sigma-\kappa)^{-1}} + \Bigl(\frac{d}{n}-1\Bigr)\frac{1}{z} = -\frac{1}{\kappa}
%\end{equation}
\begin{theorem}[Tracial convergence of sample covariance matrices with general population~\cite{10.2307/24308489}]\label{thm samp cov tr}
    Assume that \(\norm{\Sigma}\lesssim 1\), \(d/n\sim 1\) and that 
    \begin{equation}\label{quad form assump var}
        \E\left|{\frac{\vx^\top A \vx}{d} - \E \frac{\vx^\top A \vx}{d} }\right|^2 = \E\left|{\frac{\vx^\top A \vx}{d} - \langle{\Sigma A}\rangle}\right|^2 = o(\norm{A})
    \end{equation}
    for all deterministic matrices \(A\). Then it holds that 
\begin{equation}\label{resolvent conv scov}
    \E\abs{\langle{(\wh\Sigma-z)^{-1}}\rangle- \wh m(z)}^2 = o(1), \quad \E\abs{\langle{(\wc\Sigma-z)^{-1}}\rangle -\wc m(z)}^2 = o(1), \quad \text{as} \quad n,d\to\infty,
\end{equation}
for all fixed \(z\in\C\setminus\R_+\), where
%where \(m=m(z)\) is the unique solution to the equation
%\begin{equation}\label{MP eq}
%    m = \frac{1}{1-d/n-pz m/n} \braket*{\Bigl(\Sigma-\frac{z}{1-d/n-pzm/n}\Bigr)^{-1}}.
%\end{equation}
%and 
\(\wc m=\wc m(z)\) is the unique solution to the scalar equation 
\begin{equation}\label{kappa eq}
    1 - \frac{d}{n} + z\wc m =
    -\frac{d}{n}\langle{(\Sigma\wc m+ 1)^{-1}}\rangle.
\end{equation}
and 
\begin{equation}
    \wh m(z) := \frac{n}{d} \wc m(z) + \frac{d-n}{d}\frac{1}{-z} 
\end{equation}
\end{theorem}
%The two solutions \(m,\kappa\) are of course directly related and can be obtained from each other by 
%\begin{equation}
%    m(z):=\frac{\kappa(z)}{z} \braket{(\Sigma-\kappa(z))^{-1}}
%\end{equation}
%and
%\begin{equation}
%    \kappa(z)=\frac{z}{1-c-cz m(z)}.
%\end{equation}
Here \(\wh m\) is the solution to the Marchenko-Pastur equation~\eqref{SC MP} and the correspodning measure is the free multiplicative convolution of the empirical spectral measure \(\mu(\Sigma):=d^{-1}\sum_{\lambda\in\Spec(\Sigma)}\delta_\lambda\) of \(\Sigma\) and a Marchenko-Pastur distribution \(\mu_\mathrm{MP}^c\) of aspect ratio \(c=d/n\).
%\begin{definition}\label{def MP}
%    The multiplicative free convolution \(\mu=\nu\boxtimes \mu_\mathrm{MP}^c\) of the \(c\)-Marchenko-Pastur distribution \(\mu_\mathrm{MP}^c\) and \(\nu\) is the unique probability measure \(\mu\) whose Stieltjes transform 
 %   \begin{equation}
 %       m_\mu(z):=\int_\R \frac{1}{x-z}\dif\mu(x)
 %   \end{equation}
 %   satisfies 
 %   \begin{equation}\label{MP mult conv}
 %       1-c+zm_\mu(z) = -c \int_\R \frac{1}{t m_\mu +1}\dif\nu(t)
 %       %m_\mu(z) = \int_\R \frac{1}{t(1-c-cz m_\mu(z))- z}\dif\nu(t) = (1-c-cz %m_\mu(z))^{-1}m_\nu\Bigl(\frac{z}{1-c-cz m_\mu(z)}\Bigr).
 %   \end{equation}
 %   The Marchenko-Pastur distribution \%(\mu_\mathrm{MP}^{c}:=\delta_1\boxtimes\mu_\mathrm{MP}^c\) itself is the unique probability measure whose Stieltjes transform satisfies~\cref{MP mult conv} with \(\nu=\delta_1\).
%\end{definition}
Thus, by Stieltjes inversion the result of~\Cref{thm samp cov tr} can be phrased as
%from~\Cref{thm samp cov tr} and~\Cref{def MP} it follows that the Stieltjes transform of the empirical measure 
%\begin{equation}
%    \mu(\wh\Sigma):=\frac{1}{d}\sum_i \delta_{\lambda_i}, \quad m_{\mu(\wh\Sigma)}(z)=\braket{(\wh\Sigma-z)^{-1}}
%\end{equation}
%satisfies 
%\begin{equation}
%    \E\abs{ m_{\mu(\wh\Sigma)}(z)- m_{\mu(\Sigma)\boxtimes\mu_\mathrm{MP}^c}(z)}^2 = o(1),
%\end{equation}
%and therefore by Stieltjes inversion 
\begin{equation}
    \mu\Bigl(\frac{\cX^\top\cX}{n}\Bigr)=\mu(\wc\Sigma) \approx \frac{d}{n}\mu(\Sigma)\boxtimes\mu_\mathrm{MP}^{d/n} + \frac{n-d}{n}\delta_0, \quad \mu\Bigl(\frac{\cX\cX^\top}{n}\Bigr)=\mu(\wh\Sigma) \approx \mu(\Sigma)\boxtimes\mu_\mathrm{MP}^{d/n} %\quad \mu\Bigl(\frac{\cX^\top\cX}{n}\Bigr)\approx \frac{d}{n}\mu(\Sigma)\boxtimes\mu_\mathrm{MP}^{d/n} + \frac{n-d}{n}\delta_0
\end{equation}
in a weak and global sense. 
Note that we have the limits 
\begin{equation}
    \lim_{c\to\infty} \mu(\Sigma)\boxtimes\mu_\mathrm{MP}^c = \delta_0, \quad \lim_{c\to 0} \mu(\Sigma)\boxtimes\mu_\mathrm{MP}^c = \mu(\Sigma)
\end{equation}
which are precisely the expected behaviour since for large \(c=d/n\) the rank \(n\) of \(\cX\cX^\top\) grows much smaller than \(d\) and therefore the empirical measure \(\mu(\wh\Sigma)\) is concentrated on the origin, while for small \(c=d/n\) by the law of large numbers \(\cX\cX^\top/n\approx \E \cX\cX^\top/n=\Sigma\).

\subsection{Anisotropic deterministic equivalents}
The tracial result from~\Cref{thm samp cov tr} only allows to control the eigenvalues of \(\wh\Sigma,\wc\Sigma\) but not the eigenvectors. There has been extensive work on non-tracial deterministic equivalents of \(\wh\Sigma,\wc\Sigma\), either in the form of entrywise asymptotics \(\wc G_{ij}\approx \cdots\), isotropic asymptotics \(\vx^\top\wh G\vy\approx \cdots\) for deterministic vectors \(\vx,\vy\) or functional tracial asymptotics \(\langle{A\wh G}\rangle\approx \cdots\) for deterministic matrices \(A\). Any of these results contain non-trivial information on how \(\wh G,\wc G\) behave as matrices in the asymptotic limit and can be used to infer information on eigenvectors. 

For separable correlations an optimal local law in isotropic and tracial form has been obtained in~\cite{10.1007/s00440-016-0730-4}:
\begin{theorem}[\cite{10.1007/s00440-016-0730-4}, Theorem 3.6]\label{knowles}
    If \(\cX=\Sigma^{1/2}X\) for some matrix \(X\) with independent identically distributed entries\footnote{with finite moments of all orders} with mean \(0\) and variance \(1\), and the spectral density \(\mu(\Sigma)\boxtimes \mu_\mathrm{MP}^{d/n}\) is \emph{regular}\footnote{See Definition 2.7 in~\cite{10.1007/s00440-016-0730-4}}, then it holds that 
    \begin{equation}
        \abs{\langle{(\wh\Sigma-z)^{-1}}\rangle-\wh m(z)}+\abs{\langle{(\wc\Sigma-z)^{-1}}\rangle-\wc m(z)}\prec \frac{1}{n\Im z}, 
    \end{equation}
    in tracial sense, and for any deterministic vectors \(\vx,\vy\) 
    \begin{equation}
        \abs{\vx^\top\Bigl[(\wh\Sigma-z)^{-1} -(-\Sigma\wc m(z)z-z)^{-1}\Bigr]\vy} + \abs{\vx^\top(\wc\Sigma-z)^{-1}\vy -\wc m(z) \vx^\top\vy} \prec \frac{\norm{x}\norm{y}}{\sqrt{n\Im z}}
    \end{equation}
    in isotropic sense. 
\end{theorem}
Note that in particular, matrix \(\wh G(z)\) asymptotically is equal to a resolvent 
\begin{equation}
    \wh M(z):= \Bigl(-\Sigma \wc m(z) z-z\Bigr)^{-1}
\end{equation}
of the population covariance \(\Sigma\), while \(\wc G\) asymptotically is a scalar multiple of the identity. 

More recently a functional tracial local law (albeit with very much suboptimal dependence on the spectral parameter) for \(\wh G\) has been obtained in~\cite{chouard2022quantitative}:

\begin{theorem}[\cite{chouard2022quantitative}, Proposition 2.4]
    \label{thm:res_concentration}
    If \(\norm{\Sigma}\le C\) and \(\cX\) satisfies, for some positive constants \(c, C, \sigma\)
    \begin{equation}
    P(|f(\cX) - \E f(\cX)| \geq t) \leq Ce^{-c(t/\sigma)^2}\quad \forall \text{ 1-Lipschitz } f: (\R^{d \times n}, \norm{\cdot}_F) \to (\R, |\cdot|),
    \end{equation}
    we have that for all deterministic matrices $A$ and $|z| \lesssim 1$ with high probability\footnote{The statement in~\cite{chouard2022quantitative} literally gives \(\Im z\) rather than \(\dist(z,\R_+)\) but the proof verbatim gives the stronger bound since \(\Im z\) is merely used as a lower bound on the smallest singular value of a matrix of the type \(AA^\ast-z\)}.
    \begin{equation}
    \label{eq:res_conc_chouard}
    \left | \langle{ A (\wh{\Sigma} - z)^{-1} -  A (-\wc m(z)z\Sigma - z)^{-1}}\rangle \right | \le \frac{\sqrt{\langle A A^{\ast} \rangle \log n }}{n\dist(z,\R_+)^9},
    \end{equation}
    where \(\wc m=\wc m(z)\) is the unique solution to the scalar equation 
\begin{equation}\label{eq:m_check_chouard}
    1 - \frac{d}{n} + z\wc m =
    -\frac{d}{n}\langle{(\Sigma\wc m+ 1)^{-1}}\rangle.
\end{equation}
\end{theorem}
Note that the functional tracial formulation with convergence rate \(1/n\) and error in terms of the Frobenius norm of \(A\) automatically includes an isotropic local law as a special case. Indeed, for \(A=\vx\vy^\top\) it follows that
\begin{equation}\label{iso local}
    \vy^\top\Bigl((\hat\Sigma-z)^{-1}-(-\wc m z\Sigma-z)^{-1}\Bigr)\vx \prec \frac{\norm{\vx}\norm{\vy}}{\sqrt{n}\delta^9},
\end{equation}
where we denote here and in the future \(\delta \equiv \delta(z) := \mathrm{dist}(z, \R_+)\).
In this work we extend the functional tracial local law from~\cite{chouard2022quantitative} to the case of \(\wc G\) and obtain the following result:
\begin{proposition}[Functional local law for Gram matrices]\label{co resolvent}
    Under the assumptions of~\Cref{thm:res_concentration} we have that 
    \begin{equation}
        \abs{\langle{A(\wc\Sigma-z)^{-1}}\rangle  - \wc m(z)\langle{A}\rangle } \prec \frac{\langle{AA^\ast}\rangle^{1/2}}{\delta^9\sqrt{n}}. 
    \end{equation}
\end{proposition}
Note that the bound in~\Cref{co resolvent} is weaker than the bound in~\Cref{thm:res_concentration}, and both results are very much weaker than~\Cref{knowles} in the dependence on the spectral parameter. In light of related results it is natural to conjecture the following:  
\begin{conjecture}
    \label{conj opt local law}
    Assume that quadratic forms of \(\vx\) concentrate as 
    \begin{equation}
        \abs{\frac{\vx^\top A \vx}{d} -\langle{\Sigma A}\rangle} \prec \frac{\langle{AA^\ast}\rangle^{1/2}}{\sqrt{d}}
    \end{equation}
    for any deterministic matrix \(A\), and that \(\norm{\Sigma}\lesssim 1\). Then we have the functional tracial estimates 
    \begin{equation}
        \begin{split}
            \abs{\langle{zA(\wh\Sigma-z)^{-1}}\rangle - \langle{A(-\wc m(z)\Sigma-I)^{-1}} }&\prec \frac{\langle{AA^\ast}\rangle^{1/2}}{n\delta}\\
            \abs{\langle{A(\wc\Sigma-z)^{-1}}\rangle - \wc m(z)\langle{A}\rangle}&\prec \frac{\langle{AA^\ast}\rangle^{1/2}}{n\delta}.
        \end{split}
    \end{equation} 
\end{conjecture}
Note that the Lipschitz concentration required in Theorem~\ref{thm:res_concentration} is much stronger than the quadratic form concentration of Conjecture~\ref{conj opt local law} because it implies that the column vectors \(\vx\) of \(\cX\) satisfy 
\begin{equation}
    P(\abs{f( x)-\E f( x)} \geq t) \leq C \exp\Bigl(-\frac{t^2}{C\lambda_f^2}\Bigr)
\end{equation}
for all \(\lambda_f\text{-Lipschitz }f\colon\R^{d}\to\R\). Therefore by Hanson-Wright~(\cite{adamczak2015note}, Thm. 2.4)
\begin{equation}
    P\Bigl(\abs{\frac{\vx^\top A \vx}{d} -\langle{\Sigma A}\rangle} \geq \frac{t\langle{AA^\ast}\rangle^{1/2}}{\sqrt{d}}+\frac{t\norm{A}}{d}\Bigr) \leq C e^{-\min\{t^2,t\}/C}
\end{equation}
and, since also \(\norm{A}\le \sqrt{d}\langle{AA^\ast}\rangle^{1/2}\), we have that with high probability 
\begin{equation}\label{high prob quad form}
    \abs{\frac{\vx^\top A \vx}{d} -\langle{\Sigma A}\rangle} \le \log d\Bigl(\frac{\langle{AA^\ast}\rangle^{1/2}}{\sqrt{d}}+\frac{\norm{A}}{d} \Bigr)\lesssim \log d \frac{\langle{AA^\ast}\rangle^{1/2}}{\sqrt{d}}.
\end{equation}
Let us now turn to the proof of Proposition~\ref{co resolvent}. %Since sample covariance matrix and Gram matrix are closely connected, the proof will essentially follow from Theorem~\ref{thm:res_concentration}.
We will need the following result of Lipschitzness of the resolvent function, see e.g.~\cite{chouard2022quantitative}
\begin{lemma}
\label{lemma:gram_lipschitz}
The map \(\wc G: \cX \to \left(\cX^\top \cX / n 
 - z\right)^{-1}\) is \((3 \delta^{-2} |z|^{1/2} n^{-1/2})\)-Lipschitz with respect to Frobenius norm.
\end{lemma}
%\begin{proof}
%Let \(A, B \in \R^{d \times n}\). Using the identity \(U^{-1} - V^{-1} = U^{-1}(V - U)V^{-1}\), we can write\begin{equation}
% \wc G(A) - \wc G(B) = \frac{1}{n}\wc G(A) \left(B^\top B - A^\top A\right) \wc G(B) = \frac{1}{n} \wc G(A) \left((B - A)^\top B + A^\top (B - A)\right) \wc G(B),
% \end{equation}
% and therefore, using inequality \(\norm{UV}_F \leq \norm{U}\norm{V}_F\), we bound
% \begin{equation}
% \norm{\wc G(A) - \wc G(B)}_F \leq \frac{1}{n}\norm{A - B}_F \left(\norm{\wc G(A)}\norm{\wc G(B) B^\top} + \norm{\wc G(B)}\norm{\wc G(A) A^\top}\right)
% \end{equation}
% Since for any \(U \in \R^{d \times n}\), \(\norm{\wc G(U)} \leq \delta^{-1}\), we have
% \begin{equation}
%     \norm{\wc G(U) U^\top}^2 = \norm{\wc G(U) U^\top U \wc G(U)} \leq n\delta^{-1} \norm{I + z \wc G(U)} \leq n \delta^{-1}(1 + |z| \delta^{-1}). 
% \end{equation}
% Collecting the terms and using the bound \(|z|\delta^{-1} \geq 1\), we have
% \begin{equation}
%     \norm{\wc G(A) - \wc G(B)}_F \leq \norm{A - B}_F  3 \delta^{-2} |z|^{1/2} n^{-1/2}.
% \end{equation}
% \end{proof}
\begin{proof}[Proof of~\Cref{co resolvent}]
Denote \(\wc m \equiv \wc m(z)\). By the Schur complement formula we have%~\cite{Bao-Yee}
\begin{equation}\label{check Gii}
    \begin{split}
        \wc G_{ii} &= - \Bigl(z+z \frac{\vx_i^\top \hat G^{(i)}\vx_i}{n} \Bigr)^{-1} = -\Bigl(z+z c \langle{\Sigma \hat G^{(i)}}\rangle \Bigr)^{-1} + O\Bigl(\frac{1}{\sqrt{n}\delta^9}\Bigr)= \wc m + O\Bigl(\frac{1}{\sqrt{n}\delta^9}\Bigr),
    \end{split}
\end{equation}
using 
\begin{equation}
    \langle{\Sigma \hat G^{(i)}}\rangle = \langle{\Sigma\wc G}\rangle + \frac{1}{n}\left\langle{\Sigma \frac{ \hat G^{(i)}\vx_i\vx_i^\top  \hat G^{(i)}}{1+\vx_i  \hat G^{(i)}\vx_i/n}}\right\rangle = 
    %\frac{\kappa}{z}\langle{\Sigma(\Sigma-\kappa)^{-1}} + O\Bigl(\frac{1}{n\delta^9}\Bigr)
    -\frac{1}{z}\langle{\Sigma(\wc m\Sigma+I)^{-1}}\rangle + O\Bigl(\frac{1}{n\delta^9}\Bigr)
\end{equation}
and
\begin{equation}
    z-c \langle{\Sigma (\wc m \Sigma + I)^{-1}}\rangle=z - \frac{c}{\wc m} + \frac{c}{\wc m^2} \langle{(\wc m \Sigma + I)^{-1}}\rangle = z -\frac{c}{\wc m} - \frac{1}{\wc m}\Bigl(1-c+z \wc m\Bigr)=- \frac{1}{\wc m}
\end{equation}
in the last step. 
Next, for off-diagonal elements we have, again by Schur-complement, that
\begin{equation}
    \begin{split}
        \wc G_{ij} = z \wc G_{ii} \wc G^{(i)}_{jj} \frac{\vx_i^\top \hat G^{(ij)}\vx_j}{n} =  z \wc G_{ii} \Bigl(\wc G_{jj} -\frac{\wc G_{ij}\wc G_{ji}}{\wc G_{jj}}\Bigr) \frac{\vx_i^\top \hat G^{(ij)}\vx_j}{n}.
    \end{split}
\end{equation}
Here from the first equality already a bound size \(n^{-1/2}\delta^{-4}\) follows. Thus, together with~\cref{check Gii} it follows that 
\begin{equation}
    \wc G_{ij} = \wc m^2 z \frac{\vx_i^\top \hat G^{(ij)}\vx_j}{n} + O\Bigl(\frac{1}{n\delta^8}\Bigr),
\end{equation}
and therefore by mean-zero assumption that \(\E \wc G_{ij}=O(1/n\delta^8)\). This together with~\cref{check Gii} implies that 
\begin{equation}
\label{eq:exp_gram_bound}
    \norm{\E \wc G - \wc m(z)I}_\mathrm{F} = O\Bigl( \frac{1}{\delta^9}  \Bigr).
\end{equation}
We write
\begin{equation}
\label{eq:triangle_gram}
\abs{\langle{A\wc G}\rangle  - \wc m(z)\langle{A}\rangle } \leq \abs{\langle A \wc G \rangle - \E \langle A \wc G \rangle} + \abs{\E \langle A \wc G \rangle - \wc m(z) \langle A \rangle}.
\end{equation}
Note that from Lemma~\ref{lemma:gram_lipschitz} and Cauchy-Schwarz inequality,
\begin{equation}
\text{the map } \mathcal{X} \to \left\langle A \left(\frac{\cX^\top \cX}{n} - z \right) \right\rangle \quad \text{is} \quad \frac{3|z|^{1/2}\langle A A^{\ast} \rangle^{1/2}}{n\delta^2}\text{-Lipschitz},
\end{equation}
therefore,
\begin{equation}
\label{eq:gram_to_exp_gram}
\left| \langle A \wc G \rangle - \E \langle A \wc G \rangle \right| \prec \frac{|z|^{1/2}\langle A A^{\ast} \rangle^{1/2}}{n \delta^2}
\end{equation}
Also, from~\eqref{eq:exp_gram_bound}, we have
\begin{equation}
\label{eq:exp_gram_to_det}
        \abs{\E \langle A\wc G \rangle  - \wc m(z)\langle A \rangle } \leq \frac{1}{\sqrt{n}} {\langle A A^{\ast} \rangle}^{1/2}\norm{\E \wc G - \wc m(z)I}_F \prec \frac{{\langle A A^{\ast} \rangle}^{1/2}}{\delta^9\sqrt{n}}.
\end{equation}
The statement of the Proposition follows from~\eqref{eq:triangle_gram},~\eqref{eq:gram_to_exp_gram}~and~\eqref{eq:exp_gram_to_det}.
\end{proof}

\subsection{Random feature model}
We consider a one-layer random feature model, with a scalar function \(\sigma_1(x)\) applied entrywise.
\begin{equation}
    \sigma_1\Bigl(\frac{W_1X_0}{\sqrt{d}}\Bigr), \quad X_0\in\R^{d\times n},\quad W_1\in\R^{k_1\times d}.
\end{equation}
We require the following assumptions.
\begin{assumption}[Gaussian weight]
\label{asmpt:gaussian_w}
Entries of \(W_1\) are iid.\ \(\mathcal N(0,1)\) elements.
\end{assumption}
\begin{assumption}[Orthogonal and bounded data]
\label{asmpt:norm_bounds}
For a positive constant \(r_1\), \(X_0\) satisfies
\begin{equation}
\begin{aligned}
\norm{\frac{X_0^\top X_0}{d} - r_1 I}_{\max} \prec \frac{1}{\sqrt{n}}, \quad \norm{\frac{X_0}{\sqrt{d}}}_{\text{op}} \prec 1.
\end{aligned}
\end{equation}
\end{assumption}
\begin{assumption}[Nonlinearity]
\label{asmpt:nonlin}
The scalar function \(\sigma_1\) is \(\lambda_{\sigma}\)-Lipschitz and satisfies \(\langle{\sigma_1}\rangle_{\mathcal N(r_1)}=0\), where
\begin{equation}
\langle{f}\rangle_{\mathcal N(\sigma^2)} := \frac{1}{\sqrt{2\pi}\sigma}\int_\R f(x) \exp\Bigl(-\frac{x^2}{2\sigma^2}\Bigr)\dif x.
\end{equation}
\end{assumption}
\begin{assumption}[Proportional regime]
\label{asmpt:high_dim}
For some constants \(c_1, c_2\),
\begin{equation}
    c_1 n\le \min\set{d,k_1}\le\max\set{d, k_1}\le c_2 n, \quad 0<c_1<c_2<\infty.
\end{equation}
\end{assumption}
For simplicity, we set the variance of the weight matrix to be equal to 1, although the results can be easily extended to arbitrary variance $\Delta$, by scaling the function $\sigma_1$.

Let \(\tilde{w}_i\) denote the \(i\)th row of \(W_1\). We define 
% \begin{equation}
%     \cX := \sigma\Bigl(\frac{WX}{\sqrt{d}}\Bigr)^\top = \sigma\Bigl(\frac{X^\top W^\top}{\sqrt{d}}\Bigr) = \biggl(\sigma\Bigl(\frac{X^\top  w_1}{\sqrt{d}}\Bigr)\cdots \sigma\Bigl(\frac{X^\top  w_{k_1}}{\sqrt{d}}\Bigr)\biggr)\in\R^{n \times k_1}
% \end{equation}
\begin{equation}
    X_1 := \sigma_1\Bigl(\frac{W_1X_0}{\sqrt{d}}\Bigr)  = \biggl(\sigma_1\Bigl(\frac{X^\top {\tilde{w}}_1}{\sqrt{d}}\Bigr)\cdots \sigma_1\Bigl(\frac{X^\top {\tilde{w}}_{k_1}}{\sqrt{d}}\Bigr)\biggr)^\top\in\R^{k_1 \times n}
\end{equation}
as a matrix with independent identically distributed rows and corresponding sample covariance matrix
\begin{equation}
    \wh \Sigma:= \frac{X_1^\top X_1}{k_1}  = \frac{1}{k_1} \sigma_1\Bigl(\frac{X_0^\top W_1^\top}{\sqrt{d}}\Bigr) \sigma_1\Bigl(\frac{W_1X_0}{\sqrt{d}}\Bigr).
\end{equation}
We have
\((X_1)_{ij} = \sigma_1(\xi_{ij})\), for \(\xi_{ij} :=  {\tilde{w}}_i^\top \vx_j \sim \mathcal{N}\left(0, \norm{\vx_j}^2/ d \right)\), where \(\vx_j\) is the \(j\)th column of \(X_0\).
In order to analyze functions of Gaussian variables, we use the following decomposition.
\begin{lemma}[Hermite decomposition]
For any Lipschitz-continuous $f$ and any $\sigma>0$ we have the \(\sigma\)-Hermite expansion\footnote{Note that despite the appearance of the derivative smoothness is not required as by integration by parts the derivative can be transferred to the smooth Gaussian weight.},
\begin{equation}\label{eq:hermite_exp}
    \begin{split}
        f(x) & =  \sum_{k\ge 0} \frac{\sigma^k}{k!} \He_k\Bigl(\frac{x}{\sigma}\Bigr)  \langle{f^{(k)}}\rangle_{\mathcal N(\sigma^2)} %\\
        %&=\sum_{k\ge 0} \frac{1}{k!} \He_k\Bigl(\frac{x}{\sigma}\Bigr) \braket*{\He_k f(\sigma x)}_{\mathcal{N}} \\
        %&= \sum_{k\ge 0} \frac{1}{k!} \He_k\Bigl(\frac{x}{\sigma}\Bigr) \frac{1}{\sqrt{2\pi}}\int_\R f(\sigma x) \He_k(x) e^{-x^2/2}\dif x\\
        %&= \sum_{k\ge 0} \frac{(-1)^k}{k!} \He_k\Bigl(\frac{x}{\sigma}\Bigr) \frac{1}{\sqrt{2\pi}}\int_\R f(\sigma x) \frac{\dif^k}{\dif x^k } e^{-x^2/2}\dif x\\
        %&= \sum_{k\ge 0} \frac{\sigma^k}{k!} \He_k\Bigl(\frac{x}{\sigma}\Bigr)  \frac{1}{\sqrt{2\pi}}\int_\R f^{(k)}(\sigma x) e^{-x^2/2}\dif x = \sum_{k\ge 0} \frac{\sigma^k}{k!} \He_k\Bigl(\frac{x}{\sigma}\Bigr)  \braket{ f^{(k)}(\sigma x) }_\mathcal{N}
    \end{split}
\end{equation} 
where 
\begin{equation}
    \begin{split}
        \He_k(x) &:= (-1)^k \exp\Bigl(\frac{x^2}{2}\Bigr) \frac{\dif^k}{\dif x^k} \exp\Bigl(-\frac{x^2}{2}\Bigr)
    \end{split}
\end{equation}
with \(\He_k(x)\) being the standard Hermite polynomials \(\He_0(x)=1\), \(\He_1(x)=x\), \(\He_2(x)=x^2-1\), etc.
\end{lemma}
Note that the Hermite polynomials are pairwise orthogonal with respect to the Gaussian density. More precisely, 
\begin{equation}\label{eq:hermite_orthogonal}
    \begin{split}
        \E\He_k(N_1) \He_j(N_2)&=\delta_{jk}k!\Cov(N_1,N_2)^k
    \end{split}
\end{equation}
for jointly Gaussian \(N_1,N_2\) with \(\E N_1=\E N_2=0\) and \(\E N_1^2=\E N_2^2=1\). By applying~\eqref{eq:hermite_exp} twice and using~\eqref{eq:hermite_orthogonal} we obtain the Parseval identity
\begin{equation}\label{parseval}
    \langle{f^2}\rangle_{\mathcal N(\sigma)} = \sum_{k\ge 0} \frac{\sigma^{2k}}{k!} \langle{f^{(k)}}\rangle_{\mathcal N(\sigma)}^2 .
\end{equation} 
In the proof of the deterministic equivalent for the deep random features model, we rely on techniques developed in~\cite{louart2018concentration, chouard2022quantitative} which use concentration of measure theory to analyze random matrices. This approach works particularly well with common neural network architectures, where one can view transformations from layer to layer as Lipschitz mappings. The following Lemma establishes Lipschitzness of required functions.
\begin{lemma}
\label{lemma:lipschitz}
Let \(f(x)\) be a \(\lambda\)-Lipschitz function. Let \(x, y, w \in \R^{d}\), \(W \in \R^{k \times d}\) and \(X \in \R^{d \times n}\). The following maps are Lipschitz, assuming \(f(x)\) is applied entrywise:
\begin{equation}
    w \to f \left(\frac{x^\top w}{\sqrt{d}}\right)  \qquad \text{and} \qquad 
    W \to f\left(\frac{W X}{\sqrt{d}}\right),
\end{equation}
    with Lipschitz constants \(\lambda \norm{x / \sqrt{d}}\) and \(\lambda \norm{X / \sqrt{d}}\) respectively. 
    Furthermore, under the event \(Q := \{|f(\vx^\top w / \sqrt{d})| \lesssim 1 \land |f(\vy^\top v / \sqrt{d})| \lesssim 1|\}\), the map
    \begin{equation}
    w \to f \left(\frac{x^\top w}{\sqrt{d}}\right) f\left(\frac{y^\top w}{\sqrt{d}}\right)
    \end{equation}
    is also Lipschitz with corresponding constant \(\alpha \lesssim \lambda\left(\norm{\vx / \sqrt{d}} + \norm{\vy / \sqrt{d}}\right)\).
\end{lemma}
\begin{proof}
Lipschitz property of the first and second map follows directly from Cauchy-Schwarz inequality. For the third map, since the product of Lipschitz functions is not necessarily Lipschitz, one needs to condition on the "good" event \(Q\). For simplicity, denote \(f( {a},  {b}) := f \left( {a}^\top  {b} / \sqrt{d}\right)\). Under \(Q\)  we can write, for some vectors \(u, v \in \R^d\),
\begin{equation}
\begin{aligned}
&|f ( {x},  {w}) f (  y,   w) - f ( {x},  {v}) f (  y,   v)| 
\\
&\quad\leq |f ( {x},  {w}) f (  y,   w) - f ( {x},  {w}) f (  y,   v)| + |f ( {x},  {w}) f (  y,   v) - f ( {x},  {v}) f (  y,   v)|   \\
&\quad =|f (\vx,   w)||f(\vy,   w) - f(\vy,   v)| + |f (\vy,   v)||f(\vx,   w) - f(\vx,   v)|  \\
&\quad\lesssim \lambda \left(\norm{\vx / \sqrt{d}} + \norm{\vy / \sqrt{d}}\right) \norm{  w -   v}.
\end{aligned}
\end{equation}
\end{proof}
Recall the notations
\begin{equation}
\begin{aligned}
r_2 &:= \langle \sigma_1^2 \rangle_{\mathcal{N}(r_1)} \\
\kappa_1^1 &:= \langle\sigma_1'\rangle_{\mathcal{N}(r_1)} \\
\kappa_*^1 &:= \sqrt{\langle\sigma_1^2\rangle_{\mathcal{N}(r_1)} - r_1 (\kappa_1^1)^2}
% r_2 &= \E_{\xi \sim \mathcal{N}(r_1)} \left[\sigma_1(\xi)^2\right] \\
% \kappa_1^1 &= \frac{1}{r_1}\E_{\xi \sim \mathcal{N}(r_1)} \left[\xi \sigma_1(\xi)\right] \\
% \kappa_*^1 &= \sqrt{\E_{\xi \sim \mathcal{N}(r_1)}\left[\sigma_1(\xi)^2\right] - r_1 \left(\kappa_1^1)\right)^2}
\end{aligned}
\end{equation}
for the proof of~\Cref{prop 1layer}. We state technical Lemmas.

\begin{lemma}
\label{lemma:sigma_norm_bound}
For \(  w \sim \mathcal{N}(0, I)\), \(\lambda_{\sigma}\)-Lipschitz function \(\sigma(x)\) and \(\norm{x / \sqrt{d}} \lesssim 1\), with high probability
\begin{equation}
\abs{\sigma \left(\frac{x^\top   w}{\sqrt{d}}\right)} \lesssim 1.
\end{equation}
\end{lemma}
\begin{proof}
Since the map \(  w \to \sigma \left(\frac{x^\top   w}{\sqrt{d}}\right)\) is \(\norm{x / \sqrt{d}}\lambda_{\sigma}\)-Lipschitz, we have by Gaussian concentration theorem (see e.g.~Theorem 5.2.2 in~\cite{MR3837109}) that 
\begin{equation}
P\left(\left| \sigma\left(\frac{x^\top   w}{\sqrt{d}}\right) - \E_{  w} \sigma\left(\frac{x^\top   w}{\sqrt{d}}\right)\right| \geq t\right) \le e^{-\frac{t^2}{2\norm{x / \sqrt{d}}^2\lambda_{\sigma}^2}}.
\end{equation}
Next, by Equation~\eqref{eq:func gaussian mean}, for each \(i \in [n]\),
\begin{equation}
\E_{  w} \sigma\left(\frac{x^\top   w}{\sqrt{d}}\right) = \langle \sigma \rangle_{\mathcal{N}(\norm{x}^2/d)} = O(1/\sqrt{n}),
\end{equation}
which implies that, with high probability, 
\begin{equation}
\abs{\sigma \left(\frac{x^\top   w}{\sqrt{d}}\right)} \lesssim 1.
\end{equation}
\end{proof}
\begin{lemma}
\label{lemma:sample_cov_entry_subgaussian}
For \(  w \sim \mathcal{N}(0, I)\), the random variable \(\sigma \left( \frac{\vx^\top  {w}}{\sqrt{d}}\right) \sigma \left(\frac{ {w}^\top \vy}{\sqrt{d}} \right)\) is subgaussian with high probability. Its subgaussian norm is \(O(\lambda_{\sigma}(\norm{\vx / \sqrt{d}} + \norm{\vy / \sqrt{d}}))\)
\end{lemma}
\begin{proof}
Follows from~\Cref{lemma:sigma_norm_bound},~\Cref{lemma:lipschitz} and the Gaussian concentration theorem.
\end{proof}

\begin{lemma}
For matrices \(A, B \in \R^{n \times n}\), we have
\begin{enumerate}
\item \(\norm{AB}_F \leq \norm{A}\norm{B}_F\),
\item  \(\Tr(AB) \leq \norm{A}_F \norm{B}_F\),
\item \(A^{-1} - B^{-1} = A^{-1}(B - A)B^{-1}\), \quad if \(A\) and \(B\) are invertible.
\end{enumerate}
\end{lemma}
\begin{lemma}
For any positive semi-definite matrix \(Y\) and for any \(z \in \C \setminus \R_+\), we have
\begin{equation}
    \norm{(Y - z)^{-1}} \leq \mathrm{dist}(z, \R_+)^{-1}.
\end{equation}
\end{lemma}

\begin{proof}[Proof of~\Cref{prop 1layer}]
Define the population covariance matrix
\begin{equation}
\label{eq:pop_cov_def}
    \Sigma_X := \E_{  w}\sigma\Bigl(\frac{X_0^\top   w}{\sqrt{d}}\Bigr)\sigma\Bigl(\frac{  w^\top X_0}{\sqrt{d}}\Bigr)\in\R^{n\times n}, \quad   w\sim\mathcal N(0,I).
\end{equation}
Using Hermite series expansion~\eqref{eq:hermite_exp} and~\eqref{eq:hermite_orthogonal}, for fixed \(X_0\), we can write an explicit form 
\begin{equation}
    \Sigma_X = \sum_{a\ge 0}\frac{1}{a!} D^{(a)}_X\Bigl(\frac{X_0^\top X_0}{d}\Bigr)^{\odot a}D^{(a)}_X,
\end{equation}
where we defined the diagonal matrix
\begin{equation}    D^{(a)}_X:=\diag\Bigl(\langle{\sigma^{(a)}}\rangle_{\mathcal N(\norm{  x_1}^2/d)},\ldots,\langle{\sigma^{(a)}}\rangle_{\mathcal N(\norm{x_n}^2/d)}\Bigr).
\end{equation}
% With high probability it holds that
% \begin{equation}
% \label{eq:orthogonality}
%     \abs{\frac{(X^\top X)_{ij}}{d}-\delta_{ij}} = \abs{\frac{  x_i^\top  x_j}{d}-\delta_{ij}} \lesssim \frac{1}{\sqrt{n}} 
% \end{equation}
From~\Cref{asmpt:norm_bounds} and standard perturbation analysis it follows that
\begin{equation}
\label{eq:func gaussian mean}
    \langle{\sigma}\rangle_{\mathcal N(\norm{\vx_i}^2/d)} = \langle{\sigma}\rangle_{\mathcal N(r_1)} + O\Bigl(\frac{1}{\sqrt{n}}\Bigr) =  O\Bigl(\frac{1}{\sqrt{n}}\Bigr)
\end{equation}
and 
\begin{equation}
    \langle{\sigma'}\rangle_{\mathcal N(\norm{\vx_i}^2/d)} = \langle{\sigma'}\rangle_{\mathcal N(r_1)} + O\Bigl(\frac{1}{\sqrt{n}}\Bigr).
\end{equation}
 Therefore, we can conclude that, for off-diagonal \(i\ne j\),
\begin{equation}
\label{eq:pop_cov_offdiag}
    \begin{split}
        (\Sigma_X)_{ij} &= \sum_{a\ge 0} \frac{\langle{\sigma^{(a)}}\rangle_{\mathcal N(\norm{\vx_i}^2/d)}\langle{\sigma^{(a)}}\rangle_{\mathcal N(\norm{\vx_j}^2/d)}}{a!} \Bigl(\frac{\vx_i^\top\vx_j}{d}\Bigr)^a  \\
        &= \langle{\sigma'}\rangle_{\mathcal N(r_1)}^2 \frac{\vx_i^\top\vx_j}{d} + O\left(\frac{1}{n}\right) = \left(\Sigma_{\mathrm{lin}}\right)_{ij} + O\left(\frac{1}{n}\right)
    \end{split}
\end{equation}
and for diagonal entries we can write directly from~\eqref{eq:pop_cov_def},
\begin{equation}
\label{eq:pop_cov_ondiag}
\begin{aligned}
(\Sigma_X)_{ii} &= \langle \sigma^2 \rangle_{\mathcal{N}(\norm{\vx_i}^2 / d)} =
\langle \sigma' \rangle_{\mathcal N(r_1)}^2 \frac{\norm{\vx_i}^2}{d} + \biggl(\langle \sigma^2 \rangle_{\mathcal{N}(r_1)} - r_1 \langle \sigma' \rangle^2_{\mathcal{N}(r_1)}\biggr) + O\left(\frac{1}{\sqrt{n}}\right) \\
&= \left(\Sigma_{\mathrm{lin}}\right)_{ii} + O\left(\frac{1}{\sqrt{n}}\right).
    % (\Sigma_X)_{ii} = \sum_{a\ge 0} \frac{\braket{\sigma^{(a)}}_{\mathcal N(\norm{\vx_i}/\sqrt{d})}^2}{a!} \Bigl(\frac{\norm{\vx_i}^2}{d}\Bigr)^a = \braket{\sigma^2}_{\norm{\vx}/\sqrt{d}} = \braket{\sigma^2}_{\mathcal N(\delta)}\frac{\norm{\vx_i}^2}{d} + O\Bigl(\frac{1}{\sqrt{n}}\Bigr)
\end{aligned}
\end{equation}
Summing over all indices $i, j$ we get that
\begin{equation}
\label{eq:sigma_lin_frobenius}
    \norm{\Sigma_X - \Sigma_{\mathrm{lin}} }_\mathrm{F} = O(1).
\end{equation}
Let us define \(\wc m(\Sigma, z)\) as the solution to the following equation:
\begin{equation}\label{eq:m_check_sigma}
    \wc m =
    \frac{d - n}{n z} -\frac{d}{z n}\langle{(\Sigma\wc m+ 1)^{-1}}\rangle,
\end{equation}
and \(\wc m_{X} \coloneqq \wc m(\Sigma_X, z), \wc m_{\mathrm{lin}} \coloneqq \wc m(\Sigma_{\mathrm{lin}}, z)\).
Consider the sequence of approximations (in a functional tracial sense):
\begin{equation}
\Bigl(\frac{X_1^\top X_1}{k_1} - z\Bigr)^{-1} \approx \Bigl(-\wc m_X z\Sigma_{X} - z\Bigr)^{-1} \approx \Bigl(-\wc m_{\mathrm{lin}} z\Sigma_{X} - z\Bigr)^{-1} \approx \Bigl(-\wc m_{\mathrm{lin}}z\Sigma_{\mathrm{lin}} - z\Bigr)^{-1}.
\end{equation}
The first approximation follows from~\Cref{thm:res_concentration} applied to the matrix \(\cX = X_1^\top\). The matrix \(\cX\) is concentrated due to~\Cref{lemma:lipschitz} and Gaussian concentration theorem.

The second approximation requires proving a stability property of the function \(\wc m(\Sigma, z)\). In particular, we write
\begin{equation}
\begin{aligned}
& \left|\left\langle A \Bigl[\Bigl(-\wc m_X z\Sigma_{X} - z\Bigr)^{-1} - \Bigl(-\wc m_{\mathrm{lin}} z\Sigma_{X} - z\Bigr)^{-1}\Bigr] \right\rangle\right| \\
&\quad = \left|\left\langle A \Bigl[\Bigl(-\wc m_X z\Sigma_{X} - z\Bigr)^{-1}(z(\wc m_{X} - \wc m_{\mathrm{lin}})\Sigma_X)\Bigl(-\wc m_{\mathrm{lin}} z\Sigma_{X} - z\Bigr)^{-1}\Bigr] \right\rangle\right| \\
&\quad \leq \frac{|\wc m_X - \wc m_{\mathrm{lin}}|}{|z|^2\sqrt{n}} \langle A A^* \rangle^{1/2} \norm{\Bigl(\wc m_X \Sigma_{X} + I\Bigr)^{-1}}\norm{\Bigl(\wc m_{\mathrm{lin}} \Sigma_{X} + I\Bigr)^{-1}}\norm{\Sigma_X}_F \\
&\quad \leq
|z|^{-2}|\wc m_X - \wc m_{\mathrm{lin}}| \langle A A^* \rangle^{1/2}.
\end{aligned}
\end{equation}
% \begin{equation}
% \begin{aligned}
% & \left|\left\langle A \Bigl[\Bigl(-\wc m_X z\Sigma_{X} - z\Bigr)^{-1} - \Bigl(-\wc m_{\mathrm{lin}} z\Sigma_{X} - z\Bigr)^{-1}\Bigr] \right\rangle\right| \\
% & \leq \frac{1}{\sqrt{n}} \langle A A^* \rangle^{1/2} \norm{\Bigl(-\wc m_X z\Sigma_{X} - z\Bigr)^{-1} - \Bigl(-\wc m_{\mathrm{lin}} z\Sigma_{X} - z\Bigr)^{-1}}_F \\
% & = \frac{1}{|z|\sqrt{n}} \langle A A^* \rangle^{1/2} \norm{\Bigl(\wc m_X \Sigma_{X} + I\Bigr)^{-1} - \Bigl(\wc m_{\mathrm{lin}} \Sigma_{X} + I\Bigr)^{-1}}_F \\
% & \leq \frac{1}{|z|\sqrt{n}} \langle A A^* \rangle^{1/2} \norm{\Bigl(\wc m_X \Sigma_{X} + I\Bigr)^{-1}}_F \norm{\Bigl(\wc m_{\mathrm{lin}} \Sigma_{X} + I\Bigr)^{-1}}_F \norm{\wc m_X \Sigma_X - \wc m_{\mathrm{lin}} \Sigma_X} \\
% & \leq  \frac{\norm{\Sigma_X}}{|z|\sqrt{n}} \langle A A^* \rangle^{1/2} |\wc m_X - \wc m_{\mathrm{lin}}| 
% % \lesssim \frac{|\wc m_X| |\wc m_{\mathrm{lin}}|}{|z|\sqrt{n}} \langle A A^* \rangle^{1/2} |1 / \wc m_X - 1 / \wc m_{\mathrm{lin}}|
% \end{aligned}
% \end{equation}

Now, we analyze the difference between \(\wc m_X\) and \(\wc m_{\mathrm{lin}}\). According to~\eqref{eq:m_check_sigma}, we can write
\begin{equation}
\begin{aligned}
\Delta := |\wc m_X - \wc m_{\mathrm{lin}}| &= \frac{d}{|z| n^2} \Tr\left[(\wc m_X \Sigma_X+I)^{-1} - (\wc m_{\mathrm{lin}}\Sigma_{\mathrm{lin}}+I)^{-1}\right] \\
&\lesssim \frac{1}{|z|n} \Tr\left[(\wc m_{X}\Sigma_{X}+I)^{-1} (\wc m_\mathrm{lin}\Sigma_{\mathrm{lin}} - \wc m_X\Sigma_{X}) (\wc m_X\Sigma_{\mathrm{lin}}+I)^{-1}\right] \\
& \leq \frac{1}{|z|n} \norm{(\wc m_X\Sigma_{X}+I)^{-1}}_F \norm{\wc m_{\mathrm{lin}}\Sigma_{\mathrm{lin}} - \wc m_X\Sigma_{X}}_F \norm{(\wc m_\mathrm{lin}\Sigma_{\mathrm{lin}}+I)^{-1}} \\
& \leq \frac{1}{|z| \sqrt{n}} \norm{\wc m_{\mathrm{lin}}\Sigma_{\mathrm{lin}} - \wc m_X\Sigma_{X}}_F \\
& \leq \frac{1}{|z| \sqrt{n}} \norm{\wc m_{\mathrm{lin}}\Sigma_{\mathrm{lin}} - \wc m_{\mathrm{lin}}\Sigma_{X}}_F + \frac{d}{|z| n^{3/2}} \norm{\wc m_{\mathrm{lin}}\Sigma_X - \wc m_X\Sigma_{X}}_F \\
& = \frac{|\wc m_{\mathrm{lin}}|}{|z| \sqrt{n}} \norm{\Sigma_{\mathrm{lin}} - \Sigma_{X}}_F + \frac{\norm{\Sigma_X}}{|z| \sqrt{n}} \Delta.
\end{aligned}
\end{equation}
Since \(\norm{\Sigma_X} |z|^{-1} n^{-1/2} \ll 1\), we obtain using~\eqref{eq:sigma_lin_frobenius} that \(|\wc m_X - \wc m_{\mathrm{lin}}| \lesssim |z|^{-1} n^{-1/2}\), and thus, for the second approximation, 
\begin{equation}
\left|\left\langle A \Bigl[\Bigl(-\wc m_X z\Sigma_{X} - z\Bigr)^{-1} - \Bigl(-\wc m_{\mathrm{lin}} z\Sigma_{X} - z\Bigr)^{-1}\Bigr] \right\rangle\right| \lesssim \frac{1}{\delta^3\sqrt{n}} \langle A A^* \rangle^{1/2}.
\end{equation}
For the third approximation, we can write
\begin{equation}
\begin{aligned}
&\left|\left\langle A \Bigl[\Bigl(-\wc m_{\mathrm{lin}} z\Sigma_{X} - z\Bigr)^{-1} - \Bigl(-\wc m_{\mathrm{lin}} z\Sigma_{\mathrm{lin}} - z\Bigr)^{-1}\Bigr] \right\rangle\right| = \frac{1}{|z||\wc m_{\mathrm{lin}}|} B, \\ 
\text{where} \quad  B &:= \left | \langle{ A (\Sigma_X + 1 / \wc m)^{-1} -  A (\Sigma_{\mathrm{lin}} + 1 / \wc m)^{-1}}\rangle \right |  \\
& \leq 
    \frac{1}{\sqrt{n}} \langle A A^{\ast} \rangle^{1/2} \norm{(\Sigma_X + 1 / \wc m)^{-1} -  (\Sigma_{\mathrm{lin}} + 1 / \wc m)^{-1}}_{\mathrm{F}} \\
    & = \frac{1}{\sqrt{n}} \langle A A^{\ast} \rangle^{1/2} \norm{(\Sigma_X + 1 / \wc m)^{-1}(\Sigma_{\mathrm{lin}} - \Sigma_X)(\Sigma_{\mathrm{lin}} + 1 / \wc m)^{-1}}_{\mathrm{F}}\prec 
\frac{\langle A A^{\ast} \rangle^{1/2}}{\delta^2\sqrt{n}}, 
\end{aligned}
\end{equation}
% \begin{equation}
% \begin{aligned}
% \left| \langle A \left( (\Sigma_X - \kappa)^{-1} -  (\Sigma_{\mathrm{lin}} - \kappa)^{-1}\right)\rangle \right | &\leq \\
% \frac{1}{\sqrt{n}} \langle A A^{\ast} \rangle^{1/2} \norm{(\Sigma_X - \kappa)^{-1} -  (\Sigma_{\mathrm{lin}} - \kappa)^{-1}}_{\mathrm{F}} &= \\
% \frac{1}{\sqrt{n}} \langle A A^{\ast} \rangle^{1/2} \norm{(\Sigma_X - \kappa)^{-1}(\Sigma_{\mathrm{lin}} - \Sigma_X)  (\Sigma_{\mathrm{lin}} - \kappa)^{-1}}_{\mathrm{F}} &\lesssim
% \frac{\langle A A^{\ast} \rangle^{1/2}}{\mathrm{dist}(\kappa, \R_+)^2\sqrt{n}},
% \end{aligned}
% \end{equation}
where in the last inequality we used~\eqref{eq:sigma_lin_frobenius}.

Combining all the approximations together, we have proved that 
\begin{equation}
\left|\left\langle A \Bigl[\Bigl(\frac{X_1^\top X_1}{k_1} - z\Bigr)^{-1} - \Bigl(-\wc m_{\mathrm{lin}} z\Sigma_{lin} - z\Bigr)^{-1}\Bigr] \right\rangle\right| \prec \frac{\langle A A^\ast\rangle^{1/2}}{\delta^9\sqrt{n}}.
\end{equation}
Next, we will verify that Assumption~\ref{asmpt:norm_bounds} holds true when we replace matrix \(X_0\) by \(X_1\) and \(r_1\) by \(r_2\).
In particular, we want to show that, with high probability,
\begin{equation}
\norm{\frac{X_1^\top X_1}{k_1} - r_2I}_{\max} = O\left(\frac{1}{\sqrt{n}}\right).
\end{equation}
Note that Equations~(\ref{eq:pop_cov_offdiag}, \ref{eq:pop_cov_ondiag}) show that
\begin{equation}
\label{eq:pop_cov_to_lin}
\norm{\Sigma_X - r_2I}_{\max} = O\left(\frac{1}{\sqrt{n}}\right).
\end{equation}

We have that 
\begin{equation}
\label{eq:sample_cov_entry}
\left(\frac{X_1^\top X_1}{k_1}\right)_{ij} = \frac{1}{k_1} \sum_{l = 1}^{k_1} \sigma \left( \frac{\vx_i^\top  {\tilde{w}}_l}{\sqrt{d}}\right) \sigma \left(\frac{ {\tilde{w}}_l^\top \vx_j}{\sqrt{d}} \right) = \frac{1}{k_1} \sum_{l = 1}^{k_1} Y_l, \quad \text{where } Y_l := \sigma \left( \frac{\vx_i^\top  {\tilde{w}}_l}{\sqrt{d}}\right) \sigma \left(\frac{ {\tilde{w}}_l^\top \vx_j}{\sqrt{d}} \right).
\end{equation}
Note that \(Y_l\) are independent random variables and from Lemma~\ref{lemma:sample_cov_entry_subgaussian} it follows that the subgaussian norm of \(Y_l\) is \(O(\lambda_{\sigma} \norm{X / \sqrt{d}})\). Therefore, from Hoeffding inequality, we have that
\begin{equation}
P\left(\left|\left(\frac{X_1^\top X_1}{k_1}\right)_{ij} - (\Sigma_X)_{ij}\right| \geq t\right) \leq 2e^{-\frac{c t^2 k_1}{\lambda_{\sigma} \norm{X / \sqrt{d}}}},
\end{equation}
from which, applying union bound, we can deduce that
\begin{equation}
\label{eq:sample_to_pop_cov}
\norm{\frac{X_1^\top X_1}{k_1} - \Sigma_X}_{\max} = O\left(\frac{1}{\sqrt{n}}\right).
\end{equation}
Combining Equations~(\ref{eq:pop_cov_to_lin}) and~(\ref{eq:sample_to_pop_cov}) we get the required maximum norm bound.
Next, with a standard \(\varepsilon\)-net argument (see, e.g.~\cite{chouard2022quantitative}, Proposition 3.4) we can show that 
\begin{equation}
\frac{1}{\sqrt{d}}\norm{X_1 - \E X_1} \prec 1.
\end{equation}
Since \(\sqrt{n} \norm{\E X_1}_{\max} \lesssim 1\) it follows that 
\begin{equation}
\norm{\frac{\E X_1}{\sqrt{d}}} \le \sqrt{\frac{nk_1}{d}}\norm{\E X_1}_\mathrm{max}\lesssim \sqrt{\frac{k_1}{d}}\lesssim1.
\end{equation}
% for any vector \(u\), such that \(\norm{u}_2 = 1\)
% \begin{equation}
% |(\E X_1 u)_i| = |\sum_j (\E X_1)_{ij} u_j | \leq \sum_j |(\E X_1)_{ij}| |u_j| \prec \norm{u}_1 / \sqrt{n} \leq 1,
% \end{equation}
% from which we get that \((\E X_1 u)_i^2 \prec 1\), and thus
% \begin{equation}
% \norm{\frac{\E X_1}{\sqrt{d}}}^2 = \max_{\norm{u} = 1} \norm{\frac{\E X_1 u}{\sqrt{d}}}^2 \prec 1.
% \end{equation}
Finally, the claim that $\dist(-1/\wc m(z),\R_+)\ge \dist(z,\R_+)$ follows elementarily from the fixed point equation, see e.g. Proposition 6.2 in~\cite{chouard2022quantitative}.
\end{proof}

\begin{proof}[Proof of~\Cref{prop:mult_layers}]
    This follows directly from iteratively applying~\Cref{prop 1layer} until we reach 
    \begin{equation}
        \Bigl(\frac{X_\ell^\top X_\ell}{k_\ell}-z_\ell\Bigr)^{-1}\approx c_1 \cdots c_\ell \Bigl(\frac{X_0^\top X_0}{d}-z_0\Bigr)^{-1}
    \end{equation}
    in the last layer, where ``\(\approx\)'' is to be understood in the sense of~\Cref{prop 1layer}. Now, using that \(X_0X_0^\top/n\) is a sample covariance matrix with population covariance matrix \(\Omega_0\), it follows that 
    \begin{equation}
        \Bigl(\frac{X_0^\top X_0}{d}-z_0\Bigr)^{-1} = \frac{d}{n}\Bigl(\frac{X_0^\top X_0}{n}-\frac{d}{n}z_0\Bigr)^{-1} \approx \frac{d}{n} \Bigl(\frac{d}{n}m_{\mu(\Omega_0)\boxtimes \mu_\mathrm{MP}^{d/n}}\Bigl(\frac{d}{n}z_0\Bigr)+\frac{d-n}{dz_0}\Bigr),
    \end{equation}
    where we used~\Cref{co resolvent} once more in the final step. 
\end{proof}

%%%%%%%%%%%%%%%%%%%%%%%%%%%%%%%%%%%%%%%%%%%%%%%%%%%%%%%%%%%%%%%%%%%%%%%%%%%%%%%
\newpage
\section{Closed-form formulae for population covariances}
\label{App:population_heuristic}
\subsection{Multi-Layer linearization}
In this Appendix, we provide a (heuristic) derivation of closed-form expressions for the population covariances:
\begin{align}\Omega_{L}\coloneqq\E\left[\varphi(\x)\varphi(\x)^\top\right],&&
    \Phi_{L^\star L}\coloneqq\E\left[\varphi^{\star}(\x)\varphi(\x)^\top\right], && \Psi_{L^\star }\coloneqq\E \left[\varphi^\star(\x) \varphi^\star(\x)^\top\right].
\end{align}
\textcolor{black}{This derivation has appeared in \cite{Cui2023}, and we include it here for the sake of completeness.}

\paragraph{Reminder of the results}
Consider the dRF \eqref{eq:definition_multilayer_RF} and target \eqref{eq:target}, with data $x\sim \mathcal{N}(0,\Omega_0)$.  $\Omega_0$ is assumed to possess extensive Frobenius norm and trace, i.e. there exists constant $c,c^\prime$ so that asymptotically (noting $k_0=d$)
\begin{align}
\label{eq:App:Pop:condition_Omega0}
    c<\frac{1}{d}\tr \Omega_0^2=\frac{1}{d}||\Omega_0||_F^2<c^\prime<\infty, && c<\frac{1}{d}\tr \Omega_0<c^\prime<\infty.
\end{align}
In terms of the limiting spectral density $\mu$, these assumptions imply that the first and second moments are finite and non zero.
Consider the sequence of variances defined by the recurrence
\begin{align}
\label{eq:App:Pop:r_multilayer}
    r^{(\star)}_{\ell+1}=\Delta_{\ell+1}^{(\star)}\mathbb{E}_z^{\mathcal{N}(0,r_\ell)}\left[\sigma_\ell^{(\star)}(z)^2\right]
\end{align}
with the initial condition
\begin{align}
    r_1^{(\star)}=\Delta_1^{(\star)}\frac{1}{d}\tr \Omega_0
\end{align}
and the GET \cite{Gerace2020GeneralisationEI,Goldt2020ModellingTI,Goldt2021TheGE} coefficients
\begin{align}
\label{eq:App:Pop:kappa_multilayer}
    \kappa_1^{\ell(\star)}=\frac{1}{r_\ell^{(\star)}}\mathbb{E}_z^{\mathcal{N}(0,r_\ell)}\left[z\sigma_\ell^{(\star)}(z)\right]
    &&\kappa_*^{\ell (\star)}=\sqrt{\mathbb{E}_z^{\mathcal{N}(0,r_\ell)}\left[\sigma_\ell^{(\star)}(z)^2\right]-r_\ell^{(\star)}\left(\kappa_1^{\ell(\star)}\right)^2}.
\end{align}
Define the sequence of matrices
\begin{align}
\label{eq:App:Pop:Omega_Psi_multilayer}
    \Omega_{\ell+1}^{\mathrm{lin}}=\kappa_1^{\ell 2}\frac{W_{\ell+1} \Omega^{\mathrm{lin}}_\ell W_{\ell+1}^\top}{k_\ell}+\kappa_*^{\ell 2}I_{k_{\ell+1}}\\
    \Psi_{\ell+1}^{\mathrm{lin}}=\kappa_1^{\star\ell 2}\frac{W^\star_{\ell+1} \Psi^{\mathrm{lin}}_\ell W_{\ell+1}^{\star\top}}{k^\star_\ell}+\kappa_*^{\star\ell 2}I_{k^\star_{\ell+1}}
\end{align}
with initialization
\begin{align}
    \Omega_0^{\mathrm{lin}}\coloneqq \Psi_0^{\mathrm{lin}}=\Omega_0,
\end{align}
and the matrix
\begin{align}
\label{eq:App:Pop:Phi_multilayer}
    \Phi_{\ell^\star \ell}^{\mathrm{lin}}=\prod\limits_{r=1}^\ell\prod\limits_{s=1}^{\ell^\star}\kappa_1^r\kappa_1^{s\star}
    \times
    \frac{W_{\ell^\star}\cdot...\cdot W_1^\star\cdot \Sigma \cdot W_1^\top \cdot ... \cdot W_\ell^\top}{\prod\limits_{r=0}^{\ell-1}\prod\limits_{s=0}^{\ell^\star-1}\sqrt{k_rk^\star_s}}.
\end{align}
Then $\Omega_L\approx\Omega_L^{\mathrm{lin}}, \Psi_{L^\star}\approx\Psi_{L^\star}$ and $\Phi_{\ell^\star \ell}\approx \Phi_{\ell^\star \ell}^{\mathrm{lin}}$. $A \approx B $ is understood as $\sfrac{||A-B||_F^2}{d}=\mathcal{O}(\sfrac{1}{d})$.

\paragraph{Example for $L=2$}
We give for concreteness an example for $L^\star=1,  L=2$ (RF teacher, 2-layer DRN student). The recursions \eqref{eq:App:Pop:Omega_Psi_multilayer}\eqref{eq:App:Pop:Phi_multilayer} for the student reads for $L=2$
\begin{align}
\label{eq:App:Pop:two_layer_covariance}
    &\Omega_2=(\kappa_1^1)^2(\kappa_1^2)^2\frac{W_2W_1\Sigma W_1^\top W_2^\top}{k_1d}+(\kappa_1^2)^2(\kappa_*^1)^2\frac{W_2W_2^\top}{k_1}+(\kappa_*^2)^2I_{k_1}\\
    &\Psi_1=(\kappa_1^{1\star})^2\frac{W_1^\star\Sigma W_1^{\star \top}}{d}+(\kappa_*^{1\star})^2I_{k^\star_1}\\
    &\Phi_{1,2}=\kappa_1^1\kappa_1^2\kappa_1^{1\star}\frac{W_1^\star\Sigma W_2^\top W_1^\top}{d\sqrt{k_1}}
\end{align}

\paragraph{Equivalent Linear Net}
Note that the linearization means one can think of the $\ell$-th layer as a noisy linear layer,
\begin{align}
    \varphi_\ell(x)^{\mathrm{lin}}\approx \kappa_1^\ell \frac{1}{\sqrt{k_{\ell-1}}}W_\ell\cdot x+\kappa_*^\ell \xi_\ell
\end{align}
with $\xi_\ell\in\mathbb{R}^{k_\ell}$ an i.i.d Gaussian noise indepent layer from layer, and also independent between the teacher and student \textit{provided the teacher and student weights are drawn independently}. Similarly for the teacher:
\begin{align}
    \varphi^\star_\ell(x)\approx \kappa_1^{\star\ell} \frac{1}{\sqrt{k^\star_{\ell-1}}}W^\star_\ell\cdot x+\kappa_*^{\star\ell} \xi^\star_\ell
\end{align}
This provides a simple way to rederive the relations \eqref{eq:App:Pop:Omega_Psi_multilayer} and \eqref{eq:App:Pop:Phi_multilayer}.

\subsection{Derivation sketch for $\Omega_L$}
We first derive a relation between the covariance of the post-activations at two successive layers, and then iterate. Remark that since the computation for $\Psi_{L^\star}$ is identical \textit{mutatis mutandis}, we only address here $\Omega_L$.
\paragraph{Propagation through a single layer}
Consider the auxiliary single-layer problem
\begin{equation}
    h(x)=\sigma\left(\frac{1}{ \sqrt{d}}W\cdot x\right)
\end{equation}
with $x\sim\mathcal{N}(0,\Sigma)$. Suppose recursively that $\Sigma$ statisfies the properties \eqref{eq:App:Pop:condition_Omega0}. The population covariance of the post-activations $h$ reads
\begin{align}
    \Omega_{ij}=\langle h_i(x)h_j(x)\rangle_x=\int \frac{e^{-\frac{1}{2}\begin{pmatrix}
    u&v
    \end{pmatrix}
    \begin{pmatrix}
    \frac{w_i^\top \Sigma w_i}{d} & \frac{w_i^\top \Sigma w_j}{d}\\
    \frac{w_i^\top \Sigma w_j}{d}& \frac{w_j^\top \Sigma w_j}{d}
    \end{pmatrix}^{-1}
    \begin{pmatrix}
    u\\v
    \end{pmatrix}
    }}{\sqrt{\det 2\pi \begin{pmatrix}
    \frac{w_i^\top \Sigma w_i}{d} & \frac{w_i^\top \Sigma w_j}{d}\\
    \frac{w_i^\top \Sigma w_j}{d}& \frac{w_j^\top \Sigma w_j}{d}
    \end{pmatrix}}}\sigma(u)\sigma(v).
\end{align}
Note that have 
\begin{equation}
\label{eq:App:Pop:r_singlelayer}
    \mathbb{E}_w \frac{w^\top \Sigma w}{d}=\frac{\Delta}{d}\tr \Sigma \equiv r,
\end{equation}
which by assumption is of order $1$. Diagonalizing $\Sigma=U\Lambda U^\top$ and noting that $U^\top w$ is still Gaussian with independent entries, 
\begin{align}
    \mathbb{V}_w\left[\frac{w^\top \Sigma w}{d}\right]=\frac{1}{d^2}\sum\limits_{i=1}^d \lambda_i^2 \mathbb{V}_w\left[(U^\top w)_i^2\right]=\frac{2\Delta}{d^2}\tr \Sigma^2=\frac{2\Delta}{d}\frac{||\Sigma||_F^2}{d}=\mathcal{O}\left(\frac{1}{d}\right)
\end{align}
provided $\sfrac{||\Sigma||_F^2}{d}$ is finite. We used the fact that the variance of a $1-$degree of freedom $\chi^2$ variable is $2$. Plugging the definition of $r$ into the above yields, for $i\ne j$:
\begin{align}
    \Omega_{ij}&=\int \frac{e^{-\frac{1}{2}\frac{1}{r^2-\mathcal{O}\left(\frac{1}{d}\right)}(r u^2+rv^2)}e^{\frac{1}{r^2-\mathcal{O}\left(\frac{1}{d}\right)}\frac{w_i^\top\Sigma w_j}{d}uv}}{2\pi\sqrt{r^2-\mathcal{O}\left(\frac{1}{d}\right)}}\sigma(u)\sigma(v)\notag\\
    &=\left(\int \frac{e^{-\frac{1}{2r}z^2}}{\sqrt{2\pi r}}\sigma(z)\right)^2+\frac{1}{r}
    \frac{w_i^\top \Sigma w_j}{d}
    \left(\int \frac{e^{-\frac{1}{2r}z^2}}{\sqrt{2\pi r}}z\sigma(z)\right)^2+\mathcal{O}\left(\frac{1}{d}\right)\notag\\
    &=\kappa_1^2\times \frac{w_i^\top \Sigma w_j}{d}.
\end{align}

on the diagonal ($i=j$), this becomes
\begin{align}
    \Omega_{ii}=\int \frac{e^{-\frac{1}{2r}z^2}}{\sqrt{2\pi r}}\sigma(z)^2=\kappa_*^2+r\kappa_1^2
\end{align}
yielding
\begin{align}
\label{eq:App:Pop:linearization_single_layer}
    \Omega=\kappa_1^2\frac{W\Sigma W^\top}{d}+\kappa_*^2I_k
\end{align}
with
\begin{align}
\label{eq:App:Pop:kappa_single_layer}
    \kappa_1=\frac{1}{r}\mathbb{E}_{z}^{\mathcal{N}(0,r)}\left[z\sigma(z)\right]&&\kappa_*^2=\mathbb{E}_{z}^{\mathcal{N}(0,r)}\left[\sigma(z)^2\right]-r\times \kappa_1^2
\end{align}
This extends the GET \cite{Gerace2020GeneralisationEI} generalization used in \cite{dAscoli2021OnTI} to arbitrary input covariances.

\paragraph{Iterating layer to layer}
\eqref{eq:App:Pop:kappa_multilayer} and \eqref{eq:App:Pop:Omega_Psi_multilayer} follow by straightforward recursion from the single-layer results \eqref{eq:App:Pop:kappa_single_layer} and \eqref{eq:App:Pop:linearization_single_layer}. One just need to connect \eqref{eq:App:Pop:r_multilayer} to the single-layer variance $r$ \eqref{eq:App:Pop:r_singlelayer}. 
\begin{align}
    r_{\ell+1}&=\Delta_{\ell+1}\frac{1}{k_\ell}\tr \Omega_\ell\notag\\
    &=\Delta_{\ell+1}\left(
    \frac{1}{k_\ell}\left(\kappa_1^{\ell}\right)^2\tr[\frac{W_\ell\Omega_{\ell-1} W_\ell^\top}{k_{\ell-1}}]+\left(\kappa_*^{\ell}\right)^2
    \right)\notag\\
    &=\Delta_{\ell+1}\left(
    \left(\kappa_1^{\ell}\right)^2r_\ell+\left(\kappa_*^{\ell}\right)^2
    \right)\notag\\
    &=\Delta_{\ell+1}\mathbb{E}_{z}^{\mathcal{N}(0,r_\ell)}\left[\sigma_\ell(z)^2\right]
\end{align}
We used
\begin{align}
    \frac{1}{k_\ell}\tr[\frac{W_\ell\Omega_{\ell-1} W_\ell^\top}{k_{\ell-1}}]&=\frac{1}{k_{\ell-1}}\sum\limits_{i=1}^{k_{\ell-1}}\lambda_i^{\ell-1}\frac{1}{k_\ell}\left(U^\top W_\ell^\top W_\ell U\right)_{ii}\notag\\
    &=\frac{1}{k_{\ell-1}}\sum\limits_{i=1}^{k_{\ell-1}}\lambda_i^{\ell-1}\Delta_\ell
\notag\\
&=\Delta_\ell\frac{1}{k_{\ell-1}}\tr\Omega_{\ell-1}=r_\ell
\end{align}
We used that $W_\ell U$ is also an i.i.d Gaussian matrix.
Finally, one must check that the assumption on $\Sigma$ that $\sfrac{||\Sigma||_F^2}{d},\sfrac{\tr\Sigma}{d}=\mathcal{O}(1)$ carries over to $\Omega$. Because $W\Sigma W^\top$ is positive semi definite it is straightforward that 
\begin{equation}
    \frac{1}{k}||\kappa_1^2\frac{W\Sigma W^\top}{d}+\kappa_*^2I_k||_F^2 \ge \kappa_*^2 >0.
\end{equation}
The upper bound can be established using the triangle inequality and the submultiplicativity of the Frobenius norm, as
\begin{align}
    \frac{1}{k}||\kappa_1^2\frac{W\Sigma W^\top}{d}+\kappa_*^2I_k||_F^2 &\le \frac{1}{k}||\kappa_1^2\frac{W\Sigma W^\top}{d}||_F^2+\kappa_*^2\notag\\
    &\le \kappa_*^2+\frac{||W||_F^4}{d^2}\frac{||\Sigma||_F^2}{k}\notag \\
    &\le \kappa_*^2 + c^\prime<\infty.
\end{align}
We used that $\sfrac{||W||_F^2}{dk}=1$ almost surely asymptotically. Moving on to the trace,
\begin{equation}
    \frac{1}{k}\Tr[\kappa_1^2\frac{W\Sigma W^\top}{d}+\kappa_*^2I_k]=\kappa_*^2+\frac{\kappa_1^2}{kd}\Tr[\Sigma W^\top W].
\end{equation}
Bounding
\begin{equation}
    0\le \frac{\kappa_1^2}{kd}\Tr[\Sigma W^\top W]=\frac{\kappa_1^2}{kd}\sum\limits_{i=1}^k w_i^\top \Sigma w_i=\kappa_1^2\frac{1}{d}\Tr{\Sigma}\le \kappa_1^2 c^\prime,
\end{equation}
where the last bound holds asymptotically almost surely.

\subsection{Derivation sketch for $\Phi_{L^\star L}$}
We now turn to the cross-covariance $\Phi_{L^\star L}$ between the post-activations of two random networks with independent weights. Again, we first establish a preliminary result, addressing the statistics of two correlated Gaussians propagating through non-linear layers with independently drawn weights.
\paragraph{Two Gaussians propagating through two layers}
Consider two jointly Gaussian variables $u\in\mathbb{R}^{d},~v\in\mathbb{R}^k$
\begin{equation}
\label{eq:App:Pop:g3m}
   (u,v)\sim \mathcal{N}\left(
\begin{array}{cc}
     \Psi &\Phi  \\
     \Phi^\top &\Omega
\end{array}\right)
\end{equation}
each independently propagated through a non-linear layer
\begin{align}
    h^\star(u)=\sigma_\star\left(\frac{1}{ \sqrt{d_\star}}W_\star\cdot u\right),
    &&
    h(v)=\sigma\left(\frac{1}{ \sqrt{d}}W\cdot v\right).
\end{align}
The weights $W_\star\in\mathbb{R}^{k_\star\times d_\star}$ and $W\in\mathbb{R}^{k\times d}$ have independently sampled Gaussian entries, with respective variance $\Delta_\star$ and $\Delta$. The $i,j-$th element of the cross-covariance $\Phi^h$ can be expressed as

\begin{align}
    \Phi^h_{ij}=\langle h^\star_i(u)h_j(v)\rangle_{u,v}=\int \frac{e^{-\frac{1}{2}\begin{pmatrix}
    x&y
    \end{pmatrix}
    \begin{pmatrix}
    \frac{w_i^{\star\top} \Sigma w^\star_i}{d_\star} & \frac{w_i^{\star\top} \Sigma w_j}{\sqrt{d_\star d}}\\
    \frac{w_i^{\star\top} \Sigma w_j}{\sqrt{d_\star d}}& \frac{w_j^\top \Sigma w_j}{d}
    \end{pmatrix}^{-1}
    \begin{pmatrix}
    x\\y
    \end{pmatrix}
    }}{\sqrt{\det 2\pi \begin{pmatrix}
    \frac{w_i^{\star\top} \Sigma w^\star_i}{d_\star} & \frac{w_i^{\star\top} \Sigma w_j}{\sqrt{d_\star d}}\\
    \frac{w_i^{\star\top} \Sigma w_j}{\sqrt{d_\star d}}& \frac{w_j^\top \Sigma w_j}{d}
    \end{pmatrix}}}\sigma_\star(x)\sigma(y)
\end{align}
As before, the random variables $\sfrac{w_i^{\star\top} \Sigma w^\star_i}{d_\star}$ and $\sfrac{w_j^\top \Sigma w_j}{d}$ concentrate around their mean value 
\begin{align}
    r_\star\equiv\frac{\Delta_\star}{d_\star}\tr \Psi
    &&
    r\equiv \frac{\Delta}{d}\tr \Omega
\end{align}
Plugging these definitions into the above:
\begin{align}
    \Phi^h_{ij}&=\int \frac{e^{-\frac{1}{2}\frac{1}{r_\star r-\mathcal{O}\left(\frac{1}{d}\right)}(r x^2+r_\star y^2)}e^{\frac{1}{r_\star r-\mathcal{O}\left(\frac{1}{d}\right)}\frac{w_i^{\star\top}\Phi w_j}{\sqrt{d_\star d}}xy}}{2\pi\sqrt{r_\star r-\mathcal{O}\left(\frac{1}{d}\right)}}\sigma_\star(x)\sigma(y)\notag\\
    &=\left(\int \frac{e^{-\frac{1}{2r_\star}z^2}}{\sqrt{2\pi r_\star}}\sigma_\star(z)\right)\left(\int \frac{e^{-\frac{1}{2r}z^2}}{\sqrt{2\pi r}}\sigma(z)\right)+\frac{1}{r_\star r}
    \frac{w_i^{\star\top}\Phi w_j}{\sqrt{d_\star d}}
    \left(\int \frac{e^{-\frac{1}{2r_\star}z^2}}{\sqrt{2\pi r_\star}}z\sigma_\star(z)\right)\left(\int \frac{e^{-\frac{1}{2r}z^2}}{\sqrt{2\pi r}}z\sigma(z)\right)+\mathcal{O}\left(\frac{1}{d}\right)\notag\\
    &:=\kappa_1\kappa_1^\star\times  \frac{w_i^{\star\top}\Phi w_j}{\sqrt{d_\star d}}
\end{align}

yielding
\begin{align}
\label{App:eq:propgation_Phi_DRM}
    \Phi^h=\kappa_1\kappa_1^\star\frac{W_\star\Phi W^\top}{\sqrt{d_\star d}}
\end{align}
with
\begin{align}
    \kappa_1=\frac{1}{r}\mathbb{E}_{z}^{\mathcal{N}(0,r)}\left[z\sigma(z)\right]&&\kappa_1=\frac{1}{r_\star}\mathbb{E}_{z}^{\mathcal{N}(0,r_\star)}\left[z\sigma_\star(z)\right]
\end{align}

\paragraph{One Gaussian propagating through one layer}
We will need another result, addressing again two correlated Gaussians, with only one propagating through a non-linear layer. Consider two jointly Gaussian variables $u\in\mathbb{R}^{d_\star},~v\in\mathbb{R}^d$
\begin{equation}
\label{eq:App:Pop:g3m_2}
   (u,v)\sim \mathcal{N}\left(
\begin{array}{cc}
     \Psi &\Phi  \\
     \Phi^\top &\Omega
\end{array}\right)
\end{equation}
with \textit{only} $v$
being propagated through a non linear layer \begin{align}
    h(v)=\sigma\left(\frac{1}{ \sqrt{k}}W\cdot v\right).
\end{align}
The entries $W\in\mathbb{R}^{k\times d}$ are independently sampled from a Gaussian distribution with variance $\Delta$. The $i,j-$th element of the cross-covariance $\Phi$ between $h(v)$ and $u$ can be expressed as
\begin{align}
    \Phi^h_{ij}=\langle u_ih_j(v)\rangle_{u,v}=\int \frac{e^{-\frac{1}{2}\begin{pmatrix}
    x&y
    \end{pmatrix}
    \begin{pmatrix}
    \Psi_{ii} & \frac{\Phi_i w_j}{\sqrt{k}}\\
    \frac{\Phi_i w_j}{\sqrt{k}}& \frac{w_j^\top \Sigma w_j}{k}
    \end{pmatrix}^{-1}
    \begin{pmatrix}
    x\\y
    \end{pmatrix}
    }}{\sqrt{\det 2\pi \begin{pmatrix}
    \Psi_{ii} & \frac{\Phi_i w_j}{\sqrt{k}}\\
    \frac{\Phi_i w_j}{\sqrt{k}}& \frac{w_j^\top \Sigma w_j}{k}
    \end{pmatrix}}}x\sigma(y)
\end{align}
As before, the random variable $\sfrac{w_j^\top \Sigma w_j}{k}$ concentrate around its mean value 
\begin{align}
    r\equiv \frac{\Delta}{k}\tr \Omega
\end{align}
Plugging this definition into the above:
\begin{align}
    \Phi^h_{ij}&=\int \frac{e^{-\frac{1}{2}\frac{1}{\Psi_{ii} r-\mathcal{O}\left(\frac{1}{d}\right)}(rx^2+\Psi_{ii}y^2)}e^{\frac{1}{\Psi_{ii}  r-\mathcal{O}\left(\frac{1}{d}\right)}\frac{\Phi_i w_j}{\sqrt{k}}xy}}{2\pi\sqrt{\Psi_{ii}  r-\mathcal{O}\left(\frac{1}{d}\right)}}x\sigma(y)\notag\\
    &=\left(\int \frac{e^{-\frac{1}{2\Psi_{ii} }z^2}}{\sqrt{2\pi \Psi_{ii} }}z\right)\left(\int \frac{e^{-\frac{1}{2r}z^2}}{\sqrt{2\pi r}}\sigma(z)\right)+\frac{1}{\Psi_{ii}   r}
    \frac{\Phi_i w_j}{\sqrt{k}}
    \left(\int \frac{e^{-\frac{1}{2\Psi_{ii}}z^2}}{\sqrt{2\pi \Psi_{ii}}}z^2\right)\left(\int \frac{e^{-\frac{1}{2r}z^2}}{\sqrt{2\pi r}}z\sigma(z)\right)+\mathcal{O}\left(\frac{1}{d}\right)\notag\\
    &:=\kappa_1\times  \frac{\Phi_i w_j}{\sqrt{k}}
\end{align}

yielding
\begin{align}
\label{App:eq:propgation_Phi_HMM}
    \Phi^h=\kappa_1\frac{\Phi W^\top}{\sqrt{k}}
\end{align}
with
\begin{align}
    \kappa_1=\frac{1}{r}\mathbb{E}_{z}^{\mathcal{N}(0,r)}\left[z\sigma(z)\right].
\end{align}

\paragraph{Iterating} To establish \eqref{eq:Phi_lin_DRM}, we iterate \eqref{App:eq:propgation_Phi_DRM} $\min(L,L_\star)$ times, and followed by $\max(L,L_\star)-\min(L,L_\star)$ iterations of the single layer relation \eqref{App:eq:propgation_Phi_HMM}, so as to finish propagating the data through the deeper (teacher \eqref{eq:target} or student \eqref{eq:definition_multilayer_RF}) network.

\subsection{Spectrum of the covariances}
In this section, we derive the spectrum of the linearized covariance \eqref{eq:Omega_lin_recursion}, which is a result of indenpendent interest.

\paragraph{Useful identities}
We remind first some useful facts. For $W\in\mathbb{R}^{k_\ell\times k_{\ell-1}}$ with i.i.d Gaussian entries and $\Sigma\in\mathbb{R}^{k_{\ell-1}\times k_{\ell-1}}$ a deterministic matrix admitting a limiting spectral density $\mu_{\ell-1}$ as $k_{\ell-1}\rightarrow  \infty$, we have, from the fact that $XX^\top$ and $X^\top X$ share the same spectrum up to a zero eigenvalues,
\begin{align}
    \mu_{\frac{1}{k_{\ell-1}}W\Sigma W^\top}=\frac{k_{\ell-1}}{k_{\ell}}\times\left[\frac{k_\ell}{k_{\ell-1}}\otimes\mu_{\frac{1}{k_{\ell}}\Sigma^{\frac{1}{2}}W^\top W \Sigma^{\frac{1}{2}}}\right]+\frac{k_\ell-k_{\ell-1}}{k_{\ell}}\delta
\end{align}
The spectrum of $\frac{1}{k_{\ell}}\Sigma^{\frac{1}{2}}W^\top W \Sigma^{\frac{1}{2}}$ is given by
\begin{equation}
    \mu_{\mathrm{MP}}^{\frac{k_{\ell-1}}{k_\ell}}\boxtimes \mu_{\ell-1}
\end{equation}
where $ \mu_{\mathrm{MP}}^{\frac{k_{\ell-1}}{k_\ell}}$ is the Marcenko-Pastur distribution with aspect ratio $k_{\ell-1}/{k_\ell}$. In terms of Stieltjes transforms:
\begin{align}
    m_{\frac{1}{k_{\ell-1}}W\Sigma W^\top}(z)=\left(\frac{k_{\ell-1}}{k_\ell}\right)^2\times m_{\frac{1}{k_{\ell}}\Sigma^{\frac{1}{2}}W^\top W \Sigma^{\frac{1}{2}}}\left(
    \frac{k_{\ell-1}}{k_\ell} z
    \right)+\left(\frac{k_{\ell-1}}{k_\ell}-1\right)\frac{1}{z}
\end{align}
Using the Marcenko-Pastur map and using the shorthand $\gamma_\ell=\frac{k_{\ell-1}}{k_\ell}$, we reach that the Stieltjes transform for $\frac{1}{k_{\ell-1}}W\Sigma W^\top$ is the solution of
\begin{align}
    m(z)=\int \frac{
    (\gamma_\ell-1)x m(z)-\gamma_\ell}{zxm(z)+\gamma_\ell z}d\mu_{\ell-1}(x)=\frac{\gamma_\ell-1}{z}-\frac{\gamma_\ell^2}{z}\int \frac{1}{xm(z)+\gamma_\ell}d\mu_{\ell-1}(x)
\end{align}

\paragraph{Spectrum $\Omega_\ell^{\mathrm{lin}}$}
 The spectral distribution $\mu_\ell$ of $\Omega_\ell^{\mathrm{lin}}$ is then given by the recursion relation
\begin{align}
\label{eq:App:Pop:recursion_mu}\mu_\ell=\left(\kappa_1^{\ell}\right)^2\otimes\left[\frac{k_{\ell-1}}{k_{\ell}}\times\left[\frac{k_\ell}{k_{\ell-1}}\otimes\mu_{\mathrm{MP}}^{\frac{k_{\ell-1}}{k_\ell}}\boxtimes \mu_{\ell-1}\right]+\frac{k_\ell-k_{\ell-1}}{k_{\ell}}\delta\right]\oplus\left(\kappa_*^\ell\right)^2
\end{align}
with initial condition $\mu_0=\mu_{\Omega_0}$. This translates to

\begin{align}
\label{eq:App:Pop:update_m}
    m_\ell(z)=\frac{\gamma_\ell-1}{z-\left(\kappa_*^\ell\right)^2}- \frac{\gamma_\ell^2\left(\kappa_1^\ell\right)^2}{(z-\left(\kappa_*^\ell\right)^2)m_\ell(z)}m_{\ell-1}\left(-\frac{\gamma_\ell}{\left(\kappa_1^\ell\right)^2m_\ell(z)}\right).
\end{align}

\paragraph{Numerical scheme}
We now discuss a numerical scheme to solve \eqref{eq:App:Pop:update_m}. Note that each $m_{\ell-1}$ is only evaluated at
\begin{equation}
    z_{\ell-1}\equiv -\frac{\gamma_\ell}{\left(\kappa_1^\ell\right)^2m_\ell(z)}
\end{equation}

To solve this numerically we keep two arrays $(m_0,...,m_L)$ and $(z_0,...z_L)$, with $m_\ell\equiv m_\ell(z_\ell)$. For simplicity consider the case where the input covariance is identity, meaning
\begin{equation}
    m_0(z_0)=\frac{1}{1-z_0}
\end{equation}
Then until convergence we iterate
\begin{align}
    \forall 0\le i\le \ell-1,\qquad z_i\leftarrow -\frac{ \gamma_{i+1}}{\left(\kappa_1^i\right)^2m_{i+1}}
\end{align}
and keep
\begin{align}
    z_L=\lambda +i\eta
\end{align}
with $\eta=0^+$ and $\lambda$ the value at which we wish to evaluate the density $\mu_L(\lambda)$. Then we update
\begin{align}
    \forall 1\le i\le L,\qquad m_i\leftarrow \frac{\gamma_i-1-\sqrt{(\gamma_i-1)^2-4\frac{m_{i-1}\gamma_i^2}{\left(\kappa_1^i\right)^2}\left(z_i-\left(\kappa_*^i\right)^2\right)}}{2\left(z_i-\left(\kappa_*^i\right)^2\right)}
\end{align}
where we solved the update \eqref{eq:App:Pop:update_m} directly, which is empirically yielding better convergence than directly iterating \eqref{eq:App:Pop:update_m}. 

\begin{figure}[t!]
    \centering
\includegraphics[scale=0.55]{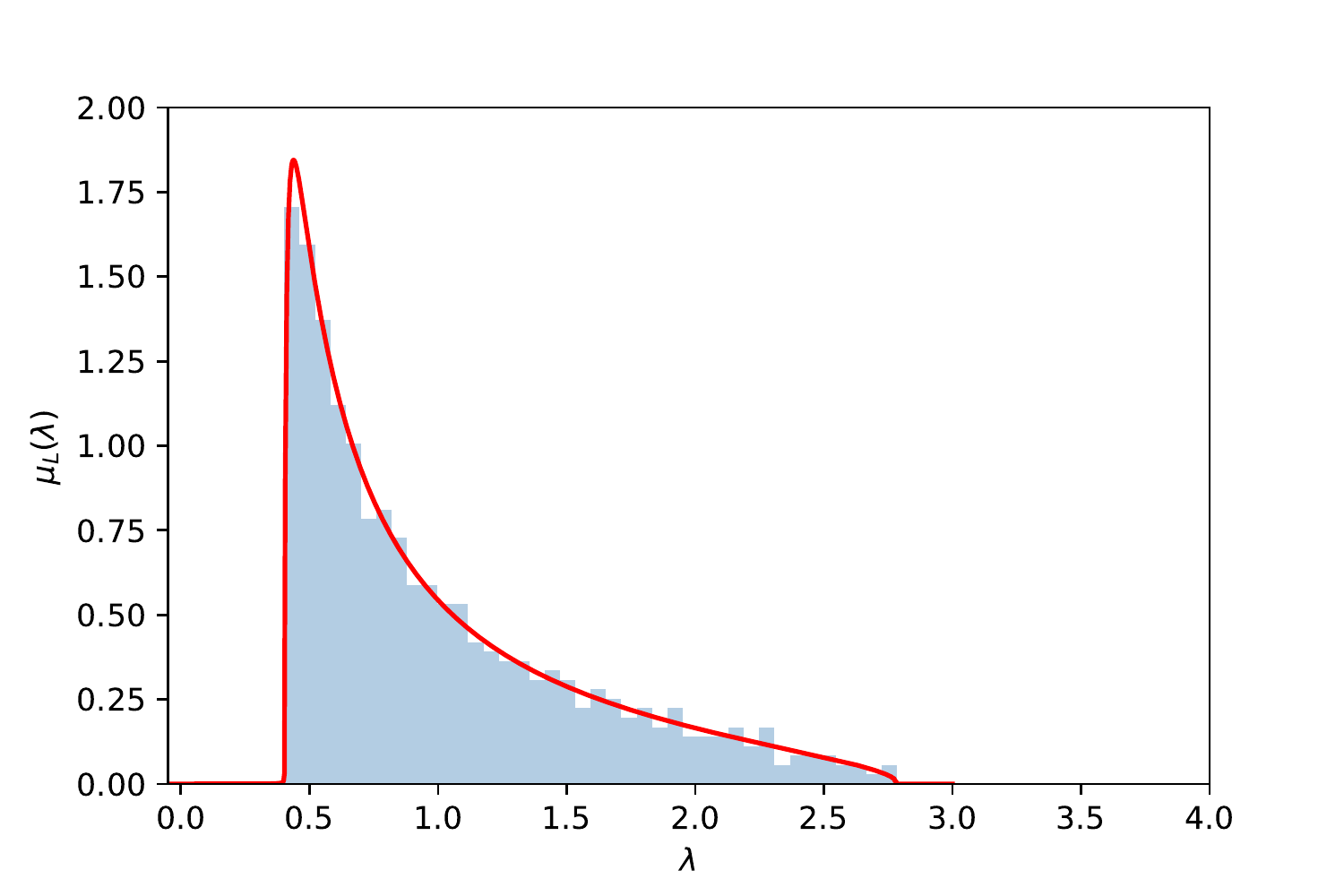}
\includegraphics[scale=0.55]{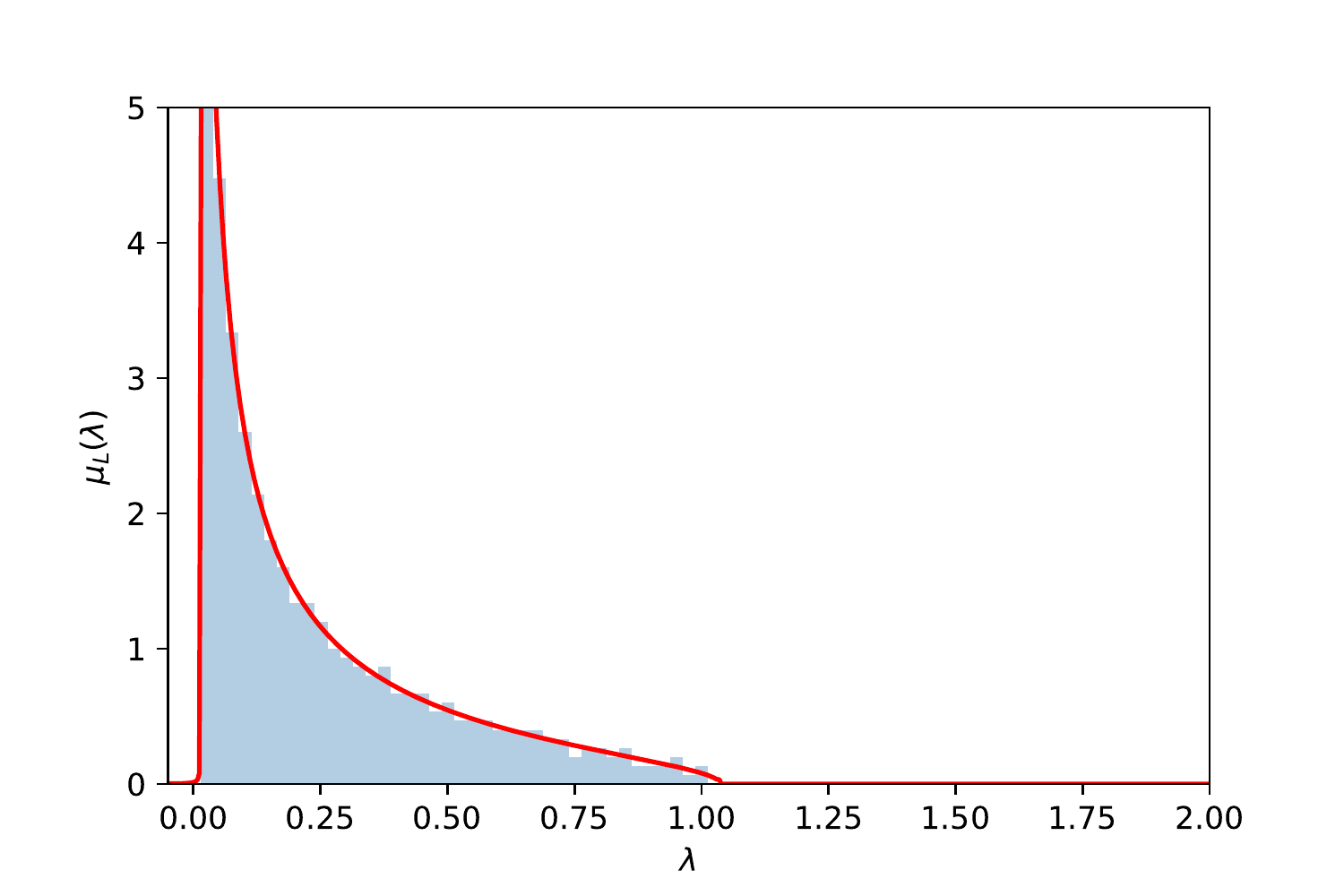}
    \vspace{-0.5cm}
    \caption{Limiting spectral distributions for the post-activation covariance $\Omega_2$ \eqref{eq:Omega_lin_recursion} of a $2-$ hidden layers network \eqref{eq:definition_multilayer_RF}, with architectures $\gamma_1=\sfrac{6}{5},\gamma_2=\sfrac{3}{5}$ and activation $\sigma_1=\sigma2 =\tanh(2\cdot)$ (top), and $\gamma_1=\sfrac{7}{10},\gamma_2=6/5$ and activation  $\sigma_1=\sigma2 =\mathrm{sign}$ (bottom)
    (red) Theoretical asymptotic spectral distribution obtained from solving the recursion \eqref{eq:App:Pop:recursion_mu} (see Appendix \ref{App:population_heuristic} for further details on the numerical scheme) (blue) Empirical distribution, estimated from the sample covariance of $10^5$ samples, in dimension $d=1000$. }
    \label{fig:spectra}
\end{figure}

\begin{figure}[t!]
    \centering
    \includegraphics[scale=0.55]{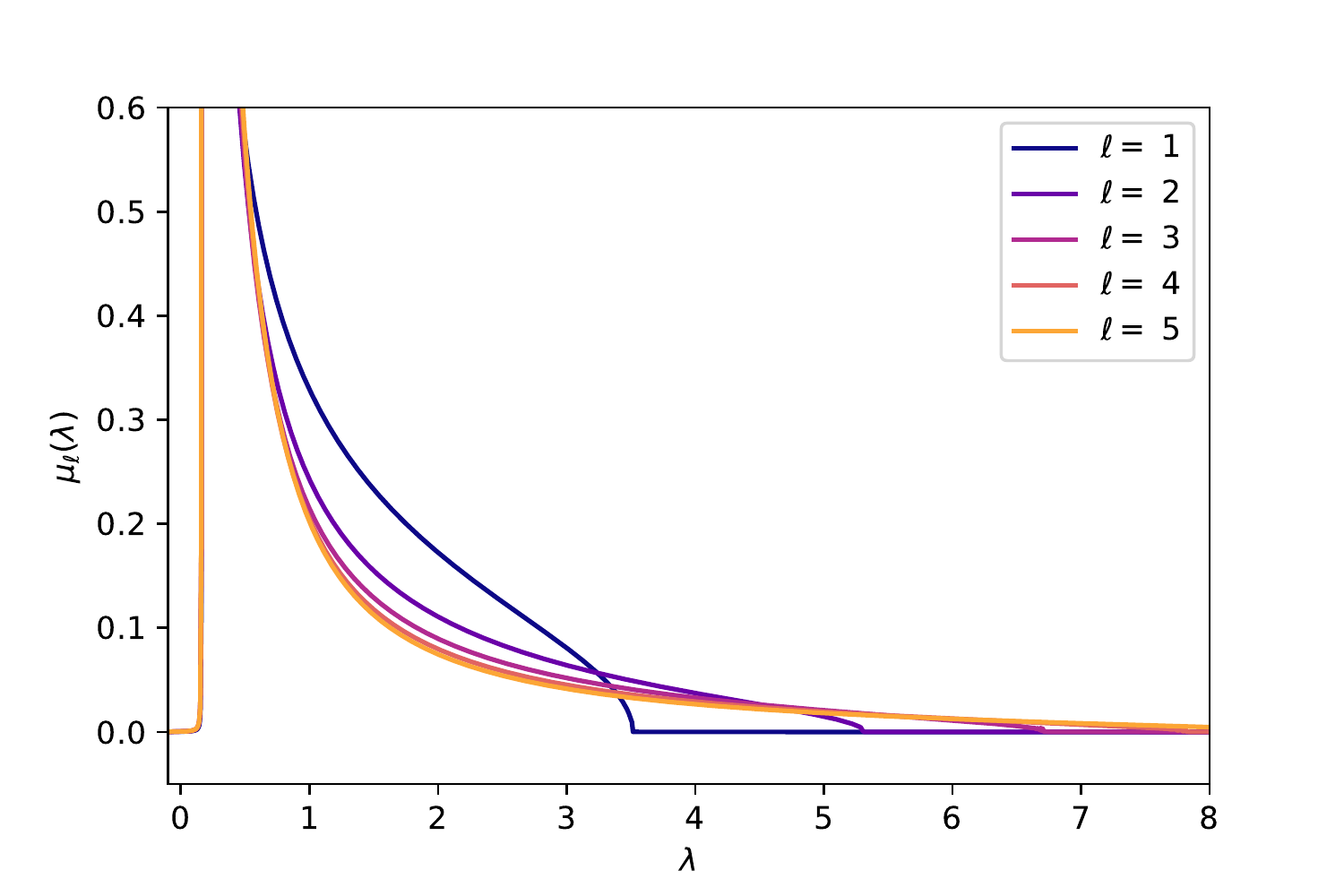}
    \vspace{-0.5cm}
    \caption{Evolution of the asymptotic spectral distribution $\mu_\ell$ of the post-activations $h_\ell(\x)$, for $1\le \ell \le L=5$, for a network with architecture $\gamma_1=...=\gamma_5=1$ and $\sigma_1=...\sigma_5=\tanh$ activation, and isotropic data $\Omega_0=I_d$. Propagation through non-linear layers tends to extend the support of the distribution, and also increase the density of small eigenvalues.}
    \label{fig:layerwise_spectrum}
\end{figure}

Fig.~\ref{fig:spectra} shows the theoretical asymptotic distribution \eqref{eq:App:Pop:recursion_mu} for $3-$layer RFs with sigmoid and sign activations, which is found to display excellent agreement with numerical estimations of the population covariance estimated with $10^5$ independent samples. Fig.~\ref{fig:layerwise_spectrum} shows the asymptotic distribution across $L=5$ layers for a rectangular $\tanh$ network. In alignment to the observations of \cite{Fan2020SpectraOT} for the conjugate kernel in similar models, the support of the distribution increases with depth, alongside an increase in the density of small eigenvalues. Note that the presence of small eigenvalues has been linked in a variety of settings \cite{Cui2021GeneralizationER, Mei2021GeneralizationEO,Misiakiewicz2022SpectrumOI,Hu2022SharpAO} to an effective additional implicit $\ell_2$ regularization when using \eqref{eq:definition_multilayer_RF} to perform regression. This intuition is further discussed in Section \ref{sec:architecture}.

%%%%%%%%%%%%%%%%%%%%%%%%%%%%%%%%%%%%%%%%%%%%%%%%%%%%%%%%%%%%%%%%%%%%%%%%%%%%%%%
\newpage
\section{Error universality of ridge regression}
\label{App:error:ridge}
In this Appendix we provide a detailed derivation of \eqref{eq:ridge:rmt} and Corollary \ref{prop:universality:ridge}. First, we start by recapping the setting for this Corollary. Here, we are interested in characterizing the asymptotic mean-squared test error:
\begin{align}
\label{eq:app:testmse}
\mathcal{E}_{\rm{gen.}}(\hat{\theta}) = \mathbb{E}\left(y-\frac{\hat{\theta}^{\top}\varphi(\x)}{\sqrt{k}}\right)^2
\end{align}
\noindent where $\varphi:\mathbb{R}^{d}\to\mathbb{R}^{k}$ are the $L$-layers random features defined in \eqref{eq:definition_multilayer_RF} and $\hat{\theta}\in\mathbb{R}^{k}$ is the ridge estimator:
\begin{align}
\label{eq:app:ridge:sol}
\hat{\theta} &= \text{argmin}\left[\sum\limits_{\mu=1}^{n}\left(y^{\mu}-\frac{\hat{\theta}^{\top}\varphi(\x^{\mu})}{\sqrt{k}}\right)^2+\frac{\lambda}{2}||\theta||^2_{2}\right]\notag\\
&=\frac{1}{\sqrt{k}}\left(\lambda I_{k}+\frac{1}{k}X_{L}X_{L}^{\top}\right)^{-1}X_{L}y
\end{align}
\noindent where, following the notation in the main, we have defined the features matrix $X_{L}\in\mathbb{R}^{k\times n}$ by stacking together $\varphi(\x^{\mu})$ column-wise and the label vector $y\in\mathbb{R}^{n}$. In particular, in Corollary \ref{prop:universality:ridge} we focus in the case where the labels are generated, up to additive Gaussian noise, by a $L$-layers random features target with the same architecture. Explicitly, this can be written as:
\begin{align}
\label{eq:app:ridge:teacher}
y^{\mu} = \frac{\theta_{\star}^{\top}\varphi(\x^{\mu})}{\sqrt{k}}+z^{\mu}
\end{align}
where $\theta_{\star}\sim\mathcal{N}(0_{k},I_{k})$ and $z^{\mu}\sim\mathcal{N}(0,\Delta)$ independently. Note that, for the purposes of the discussion here we do not need to assume the inputs $\x^{\mu}\in\mathbb{R}^{d}$ are Gaussian, but only that the data matrix $X_{0}\in\mathbb{R}^{d\times n}$ satisfies the concentration condition \eqref{X0 assump}. In particular, this implies that the results in this Appendix hold for the test error \eqref{eq:app:testmse} conditionally on the training inputs $X_{0}$.

From here, the computation is standard, and closely follows other works deriving closed-form asymptotics for ridge regression under different assumptions, e.g. \cite{Karoui2013, Dobriban2015HighDimensionalAO, Wu2020OnTO, Hastie2019SurprisesIH}. First, note we can rewrite:
\begin{align}
\label{eq:app:mse}
\mathcal{E}_{\rm{gen.}}(\hat{\theta}) &= \mathbb{E}_{\z,\theta_{\star},z,\x}\left(\frac{\theta_{\star}^{\top}\varphi(\x)}{\sqrt{k}}+z - \frac{\hat{\theta}^{\top}\varphi(\x)}{\sqrt{k}}\right)^2 \notag\\
&\overset{(a)}{=} \frac{1}{k}\mathbb{E}_{,\z,\theta_{\star}}(\hat{\theta}-\theta_{\star})^{\top}\mathbb{E}_{\x}\left[\varphi(\x)\varphi(\x)^{\top}\right](\hat{\theta}-\theta_{\star})+\Delta\notag\\
&\overset{(b)}{=} \frac{1}{k}\mathbb{E}_{\z,\theta_{\star}}\left[(\hat{\theta}-\theta_{\star})^{\top}\Omega_{L}(\hat{\theta}-\theta_{\star})\right]
\end{align}
\noindent where in (a) we used the independence of the test sample took the $z$ average explicitly, and in (b) we have used the definition \eqref{eq:pop_cov_def}. Focusing on:
\begin{align}
\label{eq:app:diff}
	\hat{\theta} - \theta_{\star} &= \sfrac{1}{\sqrt{k}}\left(\lambda I_{k} + \sfrac{1}{k}X_{L}X_{L}^{\top}\right)^{-1}X_{L}^{\top}(\sfrac{1}{\sqrt{k}}X_{L}^{\top}\theta_{\star}+\z) - \theta_{\star}\notag\\
	&= \left(\lambda I_{k} + \sfrac{1}{k}X_{L}X_{L}^{\top}\right)^{-1}\left(\sfrac{1}{k}X_{L}X_{L}^{\top}-I_{k}\right) \theta_{\star} + \sfrac{1}{\sqrt{k}}\left(\lambda I_{k} + \sfrac{1}{k}X_{L}X_{L}^{\top}\right)^{-1}X\z\\
	&\overset{(c)}{=}-\lambda \left(\lambda I_{k} + \sfrac{1}{k}X_{L}X_{L}^{\top}\right)^{-1}\theta_{\star} + \sfrac{1}{\sqrt{k}}\left(\lambda I_{k} + \sfrac{1}{k}X_{L}X_{L}^{\top}\right)^{-1}X\z
\end{align}
\noindent where in (c) we have used the following version of the Woodbury identity:
\begin{align}
\lambda\left(\lambda I_{k} + \sfrac{1}{k}X_{L}X_{L}^{\top}\right)^{-1} = I_{k} - \sfrac{1}{k}\left(\lambda I_{k} + \sfrac{1}{k}X_{L}X_{L}^{\top}\right)^{-1}X_{L}X_{L}^{\top}	
\end{align}
Inserting the above in \eqref{eq:app:mse}:
\begin{align}
\label{eq:app:averaging}
\mathcal{E}_{\rm{gen.}}(\hat{\theta}) &= \sfrac{\lambda^2}{k}~\mathbb{E}_{\theta_{\star}}\left[\theta_{\star}^{\top}\left(\lambda I_{k} + \sfrac{1}{k}X_{L}X_{L}^{\top}\right)^{-1}\Omega_{L}\left(\lambda I_{k} + \sfrac{1}{k}X_{L}X_{L}^{\top}\right)^{-1}\theta_{\star}\right]+\notag\\
&\qquad +\sfrac{1}{k^2}~\mathbb{E}_{\z}\left[\z^{\top}X_{L}^{\top}\left(\lambda I_{k} + \sfrac{1}{k}X_{L}X_{L}^{\top}\right)^{-1}\Omega_{L}\left(\lambda I_{k} + \sfrac{1}{k}X_{L}X_{L}^{\top}\right)^{-1}X_{L}\z\right] + \Delta\notag\\
&\overset{(d)}{=} \lambda^2~\left\langle\left(\lambda I_{k} + \sfrac{1}{k}X_{L}X_{L}^{\top}\right)^{-1}\Omega_{L}\left(\lambda I_{k} + \sfrac{1}{k}X_{L}X_{L}^{\top}\right)^{-1}\right\rangle+\notag\\
&\qquad +\Delta\left\langle\sfrac{1}{k}X_{L}X_{L}^{\top}\left(\lambda I_{k} + \sfrac{1}{k}X_{L}X_{L}^{\top}\right)^{-1}\Omega_{L}\left(\lambda I_{k} + \sfrac{1}{k}X_{L}X_{L}^{\top}\right)^{-1}\right\rangle + \Delta
\end{align}
\noindent where in (d) we took the expectations over the noise and target weights and used the definition $\langle \cdot \rangle \equiv \sfrac{1}{k}\tr(\cdot)$ with the cyclicity of the trace. We can put the expression above in a shape in which Theorem \ref{prop:mult_layers} apply by adding and subtracting $\lambda I_{k}$ to $\sfrac{1}{k}X_{L}X_{L}^{\top}$ the second trace term. This leads to the expression \eqref{eq:ridge:rmt} quoted in the main text:
\begin{align}
\mathcal{E}_{\rm{gen.}}(\hat{\theta}) &= \Delta\left(\left\langle\Omega_{L}\left(\lambda I_{k}+\sfrac{1}{k}X_{L}X_{L}\right)^{-1}\right\rangle+1\right)+\lambda(\lambda-\Delta)\left\langle\Omega_{L}\left(\lambda I_{k}+\sfrac{1}{k}X_{L}X_{L}\right)^{-2}\right\rangle\notag\\
&=\Delta\left(\left\langle\Omega_{L}\left(\lambda I_{k}+\sfrac{1}{k}X_{L}X_{L}\right)^{-1}\right\rangle+1\right)-\lambda(\lambda-\Delta)\partial_{\lambda}\left\langle\Omega_{L}\left(\lambda I_{k}+\sfrac{1}{k}X_{L}X_{L}\right)^{-1}\right\rangle
\end{align}
Note that the last expression requires applying~\Cref{prop:mult_layers} to the derivative of the resolvent. In general, this can be justified by writing a squared resolvent \(G(z)^2=(H-z)^{-2}\) of some non-negative matrix $H\ge0$ in terms of a Cauchy-integral 
\begin{equation}
    G(z)^2=G'(z)=\frac{1}{2\pi\mathrm{i}} \oint_\gamma \frac{1}{(w-z)^2} G(w)\dif w,
\end{equation}
where \(\gamma\) is any contour around \(z\) not crossing \(\R_+\). In this way some local law of the type \(\abs{\langle A(G-M)\rangle}\prec \epsilon \dist(z,\R_+)^{-k}\) can be transferred to the derivative as 
\begin{equation}\label{derivative trick}
    \langle A(G'(z)-M'(z)\rangle = \frac{1}{2\pi\mathrm{i}} \oint_\gamma \frac{1}{(w-z)^2}\langle A(G(w)-M(w))\rangle \dif w = O_\prec\Bigl(\frac{\epsilon}{\dist(z,\R_+)^{k+1}}\Bigr),
\end{equation}
by choosing \(\gamma\) to be a small circle of radius \(\dist(z,\R_+)/2\) around \(z\), using that the deterministic equivalent is also holomorphic away from \(\R_+\). %Technically, as the local law is a probabilistic statement for any fixed spectral parameter, the application inside the integral in~\cref{derivative trick} requires a standard union bound over a discrete grid together with the H\"older-continuity of \(G,M\).

Therefore, in the high-dimensional limit where $n,k_{\ell}, d\to\infty$ at fixed ratios $\alpha =\sfrac{n}{d}$ and $\gamma_{\ell}=\sfrac{k_{\ell}}{d}$, under the assumptions of Theorem \ref{prop:mult_layers} for the input data $X_{0}\in\mathbb{R}^{d\times n}$ \eqref{X0 assump} and the architecture of the deep random features and for $\lambda>0$\footnote{Technically, we don't need to assume the regularization is bounded away from here. It suffices to take it decaying slower than $n^{-\sfrac{1}{18}}$ for Thm. \ref{prop:mult_layers} to apply.} we can apply Theorem \ref{prop:mult_layers} to write the asymptotic limit of the test error:
\begin{align}
\label{eq:app:asymptotic:error}
\lim\limits_{k\to\infty}\mathcal{E}_{\rm{gen}.}(\hat{\theta}) = \mathcal{E}_{\rm{gen.}}^{\star}(\lambda,\Delta, \alpha,\gamma_{\ell}, \kappa_{1}^{\ell},\kappa_{\star}^{\ell}) \equiv \Delta\left(\langle\Omega_{L}\rangle\wc m_{L}(-\lambda)+1\right)
-\lambda(\lambda-\Delta)\langle\Omega_{L}\rangle\partial_{\lambda}\wc m_{L}(-\lambda)
\end{align}
\noindent where $\wc m_{L}(z)$ can be computed recursively from \eqref{eq:companion_stieltjes} for a given regularization strength $\lambda>0$, noise level $\Delta>0$, sample complexity $\alpha>0$ and features architecture $(\gamma_{\ell},\sigma_{\ell})_{\ell\in[L]}$.
On the other hand, it follows from the recursion \eqref{eq:Omega_lin_recursion} that the trace of the last-layer covariance $\langle\Omega_{L}\rangle$ admits the compact expression
\begin{equation}
\label{eq:app:traceomega}
    \langle\Omega_{L}\rangle=\sum\limits_{\ell=1}^{L-1}\left(\kappa_*^\ell\right)^2\prod\limits_{\ell^\prime=\ell+1}^L \left(\kappa_1^{\ell^\prime}\right)^2\Delta_{\ell^\prime}
    +\left(\kappa_*^L\right)^2+\frac{1}{d}\langle\Omega_{0}\rangle\prod\limits_{\ell=1}^L \left(\kappa_1^{\ell}\right)^2\Delta_\ell
\end{equation}
in terms only of the coefficients \eqref{eq:kappa_multilayer}.

Note that \eqref{eq:app:asymptotic:error} agrees exactly with the formula for the asymptotic test error of ridge regression on a equivalent Gaussian dataset $\mathcal{D} = \{(v^{\mu}, y^{\mu})\}_{\mu\in[n]}$:
\begin{align}
y^{\mu} = \frac{\theta_{\star}^{\top}v^{\mu}}{\sqrt{k}}+z^{\mu}, && v^{\mu}\sim\mathcal{N}(0_{k}, \Omega_{L}).
\end{align}
\noindent which, to our best knowledge, was first derived in \cite{Dobriban2015HighDimensionalAO}. This establishes the Gaussian universality of the asymptotic test error for this model.
%%%%%%%%%%%%%%%%%%%%%%%%%%%%%%%%%%
\subsection{Possible extensions}
\label{sec:app:extensions}
%%%%%%%%%%%%%%%%%%%%%%%%%%%%%%%%%%
We now discuss some possible extensions of the universality result above. They require, however, a more involved analysis, which we leave for future work. Our goal here is simply to highlight other possible applications of our deterministic equivalent in Thm. \ref{prop:mult_layers}. 

\paragraph{Deterministic last-layer weights: } The first extension is to generalize the result above to deterministic last layer weights $\theta_{\star}$. 
%Indeed, notice that the assumption of $\theta_{\star}\sim\mathcal{N}(0_{k},I_{k})$ is only used in \eqref{eq:app:averaging} to pass to the trace. 
Indeed, \cite{Wei2022} shows that for ridge regression on a deterministic target $y^{\mu} = \sfrac{1}{\sqrt{k}}~\theta_{\star}^{\top}\varphi(\x^{\mu})$\footnote{For simplicity, we discuss the noiseless $\Delta=0$ case here. See Appendix B of \cite{Wei2022} for a discussion of noisy targets}, the test error can be asymptotically estimated from the \emph{generalized cross-validation} (GCV) estimator, defined as:
\begin{align}
\text{GCV}_{\lambda} = \lambda\left\langle (\lambda I_{k}+\hat{\Omega}_{L})^{-1}\right\rangle \mathcal{E}_{\rm{train.}}(\hat{\theta})
\end{align}
\noindent where $\hat{\Omega}_{L}=\sfrac{1}{n}X_{L}X_{L}^{\top}$ is the \emph{sample covariance matrix} of the features and $\mathcal{E}_{\rm{train.}}(\hat{\theta})$ is the training error associated to the ridge estimator:
\begin{align}
\mathcal{E}_{\rm{train.}}(\hat{\theta}) = \frac{1}{n}\sum\limits_{\mu=1}^{n}\left(y^{\mu}-\frac{\hat{\theta}^{\top}\varphi(\x^{\mu})}{\sqrt{k}}\right)^2
\end{align}
In particular, it is shown that:
\begin{theorem}[Thm. 8 of \cite{Wei2022}] 
\label{thm:app:gcv}
Assume that 
\begin{equation}\label{wei assump}
    \abs{\langle \Bigl(\frac{X_L^\top X_L}{n} - z\Bigr)^{-1}\rangle - \wc m(z)}+  \abs{v^\top \Bigl[\Bigl(\frac{X_LX_L^\top}{zn}-I\Bigr)^{-1}+\Bigl(\Omega_L\wc m(z)+1\Bigr)^{-1}\Bigr]v} \prec \frac{\sqrt{\Im \wc m(z)}}{\sqrt{n\Im z}},
\end{equation}
for all deterministic vectors \(v\) with \(v^\top\Omega_L v\le 1\), where 
\begin{equation}
    \wc m(z) = \frac{k_L-n}{nz} + \frac{k_L}{n} m_{\mu(\Omega_L)\boxtimes \mu_\mathrm{MP}^{k_L/n}}(z).
\end{equation}
Then for all $\lambda>0$ it holds that
\begin{align}
\left|\text{GCV}_{\lambda} - \mathcal{E}_{\rm{gen.}}(\hat{\theta})\right| \lesssim n^{-\sfrac{1}{2}+o(1)} \theta_{\star}^{\top}\Omega_{L}\theta_{\star} \left[\frac{||\Omega_{L}||_{\rm{op.}}}{\lambda}+\left(\frac{\tr\Omega_{L}}{\lambda n}\right)^{\sfrac{3}{2}}\right]
\end{align}
\end{theorem}
Applying~\Cref{thm:res_concentration} for fixed weights \(W_1,\ldots,W_L\) shows\footnote{Technically this requires some argument that with high probability the deep RF model with quenched weights satisfies Lipschitz concentration with respect to \(X_0\)} that assumption~\eqref{wei assump} is satisfied in the proportional regime \(k_L\sim n\), up to a worse \(z\)-dependence of the error \(\dist(z,\R_+)^{-9}\) rather than \((\Im \wc m(z)/\Im z)^{1/2}\), and only for bounded vectors \(\norm{v}\lesssim1\). 

Instead, our result~\cref{prop:mult_layers} proves a preliminary version of~\eqref{wei assump} with an explicit deterministic equivalent only depending on the input population covariance \(\Omega_0\) rather than the output population covariance \(\Omega_L\), at the price of having an error which is larger by a factor of \(\sqrt n\). It is an interesting question whether our error rates can be improved to imply~\cref{wei assump} which is left for future work.  

\paragraph{General case:} As discussed in the introduction, in the general case we are interested in a target:
 \begin{align}
 \label{eq:app:target}
     f_{\star}(\x^{\mu})=\frac{1}{\sqrt{k_{\star}}}\theta_{\star}^{\top}\varphi_{\star}(\x^{\mu}), && \theta_{\star}\sim\mathcal{N}(0_{k_{\star}},I_{k_{\star}}),
 \end{align}
where the $L_{\star}$ multi-layer random features $\varphi_{\star}:\mathbb{R}^{d}\to\mathbb{R}^{k_{\star}}$ are not necessarily the same as the $L$ multi-layer random features $\varphi:\mathbb{R}^{d}\to\mathbb{R}^{k}$. As discussed in the introduction, this contains as a special case the \emph{hidden-manifold model} (HMM), introduced in \cite{Goldt2020ModellingTI} as a model for structured high-dimensional data where the labels depend only on the coordinates of a lower-dimensional "latent space". While in Section \ref{sec:Error_uni} we provide an exact but heuristic formula to compute the error in this case (valid for arbitrary convex losses), the challenge in proving it with random matrix theory methods in the case of ridge regression comes from the fact that this is a mismatched model. Indeed, naively writing the expression for the test error in this case:
\begin{align}
\label{eq:app:mse:mismatched}
\mathcal{E}_{\rm{gen.}}(\hat{\theta}) &= \mathbb{E}_{\theta_{\star},\x}\left(\frac{\theta_{\star}^{\top}\varphi_{\star}(\x)}{\sqrt{k_{\star}}}+z - \frac{\hat{\theta}^{\top}\varphi(\x)}{\sqrt{k}}\right)^2 \notag\\
&= \langle\Psi_{L_{\star}}\rangle+\frac{2}{\sqrt{k_{\star}k}}\mathbb{E}_{\theta_{\star}}\left[\theta_{\star}^{\top}\Phi_{L_{\star}L}\hat{\theta}\right] + \frac{1}{k}\mathbb{E}_{\theta_{\star}}\left[\hat{\theta}^{\top}\Omega_{L}\hat{\theta}\right]
\end{align}
\noindent where we recall the reader of the definitions:
\begin{align}
\Psi_{\ell} = \mathbb{E}\left[h_{\ell}(\x)h_{\ell}(\x)^{\top}\right], && \Phi_{\ell\ell'} = \mathbb{E}\left[h_{\ell}(\x)h_{\ell'}(\x)^{\top}\right].
\end{align}
Indeed, applying Thm. \ref{prop:mult_layers} to the expression above is not as straightforward as above. To see this, focus on the second term:
\begin{align}
\mathbb{E}_{\theta_{\star}}\left[\theta_{\star}^{\top}\Phi_{L_{\star}L}\hat{\theta}\right] = \tr\left[\Phi_{L_{\star}L}(\lambda I_{k}+\sfrac{1}{k}X_{L}X_{L}^{\top})^{-1}X_{L}^{\top}X_{L_{\star}}\right]
\end{align}
\noindent where we defined the target feature matrix $X_{L_{\star}}\in\mathbb{R}^{k_{\star}\times n}$ with columns given by $\varphi(\x^{\mu})\in\mathbb{R}^{k_{\star}}$. This would, naively, require a more refined deterministic equivalent than Thm. \ref{prop:mult_layers} provides. Possible alternative approaches would be to rewrite the misspecification as an effective additive noise (e.g. as in Appendix B of \cite{Clarte2022ASO}) and derive a local-law akin to Assumption \ref{wei assump} with a control over the noise (see Appendix B of \cite{Wei2022} for a discussion) or to use the linear pencil method as in \cite{Mei2019TheGE}. This provides an interesting avenue for future work. 
%%%%%%%%%%%%%%%%%%%%%%%%%%%%%%%%%%%%%%%%%%%%%%%%%%%%%%%%%%%%%%%%%%%%%%%%%%%%%%%
\newpage
\section{Exact asymptotics for the general case}
\label{App:error:general}
In this appendix, we detail the sharp asymptotic characterization for the test error of the dRF \eqref{eq:definition_multilayer_RF} on a deep random network target \eqref{eq:target}, for regression (Fig.\,\ref{fig:regression}) and classification (Fig.\,\ref{fig:logistic}).

The backbone of the derivation is the theorem of \cite{Loureiro2021CapturingTL}, which fully characterizes the test error of the GCM \eqref{eq:g3m} in terms of the covariance matrices $\Psi_{L_\star},\Omega_L$ and $\Phi_{L_\star L}$. In the original work of \cite{Loureiro2021CapturingTL}, these matrices for the dRF model had to be estimated numerically through a Monte-Carlo algorithm. In the present work however, the closed-form expressions afforded by \eqref{eq:Omega_lin_recursion}, \eqref{eq:Psi_lin_recursion} and \eqref{eq:Phi_lin_DRM}, which we remind in the next subsection, now afford a way to access fully analytical formulas. We successively detail these characterizations for ridge regression and logistic regression readouts.

\subsection{Reminder of second-order statistics of network activations}
Before providing detailed asymptotic characterizations for the test error of ridge and logistic regression, we first provide a reminder for the expressions of the linearized matrices $\Psi_{L_\star}^{\mathrm{lin}},\Omega_L^{\mathrm{lin}}$ and $\Phi_{L_\star L}^{\mathrm{lin}}$ (\ref{eq:Omega_lin_recursion},\ref{eq:Psi_lin_recursion},\ref{eq:Phi}). Using conjecture \ref{conj:error_uni_lin}, these matrices can then be used in the formulas of \cite{Loureiro2021CapturingTL} to access fully analytical formulas for the test errors, in terms only of the target network weights \eqref{eq:target} and the coefficients \eqref{eq:kappa_multilayer}.
The following expressions follow from expliciting the solution of the recursions (\ref{eq:Omega_lin_recursion},\ref{eq:Psi_lin_recursion},\ref{eq:Phi})
\begin{align}
    &\Omega_L^{\mathrm{lin}}=\left(\prod\limits_{\ell^\prime=1}^L \frac{\kappa_1^{\ell^\prime}W_{\ell^\prime}^\top}{\sqrt{k_{\ell^\prime-1}}}\right)^\top \Omega_0\left(\prod\limits_{\ell^\prime=1}^L \frac{\kappa_1^{\ell^\prime}W_{\ell^\prime}^\top}{\sqrt{k_{\ell^\prime-1}}}\right)+\sum\limits_{\ell=1}^{L-1}\left(\kappa_*^\ell\right)^2\left(\prod\limits_{\ell^\prime=\ell+1}^L \frac{\kappa_1^{\ell^\prime}W_{\ell^\prime}^\top}{\sqrt{k_{\ell^\prime-1}}}\right)^\top\Omega_0\left(\prod\limits_{\ell^\prime=\ell+1}^L \frac{\kappa_1^{\ell^\prime}W_{\ell^\prime}^\top}{\sqrt{k_{\ell^\prime-1}}}\right) + \left(\kappa_*^L\right)^2 I_{k_L}\\
    &\Psi_{L_\star}^{\mathrm{lin}}=\left(\prod\limits_{\ell^\prime=1}^{L_\star}\frac{\kappa_1^{\ell^\prime\star}W_{\ell^\prime}^\top}{\sqrt{k_{\ell^\prime-1}}}\right)^\top\Omega_0\left(\prod\limits_{\ell^\prime=1}^{L_\star} \frac{\kappa_1^{\ell^\prime\star}W_{\ell^\prime}^\top}{\sqrt{k_{\ell^\prime-1}}}\right)+\sum\limits_{\ell=1}^{L_\star-1}\left(\kappa_*^{\ell\star}\right)^2\left(\prod\limits_{\ell^\prime=\ell+1}^{L_\star} \frac{\kappa_1^{\ell^\prime\star}W_{\ell^\prime}^\top}{\sqrt{k_{\ell^\prime-1}}}\right)^\top\Omega_0\left(\prod\limits_{\ell^\prime=\ell+1}^{L_\star} \frac{\kappa_1^{\ell^\prime\star}W_{\ell^\prime}^\top}{\sqrt{k_{\ell^\prime-1}}}\right) + \left(\kappa_*^{L_\star\star}\right)^2 I_{k_{L_\star}}\\
    &\Phi_{LL_\star}^{\mathrm{lin}}=\left(\prod\limits_{\ell=L^\star}^1 \frac{\kappa_1^{\ell\star}W_{\ell}^\star}{\sqrt{k_\ell^\star}}\right)\cdot \Omega_0 \cdot \left( \prod\limits_{\ell=1}^L\frac{\kappa_1^\ell W_\ell^\top}{\sqrt{k_\ell}}\right)
\end{align}
In the special case where the teacher has depth $L_\star=0$ (i.e. possesses an architecture with no hidden layer), the above expression reduce to
\begin{align}
    &\Psi_{0}^{\mathrm{lin}}=\Omega_0\\
    &\Phi_{LL_\star}^{\mathrm{lin}}=\Omega_0 \cdot \left( \prod\limits_{\ell=1}^L\frac{\kappa_1^\ell W_\ell^\top}{\sqrt{k_\ell}}\right).
\end{align}
The $L_\star=0, L=1$ case has been studied in the literature \cite{Goldt2020ModellingTI, Gerace2020GeneralisationEI} as the \textit{Hidden Manifold Model}. The present work encompasses the analysis of its generalization to deep learners with $L>1$ hidden layers.

\subsection{Ridge regression}
We consider the supervised learning problem of training the readout weights $\theta$ of the dRF \eqref{eq:definition_multilayer_RF} on a dataset $\mathcal{D}=\{x^\mu, y^\mu\}_{\mu=1}^n$, with $x^\mu\sim\mathcal{N}(0_d,\Omega_0)$ independently. The labels are given by a deep random network
\begin{equation}
y^\mu=\frac{\theta_\star^\top\varphi_\star(x^\mu)}{\sqrt{k_{L_\star}}}+z^\mu,
\end{equation}
where $z^\mu\sim \mathcal{N}(0,\Delta)$ is a Gaussian additive noise and the teacher feature map is 
\begin{align}
    \varphi_\star(x)=\varphi^\star_{L_\star}\circ...\circ\varphi^\star_1(x).
\end{align}
Note that compared to \eqref{eq:target}, we have adopted the notation $\theta_\star:=W^\star_{L_\star+1}\in\mathbb{R}^{k_{L_\star}}$ for the sake of clarity. Defining
\begin{align}
    \rho:=\frac{\theta_\star^\top \Psi_{L_\star}\theta_\star}{k_{L_\star}},
\end{align}
We consider the problem training the last layer $\theta$ of the learner dRF \eqref{eq:definition_multilayer_RF} with ridge regression, by minimizing the risk
\begin{equation}
    \hat{\theta}=\underset{\theta}{\mathrm{argmin}}\left\{\sum\limits_{\mu=1}^n\left(y^\mu-\frac{\theta^\top \varphi(x^\mu)}{\sqrt{k_L}}\right)^2+\frac{\lambda}{2}||\theta||^2\right\}.
\end{equation}
Building on the theorem of \cite{Loureiro2021CapturingTL} and conjecture \ref{conj:error_uni_lin}, the mean squared error achieved by this ERM algorithm is given by
\begin{align}
    \epsilon_g:=
    \mathbb{E}_{\mathcal{D}}\mathbb{E}_{x\sim\mathcal{N}(0_d,\Omega_0)}\left(
    f_\star(x)-\frac{\hat{\theta}^\top \varphi(x)}{\sqrt{k_L}}
    \right)^2
    =\rho+q-2m,
\end{align}
with $q,m$ the solutions of the system of equations
\begin{align}
    \begin{cases}
        \hat{V} = \frac{1}{\gamma_L}\frac{\alpha}{1+V}\\
        \hat{q} = \frac{1}{\gamma_L}\alpha\frac{\rho+q-2m}{(1+V)^2}\\
        \hat{m} = \frac{1}{\sqrt{\gamma_L\gamma^\star_{L_\star}}}\frac{\alpha}{1+V}
    \end{cases}, && 
    \begin{cases}
		V =  \frac{1}{k_L}\tr\left(\lambda I_{d}+\hat{V} \Omega_L^{\mathrm{lin}}\right)^{-1} \Omega_L^{\mathrm{lin}}\\
		q = \frac{1}{k_L}\tr\left[\left(\hat{q} \Omega_L^{\mathrm{lin}}+\hat{m}^{2}\Phi_{LL_\star}^{\mathrm{lin}\top}\theta_\star\theta_\star^{\top}\Phi_{LL_\star}^{\mathrm{lin}}\right) \Omega_L^{\mathrm{lin}}\left(\lambda I_{d}+\hat{V} \Omega_L^{\mathrm{lin}}\right)^{-2}\right]\\
		m=\sqrt{\frac{\gamma_L}{\gamma^\star_{L_\star}}}\frac{\hat{m}}{k_L}\tr \Phi_{LL_\star}^{\mathrm{lin}\top}\theta_\star\theta_\star^{\top}\Phi_{LL_\star}^{\mathrm{lin}}\left(\lambda I_{d}+\hat{V} \Omega_L^{\mathrm{lin}}\right)^{-1}
	\end{cases}.
\end{align}

\subsection{Logistic regression} We now turn to the classification setting, when the labels are given by a deep random network with sign readout
\begin{equation}
y^\mu=\mathrm{sign}\left(\frac{\theta_\star^\top\varphi_\star(x^\mu)}{\sqrt{k_{L_\star}}}\right).
\end{equation}
Note that this corresponds to $\sigma_{L_\star+1}=\mathrm{sign}$.
and the dRF readout weights $\theta$ are trained with logistic regression, using the ERM
\begin{equation}
    \hat{\theta}=\underset{\theta}{\mathrm{argmin}}\left\{\sum\limits_{\mu=1}^n\ln \left(1+e^{-y^\mu\frac{\theta^\top \varphi(x^\mu)}{\sqrt{k_L}}}\right)+\frac{\lambda}{2}||\theta||^2\right\}.
\end{equation}

By the same token, introducing following \cite{Loureiro2021CapturingTL} the auxiliary functions
$$ Z(y,\omega,V):=\frac{1}{2}\left(
    1+\mathrm{erf}\left(\frac{y\omega}{\sqrt{2V}}\right)
    \right)
$$
and $f(y,\omega,V)$ defined as the solution of
$$
f(y,\omega,V)=\frac{y}{1+e^{y(Vf(y,\omega,V)+\omega)}}.
$$
It follows from \cite{Loureiro2021CapturingTL} and Conjecture \ref{conj:error_uni_lin} that the associated test error reads
\begin{align}
    \epsilon_g:=\mathbb{E}_{\mathcal{D}}\mathbb{P}_{x\sim\mathcal{N}(0_d,\Omega_0)}\left(
    f_\star(x)\ne \mathrm{sign}\left(
    \frac{\hat{\theta}^\top \varphi(x)}{\sqrt{k_L}}
    \right)
    \right)
    =\frac{1}{\pi}\arccos{\frac{m}{\sqrt{\rho q}}},
\end{align}
where $m,q$ are the solutions of the system of equations
\begin{align}
	&\begin{cases}
		\hat{V} = -\frac{\alpha}{\gamma_L}\int \frac{d\xi e^{-\frac{\xi^2}{2}}}{\sqrt{2\pi}}\left[\sum\limits_{y=\pm 1}~Z\left(y,\frac{m}{\sqrt{q}}\xi, \rho-\frac{m^2}{q}\right) ~\partial_{\omega}f(y,\sqrt{q}\xi,V)\right]\\
		\hat{q} = \frac{\alpha}{\gamma_L}\int \frac{d\xi e^{-\frac{\xi^2}{2}}}{\sqrt{2\pi}}\left[\sum\limits_{y=\pm 1}~Z\left(y,\frac{m}{\sqrt{q}}\xi, \rho-\frac{m^2}{q}\right) f(y,\sqrt{q}\xi,V)^2\right]\\
		\hat{m} = \frac{\alpha}{\sqrt{\gamma_L\gamma^\star_{L_\star}}} \int \frac{d\xi e^{-\frac{\xi^2}{2}}}{\sqrt{2\pi}}\left[\sum\limits_{y=\pm 1}~\partial_{\omega}Z\left(y,\frac{m}{\sqrt{q}}\xi, \rho-\frac{m^2}{q}\right)~f(y,\sqrt{q}\xi,V)  \right]
	\end{cases} \notag \\
	&\begin{cases}
		V =  \frac{1}{k_L}\tr\left(\lambda I_{d}+\hat{V} \Omega_L^{\mathrm{lin}}\right)^{-1} \Omega_L^{\mathrm{lin}}\\
		q = \frac{1}{k_L}\tr\left[\left(\hat{q} \Omega_L^{\mathrm{lin}}+\hat{m}^{2}\Phi_{LL_\star}^{\mathrm{lin}\top}\theta_\star\theta_\star^{\top}\Phi_{LL_\star}^{\mathrm{lin}}\right) \Omega_L^{\mathrm{lin}}\left(\lambda I_{d}+\hat{V} \Omega_L^{\mathrm{lin}}\right)^{-2}\right]\\
		m=\sqrt{\frac{\gamma_L}{\gamma^\star_{L_\star}}}\frac{\hat{m}}{k_L}\tr \Phi_{LL_\star}^{\mathrm{lin}\top}\theta_\star\theta_\star^{\top}\Phi_{LL_\star}^{\mathrm{lin}}\left(\lambda I_{d}+\hat{V} \Omega_L^{\mathrm{lin}}\right)^{-1}
	\end{cases}
	\label{eq:app:sp}
\end{align}

%%%%%%%%%%%%%%%%%%%%%%%%%%%%%%%%%%%%%%%%%%%%%%%%%%%%%%%%%%%%%%%%%%%%%%%%%%%%%%%
\newpage
\section{Architecture-induced implicit regularization}
\label{App:architecture}
A seminal pursuit in machine learning research is the theoretical understanding of the interplay between the network architecture and it learning ability. While this is a challenging open question, the study of dRF \eqref{eq:definition_multilayer_RF}, i.e. networks with intermediate layers frozen at initialization, allows to make some headway and gather some preliminary insight into these interrogations. It constitutes a highly stylized, but nonetheless versatile, playground for which some questions can be explored, and which hopefully pave the first preliminary steps in the understanding of networks trained end-to-end.

Section \ref{sec:architecture} in the main text discussed the regularization induced by depth in dRF architectures. In this section, we further explore, using conjecture \ref{conj:error_uni_lin} as a flexible toolbox to access asymptotic test errors, the role of other architectural features in the performance of dRF. Our purpose is mainly to complement the discussion of section \ref{App:architecture}, and highlight some observations of interest. A more complete study falls out of the scope of the present manuscript and is left for future work. In this section, we briefly discuss two questions:
\begin{itemize}
    \item For a fixed number of parameters, is it better to have a deep or wide architecture?
    \item What is the influence of a narrow (bottleneck) hidden layer on the test error?
\end{itemize}

\subsection{Deeper or wider}

For a given number of parameters $m\gamma d$, for $m\in[14]$ we explore the performance of
\begin{itemize}
    \item A rectangular deep net of depth $L=m$ with width sequence $\gamma_1=...=\gamma_m=\gamma$.
    \item A wide net  with widths $\gamma_1=\gamma_3=\gamma$ and $\gamma_2=(m-2)\gamma$ of depth $L=3$. 
\end{itemize}
for $m\ge 3$, learning from a two-layers target with sign activation. The activation is taken to be $\tanh$ for all layers, in both networks. Note that in both architectures, the number of trainable parameter is also always the same, since the readout layer is in any case of width $\gamma d$.

Fig.\,\ref{fig:depth_width} compares the deep architecture (dashed lines) with the wide architecture (solid lines). In general, the wide architecture provides smaller test errors, in accordance with the intuition that additional layers introduce more effective noise and therefore generically prove detrimental to the learning ability of the dRF, see Fig.\,\ref{fig:depth_width}, right panel. However, as discussed in section \ref{sec:architecture} in the main text, the implicit regularization induced by this noise can help mitigate overfitting in some regimes. This is in particular the case in the vicinity of the interpolation peaks, for noisy targets. The left panel of Fig.\,\ref{fig:depth_width} shows such a case, where deep architectures outperform wide architectures in small data regimes $\alpha=0.5, 1$. If the explicit regularization $\lambda$ is optimized over, this effect disappears.

\begin{figure}[t!]
    \centering
    \includegraphics[scale=0.5]{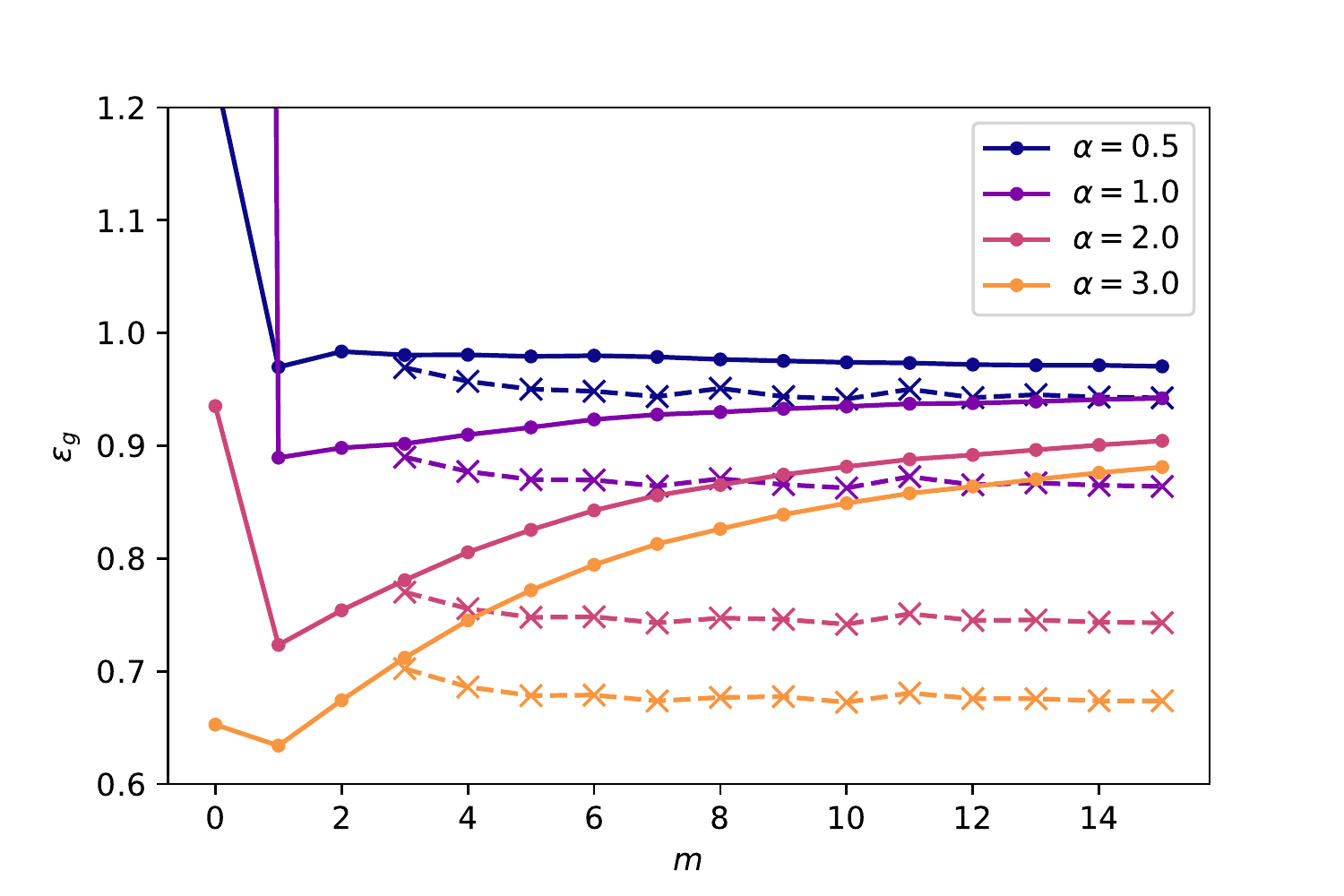}
    \includegraphics[scale=0.5]{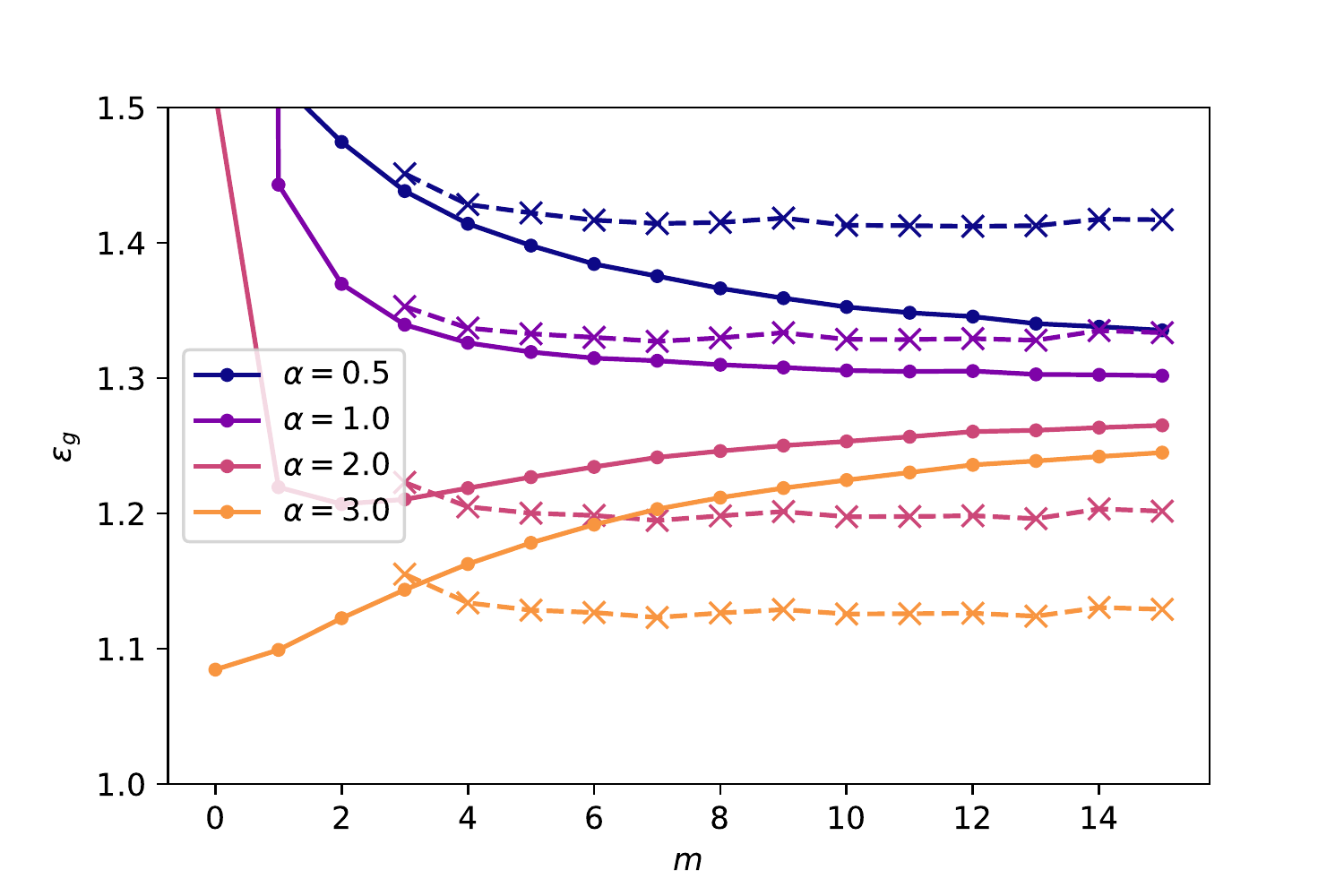}
    \caption{Test error for a regression task on a $L_\star=1$ two-layer target with sign activation and width $\gamma^\star_1=4$. Solid lines represent the test error of a dRF of depth $m$ and widths $\gamma_1=...=\gamma_m=\gamma$, while dashed lines indicate the test errors of wide and shallow dRFs with architecture $\gamma_1=\gamma_3=\gamma$ and $\gamma_2=(m-2)\times k$. All values were evaluated using the sharp asymptotic characterization of conjecture \ref{conj:error_uni_lin}, see also App.\ref{App:error:general}. The parameter $m$, which parameterizes the number of parameters in these two networks, is varied from 0 to 15. For $\gamma=4$ (left), the deep architecture is consistently outperformed by the wide architecture. Closer to the interpolation peak, for $\gamma=1$ (right), the implicit depth-induced regularization means that deeper architectures perform better than wider architectures. }
    \label{fig:depth_width}
\end{figure}

\begin{figure}[t!]
    \centering
    \includegraphics[scale=0.5]{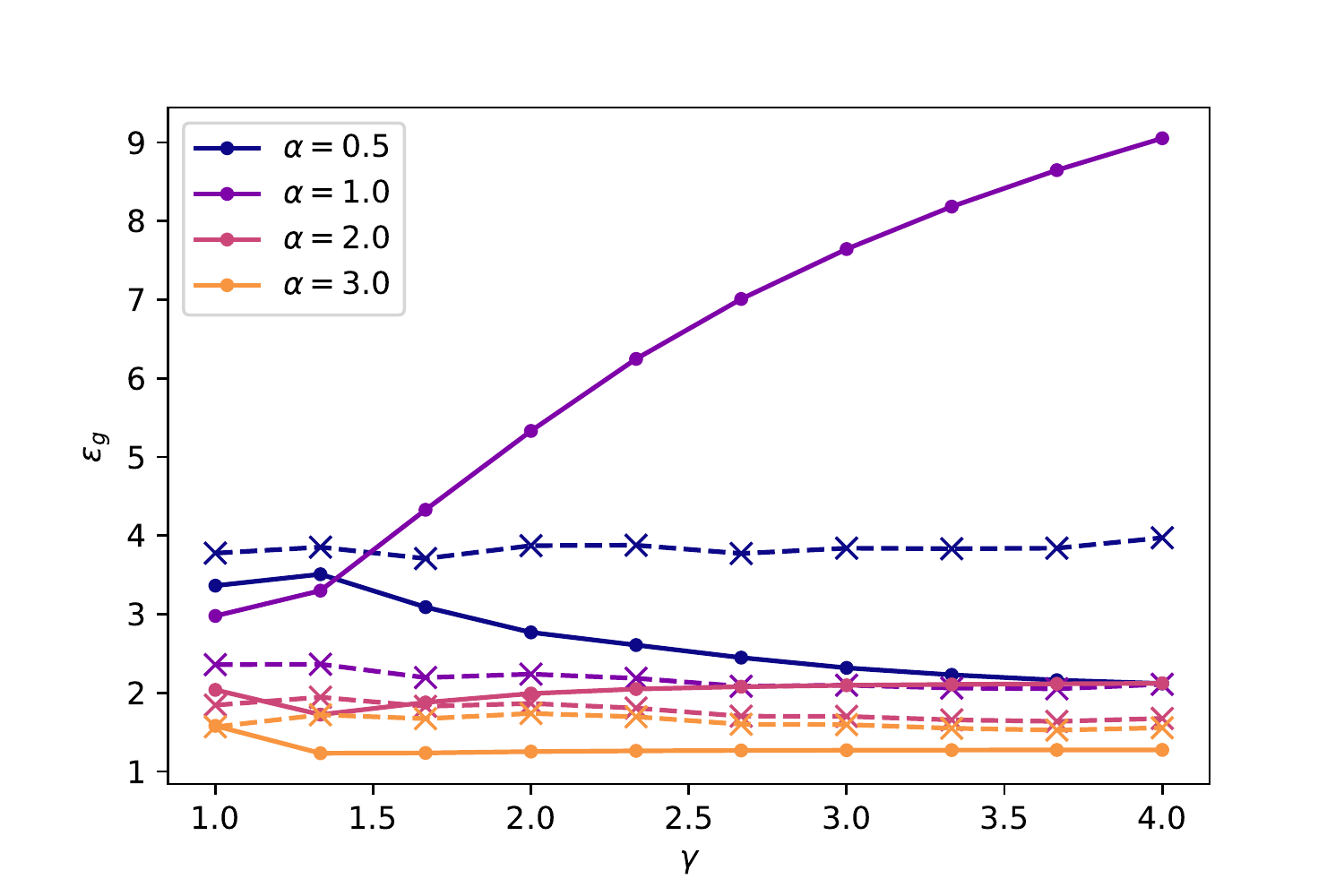}
    \caption{Regression problem over a $L_\star=1$ target with sign activation and width $ \gamma_1^\star$. Dashed lines represent the test error (evaluated using the sharp asymptotics of conjecture \ref{conj:error_uni_lin}, see also App.\ref{App:error:general}) of $L=4$ dRFs, with $\gamma_1=\gamma_2=\gamma_4=\gamma$ and a bottleneck third layer $\gamma_3=\sfrac{1}{2}$. Solid lines corresponds to a rectangular network with no bottleneck $\gamma_1=\gamma_2=\gamma_3=\gamma_4=\gamma$. Close to the interpolation peak ($\alpha=1,2$) the regularization induced by the bottleneck mitigates the overfitting and leads to smaller test errors. }
    \label{fig:bottleneck}
\end{figure}

\subsection{Bottleneck hidden layer}
Another question of interest is the effect of a very narrow hidden layer. Fig.\,\ref{fig:bottleneck} investigates the performance over a $L_\star=1$ target with sign activation and width $ \gamma_1^\star$, of $L=4$ dRFs, with $\gamma_1=\gamma_2=\gamma_4=\gamma$ and a bottleneck third layer $\gamma_3=\sfrac{1}{2}$. The parameter $\gamma$ was varied between 1 and 4. As intuitively expected, the bottleneck, by forcing an intermediary low-dimensional representation, has a regularizing effect. While generically the bottleneck translates into a loss of information, it is beneficial in regimes where regularization is helpful, e.g. close to interpolation peaks or noisy settings. Such an instance is presented in Fig.\,\ref{fig:bottleneck}. Again, if the explicit regularization $\lambda$ is tuned, this effect disappears.

%%%%%%%%%%%%%%%%%%%%%%%%%%%%%%%%%%%%%%%%%%%%%%%%%%%%%%%%%%%%%%%%%%%%%%%%%%%%%%%
\end{document}